\documentclass{article}

    \PassOptionsToPackage{numbers, compress}{natbib}

\usepackage[main, final]{neurips_2026}
\usepackage[utf8]{inputenc} 
\usepackage[T1]{fontenc}    
\usepackage{hyperref}       
\usepackage{url}            
\usepackage{booktabs}       
\usepackage{amsfonts}       
\usepackage{nicefrac}       
\usepackage{microtype}      
\usepackage{xcolor}         

\usepackage{pifont}
\usepackage{amsmath}
\allowdisplaybreaks
\usepackage{amssymb}
\usepackage{mathtools}
\usepackage{amsthm}

\usepackage{subfigure}
\usepackage{tikz-cd}
\usepackage{mathrsfs}
\usepackage{bbm}

\usepackage{enumitem}
\setlist[itemize,enumerate]{noitemsep, topsep=0pt, leftmargin=1em}
\usepackage{multirow}
\usepackage{colortbl}
\usepackage{diagbox}  

\usepackage{caption}
\usepackage{subcaption}
\usepackage{wrapfig}
\usepackage{makecell}
\graphicspath{{Figures/}}
\usepackage{titletoc}
\usepackage{crossreftools}
\usepackage{aliascnt}

\theoremstyle{plain}
\newtheorem{theorem}{Theorem}[section]

\newaliascnt{proposition}{theorem}
\newtheorem{proposition}[proposition]{Proposition}
\aliascntresetthe{proposition}

\newaliascnt{lemma}{theorem}
\newtheorem{lemma}[lemma]{Lemma}
\aliascntresetthe{lemma}

\newaliascnt{corollary}{theorem}

\aliascntresetthe{corollary}

\theoremstyle{definition}
\newaliascnt{definition}{theorem}
\newtheorem{definition}[definition]{Definition}
\aliascntresetthe{definition}

\newaliascnt{assumption}{theorem}

\aliascntresetthe{assumption}

\theoremstyle{remark}
\newaliascnt{remark}{theorem}
\newtheorem{remark}[remark]{Remark}
\aliascntresetthe{remark}

\usepackage[capitalize,noabbrev]{cleveref}

\providecommand{\calN}{\mathcal{N}}

\providecommand{\calL}{\mathcal{L}}

\providecommand{\bbZ}{\mathbb{Z}}
\providecommand{\bbK}{\mathbb{K}}
\providecommand{\bbD}{\mathbb{D}}
\providecommand{\bbL}{\mathbb{L}}
\providecommand{\bbV}{\mathbb{V}}
\providecommand{\bbW}{\mathbb{W}}
\providecommand{\sym}[1]{\mathcal{S}^{#1}}
\providecommand{\spd}[1]{\mathcal{S}^{#1}_{++}}

\providecommand{\tril}[1]{\mathcal{L}^{#1}}
\providecommand{\bbR}[1]{\mathbb {R}^{#1}}

\providecommand{\bbRscalar}{\mathbb {R}}

\providecommand{\orth}[1]{\mathrm{O}({#1})}

\providecommand{\chol}{\operatorname{Chol}}
\providecommand{\dexp}{\operatorname{Dexp}}
\providecommand{\calM}{\mathcal{M}}
\providecommand{\calH}{\mathcal{H}}
\providecommand{\stiefel}[1]{\mathrm{St}(#1)}
\providecommand{\GL}[1]{\mathrm{GL}({#1})}

\providecommand{\dist}{\operatorname{d}}
\providecommand{\rieexp}{\operatorname{Exp}}
\providecommand{\rielog}{\operatorname{Log}}
\providecommand{\pt}[2]{\Gamma_{#1 \rightarrow #2}}

\providecommand{\gyr}{\operatorname{gyr}}
\providecommand{\gyrinner}[2]{\left\langle #1, #2 \right\rangle_{\mathrm{gr}}}

\providecommand{\scrL}{\mathscr{L}}
\providecommand{\bbD}{\mathbb {D}}

\providecommand{\tr}{\operatorname{tr}}

\providecommand{\dlog}{\operatorname{Dlog}}
\providecommand{\pow}{\operatorname{P}}

\providecommand{\inner}[2]{\left\langle #1,#2 \right\rangle}
\providecommand{\Fnorm}[1]{\left\| #1 \right\|_{\mathrm{F}}}
\providecommand{\norm}[1]{\left\| #1 \right\|}
\providecommand{\alphabetainner}[2]{\left\langle #1,#2 \right\rangle^{(\alpha,\beta)}}
\providecommand{\bbzero}{\mathbf{0}}

\providecommand{\LE}{\mathrm{LE}}
\providecommand{\AI}{\mathrm{AI}}
\providecommand{\LC}{\mathrm{LC}}
\providecommand{\PE}{\mathrm{PE}}
\providecommand{\BW}{\mathrm{BW}}

\providecommand{\alphabeta}{(\alpha,\beta)}

\providecommand{\odotle}{\odot^{\mathrm{LE}}}
\providecommand{\odotlc}{\odot^{\mathrm{LC}}}

\providecommand{\oplusle}{\oplus^{\mathrm{LE}}}
\providecommand{\opluslc}{\oplus^{\mathrm{LC}}}

\providecommand{\grasonb}[1]{\mathrm{Gr}(#1)}
\providecommand{\graspp}[1]{\widetilde{\mathrm{Gr}}(#1)}
\providecommand{\onb}{\mathrm{ONB}}
\providecommand{\pp}{\mathrm{PP}}

\providecommand{\rank}{\operatorname{rank}}
\providecommand{\idonb}{I_{p,n}}
\providecommand{\idonbqm}{I_{q,m}}

\providecommand{\idpp}{\widetilde{I}_{p,n}}
\providecommand{\idppqm}{\widetilde{I}_{q,m}}

\providecommand{\qr}{\mathcal{Q}}

\providecommand{\gyr}{\operatorname{gyr}}

\providecommand{\pball}[1]{\mathbb{P}^{#1}_K}
\providecommand{\Moplus}{\oplus_\mathrm{M}}
\providecommand{\Mominus}{\ominus _\mathrm{M}}
\providecommand{\Motimes}{\otimes _\mathrm{M}}

\providecommand{\klein}[1]{\mathbb{K}^{#1}_K}
\providecommand{\Eoplus}{\oplus_\mathrm{E}}
\providecommand{\Eominus}{\ominus_\mathrm{E}}
\providecommand{\Eotimes}{\otimes _\mathrm{E}}

\providecommand{\bbh}[1]{\mathbb{H}^{#1}_K}
\providecommand{\Lnorm}[1]{\left\| #1 \right\|_{\mathcal{L}}}
\providecommand{\Linner}[2]{\left\langle #1, #2 \right\rangle_{\mathcal{L}}}

\providecommand{\sign}{\operatorname{sign}}
\providecommand{\calF}{\mathcal{F}}
\providecommand{\bfP}{\mathbf{P}}
\providecommand{\bfA}{\mathbf{A}}

\definecolor{HilightColortwo}{gray}{0.95}

\definecolor{HilightColor}{RGB}{240, 255, 240} 

\providecommand{\ColFirst}[1]{\ifmmode{\color{red}\boldsymbol{#1}}\else{\color{red}\bfseries\boldmath #1}\fi}
\providecommand{\ColSecond}[1]{\ifmmode{\color{blue}\boldsymbol{#1}}\else{\color{blue}\bfseries\boldmath #1}\fi}
\providecommand{\ColThird}[1]{\ifmmode{\color{cyan}\boldsymbol{#1}}\else{\color{cyan}\bfseries\boldmath #1}\fi}

\newcommand{\na}{\textcolor{gray}{N/A}}%
\newcommand{\cmark}{\textcolor{green}{\text{\ding{51}}}}%
\newcommand{\xmark}{\textcolor{red}{\text{\ding{55}}}}
\renewcommand{\tiny}{\fontsize{5pt}{6pt}\selectfont}
\providecommand{\ie}{\emph{i.e.}}

\crefname{equation}{Eq.}{Eqs.}
\Crefname{equation}{Equation}{Equations}
\crefname{figure}{Fig.}{Figs.}
\Crefname{figure}{Figure}{Figures}
\crefname{table}{Tab.}{Tabs.}
\Crefname{table}{Table}{Tables}
\crefname{algocf}{Alg.}{Algs.}
\Crefname{algocf}{Algorithm}{Algorithms}
\crefname{section}{Sec.}{Secs.}
\Crefname{section}{Section}{Sections}
\crefname{appendix}{App.}{Apps.}
\Crefname{appendix}{Appendix}{Appendices}

\crefname{theorem}{Thm.}{Thms.}
\Crefname{theorem}{Theorem}{Theorems}
\crefname{lemma}{Lem.}{Lems.}
\Crefname{lemma}{Lemma}{Lemmas}
\crefname{definition}{Def.}{Defs.}
\Crefname{definition}{Definition}{Definitions}
\crefname{corollary}{Cor.}{Cors.}
\Crefname{corollary}{Corollary}{Corollaries}
\crefname{remark}{Rmk.}{Rmks.}
\Crefname{remark}{Remark}{Remarks}
\crefname{proposition}{Prop.}{Props.}
\Crefname{proposition}{Proposition}{Propositions}
\crefname{proof}{Pr.}{Prs.}
\Crefname{proof}{Proof}{Proofs}

\input{link_proof}
\usepackage[acronym, toc, hyperfirst=true]{glossaries-extra}
\glsdisablehyper

\newglossary*{notation}{List of Symbols}
\setglossarystyle{long} 
\setabbreviationstyle[acronym]{long-short} 

\makeglossaries 

\newacronym[sort=nn]{CNNs}{CNNs}{Convolutional Neural Networks}
\newacronym[sort=nn]{RNNs}{RNNs}{Recurrent Neural Networks}
\newacronym[sort=nn]{DNNs}{DNNs}{Deep Neural Networks}
\newacronym[sort=nn]{FC}{FC}{Fully Connected}
\newacronym[sort=nn]{MLR}{MLR}{Multinomial Logistics Regression}

\newacronym[sort=nn]{SPDNet}{SPDNet}{SPD Neural Network}
\newacronym[sort=nn]{GrNet}{GrNet}{Grassmann network}
\newacronym[sort=nn]{GrConv}{GrConv}{Grassmannian Convolution}
\newacronym[sort=nn]{HNN}{HNN}{Hyperbolic Neural Network}
\newacronym[sort=nn]{HNN++}{HNN++}{Hyperbolic Neural Network++}
\newacronym[sort=nn]{CorNet}{CorNet}{Correlation Network}
\newacronym[sort=nn]{CorNets}{CorNets}{Correlation Networks}

\newacronym[sort=bn]{BN}{BN}{Batch Normalization}
\newacronym[sort=bn]{RBN}{RBN}{Riemannian Batch Normalization}
\newacronym[sort=bn]{LieBN}{LieBN}{Lie Group Batch Normalization}
\newacronym[sort=bn]{SPDBN}{SPDBN}{SPD Batch Normalization}
\newacronym[sort=bn]{GyroBN}{GyroBN}{Gyrogroup Batch Normalization}

\newacronym[sort=spd]{SPD}{SPD}{Symmetric Positive Definite}
\newacronym[sort=spd]{AIM}{AIM}{Affine-Invariant Metric}
\newacronym[sort=spd]{LCM}{LCM}{Log-Cholesky Metric}
\newacronym[sort=spd]{LEM}{LEM}{Log-Euclidean Metric}
\newacronym[sort=spd]{PEM}{PEM}{Power-Euclidean Metric}
\newacronym[sort=spd]{BWM}{BWM}{Bures--Wasserstein Metric}
\newacronym[sort=spd]{GBWM}{GBWM}{Generalized Bures--Wasserstein Metric}
\newacronym[sort=spd]{BWCM}{BWCM}{Bures--Wasserstein--Cholesky Metric}
\newacronym[sort=spd]{PCM}{PCM}{Power-Cholesky Metric}

\newacronym[sort=gr]{ONB}{ONB}{OrthoNormal Basis}
\newacronym[sort=gr]{PP}{PP}{Projector Perspective}

\newacronym[sort=cor]{ECM}{ECM}{Euclidean--Cholesky Metric}
\newacronym[sort=cor]{LECM}{LECM}{Log-Euclidean--Cholesky Metric}
\newacronym[sort=cor]{LSM}{LSM}{Log-Scaled Metric}
\newacronym[sort=cor]{OLM}{OLM}{Off-Log Metric}
\newacronym[sort=cor]{PHCM}{PHCM}{Poly-Hyperbolic-Cholesky Metric}

\newglossaryentry{manifold}{
  type=notation, 
  name={$\mathcal{M}$}, 
  text={$\mathcal{M}$}, 
  description={Riemannian manifold}, 
}

\newglossaryentry{metric}{
  type=notation,
  name={$\mathbf{g}$},
  text={$\mathbf{g}$},
  description={Metric on $\mathcal{M}$},
}

\newglossaryentry{reals}{
  type=notation,
  name={$\mathbb{R}^n$},
  text={$\mathbb{R}^n$},
  description={A $n$-dimensional Euclidean space},
}

\title{Building Transformation Layers for Riemannian Neural Networks}

\author{%
  Ziheng Chen \\
  University of Trento; MPI for Intelligent Systems, Tübingen \\
  \texttt{ziheng\_ch@163.com} \\
}

\begin{document}

\maketitle

\begin{abstract}
Recently, deep neural networks on manifold-valued representations have garnered significant attention across various machine learning applications. One recent focus is the generalization of Euclidean fully connected (FC) and convolutional layers to non-Euclidean geometries. However, previous approaches typically focus on a few selected manifolds and rely on specific properties of the target manifold. In contrast, this work proposes a framework for constructing FC and convolutional layers over computationally tractable Riemannian spaces. This framework incorporates several previous FC layers across different geometries as special cases and is instantiated on ten representative manifolds, including three hyperbolic models, five geometries of the symmetric positive definite (SPD) manifold, and two Grassmannian perspectives. Experiments on different manifolds demonstrate the effectiveness and applicability of our approach. Code can be found at \url{https://github.com/GitZH-Chen/RieTrans}.
\end{abstract}

\section{Introduction}
\label{sec:introduction}
Deep neural networks on Riemannian manifolds have achieved remarkable success across various applications \citep{huang2017riemannian,ganea2018hyperbolic,skopek2019mixed,lopez2021vector,huang2022riemannian,wangspdmetric2024,chen2024flow,khan2025hyperbolic,pouliquen2025schurs,li2025spdim,li2026heegnet,hu2026rhop,jin2026robust,chen2026busemann,chen2026correlation}. Commonly encountered manifolds include \gls{SPD} \citep{pennec2006riemannian}, Grassmannian \citep{bendokat2024grassmann}, matrix Lie group \cite{hall2013lie}, and hyperbolic \citep{ungar2022gyrovector} manifolds. These manifolds admit \emph{computationally tractable tools} such as geodesics, exponential and logarithmic maps, and parallel transport, which have enabled the extension of fundamental deep learning components, including normalization \citep{brooks2019riemannian,chakraborty2020manifoldnorm,lou2020differentiating,kobler2022spd,chen2024liebn,chen2025gyrobn,wang2025gbwbn}, attention \citep{gulcehre2018hyperbolic,pan2022matt,wang2024grassatt,wang2025gyroatt}, residual blocks \citep{van2023poincare,katsman2024riemannian}, and classification layers \citep{ganea2018hyperbolic,nguyen2023building,chen2024rmlrspd,chen2024rmlr,bdeir2024fully}.

Yet, the extension of the most basic layers, \gls{FC} and convolutional layers, remains particularly challenging. Early works targeted specific manifolds: \citet{huang2017riemannian,huang2017deep,huang2018building} proposed layers for \gls{SPD}, special orthogonal, and Grassmann manifolds. Later, \citet{ganea2018hyperbolic,mao2024klein} developed hyperbolic counterparts via tangent spaces, and \citet{chen2022fully} introduced LFC via spacetime transformations. \citet{fan2022nested} further introduced nested hyperbolic spaces and a Lorentz-specific feature transformation that supports dimensionality reduction. To better respect geometry, \citet{shimizu2021hyperbolic} extended FC and convolutional layers on the Poincaré model, while \citet{nguyen2024matrix,nguyen2025symmetric} proposed SPD counterparts based on gyrovector structures and symmetric spaces. However, these methods largely rely on specific properties, such as Poincaré geometry, gyro or symmetric structures, which restricts their generality. In another direction, \citet{chakraborty2020manifoldnet} introduced a convolution based on the weighted Fréchet mean, but unlike the Euclidean convolution, its output manifold dimension is restricted to match the input, limiting flexibility. Consequently, a general and flexible framework for constructing FC and convolutional layers across different geometries remains unsolved.

We address this challenge by proposing a principled framework for building Riemannian FC and convolutional layers on computationally tractable manifolds. Our framework relies solely on Riemannian operators such as exponential and logarithmic maps. It thus applies broadly to different geometries, such as hyperbolic, SPD, and Grassmannian spaces. Our contributions are summarized as follows.
\begin{itemize}
    \item
    \textbf{Riemannian FC and convolutional layers.}
    We introduce a principled generalization of FC and convolutional layers to Riemannian spaces. In contrast to previous approaches, our framework requires only tractable Riemannian operators, ensuring broad applicability. Moreover, several existing Riemannian FC layers are subsumed as special cases.
    \item
    \textbf{Ten concrete instantiations.}
    We instantiate our framework on three hyperbolic models, five SPD geometries, and two Grassmannian perspectives. Our approach enables direct variation of the latent geometry under a consistent network architecture.
    \item
    \textbf{Empirical validation.}
    We validate our approach on benchmark tasks across hyperbolic, SPD, and Grassmannian manifolds, demonstrating both effectiveness and versatility.
\end{itemize}

\section{Preliminaries}

\textbf{Notations.}
For the Euclidean space $\bbR{n}$ or $\bbR{n \times n}$, we denote the standard inner product by $\inner{\cdot}{\cdot}$ and the induced norm by $\norm{\cdot}$, \ie, the $L_2$-norm for vectors and the Frobenius norm for matrices. A Riemannian manifold $(\calM, g)$ with the Riemannian metric $g$ is abbreviated as $\calM$. Its tangent space at $P \in \calM$ is denoted by $T_P\calM$. The Riemannian logarithmic map, exponential map, and metric at $P \in \calM$ are denoted by $\rielog_P$, $\rieexp_P$, and $\inner{\cdot}{\cdot}_{P} = g_P(\cdot,\cdot)$, respectively. The parallel transport along the geodesic connecting $P, Q \in \calM$ is $\pt{P}{Q}$. A table of notation is summarized in \cref{app:sec:notations}.

\textbf{Hyperbolic manifold.}
There are five isometric hyperbolic models \citep{cannon1997hyperbolic}. We focus on the Poincaré ball $\pball{n} =\left\{x \in \mathbb{R}^{n} \mid \norm{x}^2 < -\nicefrac{1}{K} \right\}$, the Beltrami–Klein ball $\klein{n} =\left\{x \in \mathbb{R}^{n} \mid \norm{x}^2 < -\nicefrac{1}{K} \right\}$, and the hyperboloid (or Lorentz) model $\bbh{n}=\left\{x \in \mathbb{R}^{n+1} \mid \Lnorm{x}^2 = \nicefrac{1}{K}, x_1>0 \right\}$, where $\Lnorm{x}^2 = \sum_{i=2}^{n+1} x _i ^2 - x _1 ^2$ is the Lorentz inner product. Here, $K<0$ is the constant curvature. The Poincaré and Beltrami–Klein balls admit gyrovector spaces, known as the Möbius and Einstein gyrovector spaces \citep{ungar2022analytic}, respectively. The Möbius gyroaddition and scalar gyromultiplication are denoted by $\Moplus$ and $\Motimes$, while the Einstein counterparts are $\Eoplus$ and $\Eotimes$. \cref{app:subsec:geom_hyperbolic} summarizes the associated Riemannian and gyro operators.

\textbf{SPD manifold.}
The set of $n \times n$ SPD matrices, denoted $\spd{n}$, forms a smooth manifold, called the SPD manifold \citep{arsigny2005fast}. It admits five widely used Riemannian metrics: \gls{AIM} \citep{pennec2006riemannian}, \gls{LEM} \citep{arsigny2005fast}, \gls{PEM} \citep{dryden2010power}, \gls{LCM} \citep{lin2019riemannian}, and \gls{BWM} \citep{bhatia2019bures}. Each metric provides closed-form expressions for Riemannian operators, which are summarized in \cref{app:subsec:geom_spd}.

\textbf{Grassmannian.}
The Grassmannian is the manifold of $p$-dimensional subspaces of the $n$-dimensional vector space \citep[Prob. 7.8]{loring2011introduction}. It has two common matrix representations \citep{bendokat2024grassmann}. The \gls{PP} embeds each element as an $n \times n$ symmetric matrix: $\graspp{p,n}= \{P \in \sym{n} \mid P^2=P, \rank(P)=p \}$, where $\sym{n}$ is the Euclidean space of symmetric matrices. The \gls{ONB} perspective is the quotient of the Stiefel manifold $\stiefel{p, n}$: $\grasonb{p,n} = \stiefel{p, n} / \orth{p} = \{ [U] \mid [U]= \{\widetilde{U} \in \stiefel{p, n} \mid \widetilde{U}=U R, R \in \orth{p}\} \}$, where $\orth{p}$ is the $p \times p$ orthogonal group. By abuse of notation, we use $[U]$ and $U$ interchangeably. The associated Riemannian structures are summarized in \cref{app:subsec:geom_grass}.

The considered manifolds admit multiple geometries, including isometric hyperbolic models, distinct SPD metrics, and diffeomorphic Grassmannian perspectives, whose empirical performance often varies across tasks \citep{nguyen2022gyro,katsman2024riemannian,chen2025gyrobnextension}. This motivates a unified framework to flexibly handle such variants. Moreover, exponential and logarithmic maps may encounter singular cases, but these cases can be handled numerically (\cref{app:rmk:incomplete_spd,app:rmk:cutlocus_gras}), and the maps are assumed well-defined.

\section{Proposed Framework}

\subsection{Riemannian fully connected layers}
\label{subsec:riem_fc}

Our method for building FC layers over Riemannian manifolds relies on the point-to-hyperplane distance, which has shown success in building hyperbolic and SPD networks \citep{shimizu2021hyperbolic,nguyen2023building,bdeir2024fully,chen2024rmlrspd,chen2024rmlr,nguyen2024matrix,nguyen2025symmetric}. The hyperplane in the Riemannian manifold $\calN$ \citep[Eq. 5]{chen2024rmlr} is $H_{A, P} = \{X \in \calN \mid \inner{\rielog_{P} (X)}{A} _{P} =0\}$, with $P \in \calN$ and $A \in T _P \calN$. When $\calN = \bbR{n}$, it recovers the Euclidean hyperplane, $H_{a, p} =\{x \in \mathbb{R}^n \mid \langle a, x - p \rangle=0\}$.

The Euclidean FC layer is defined as $y = Ax + b$ with $A \in \bbR{m \times n}$ and $b \in \bbR{m}$. It can be expressed element-wise as $y_k = \langle a_k, x \rangle - b_k= \langle a_k, x - p_k \rangle$ with $a_k, p_k \in \mathbb{R}^n$ and $\langle p_k, a_k \rangle = b_k$. As shown by \citet[Sec. 3.2]{shimizu2021hyperbolic}, the LHS $y _k$ is the signed distance from $y$ to the hyperplane passing through the origin and orthogonal to the $k$-th axis of the output space, which can be formulated as
\begin{equation} \label{eq:euc_fc}
    \operatorname{sign}(\langle e_k, y - \bbzero \rangle) d (y, H_{e_k, \bbzero}) = \langle a_k, x - p_k \rangle, \quad \forall 1 \leq k \leq m,
\end{equation}
where $\bbzero \in \bbR{m}$ is the zero vector, and $\{e_k\}_{k=1}^m$ forms an orthonormal basis over $\bbR{m}$ with $e_k$ denoting the vector whose $k$-th element is 1 and all others are 0. Here, the LHS of \cref{eq:euc_fc} equals $y_k$.

Given a point-to-hyperplane distance, \cref{eq:euc_fc} can be readily generalized to manifolds. Noting that $\rielog_{p}(x)=x-p$ under the Euclidean geometry and that $T_{\bbzero}\bbR{m} \cong \bbR{m}$, the counterparts of $H_{e_k, \bbzero}$ on an $m$-dimensional Riemannian manifold $\calM$ are defined as
\begin{equation}
    \label{eq:hyperplane_e_b_k}
    H_{B_k, E} = \{S \in \calM \mid \inner{\rielog_{E} S}{B_k}_{E} =0\}, \quad \forall 1 \leq k \leq m,
\end{equation}
where $E \in \calM$ is the predefined origin, and $\{B_k\}_{k=1}^m$ is an orthogonal basis over $\{T_E\calM,g_E\}$. \cref{eq:hyperplane_e_b_k} characterizes the hyperplane containing the origin $E$ and orthogonal to the geodesic starting from $E$ with the initial velocity $B_k$, which recovers the Euclidean $H_{e_k, \bbzero}$ as $\calM=\bbR{m}$. With all the above ingredients, we define the Riemannian FC layer in the following.

\begin{definition} \label{def:riem_fc}
    Given an $n$-dimensional manifold $\calN$ and an $m$-dimensional manifold $\calM$, the Riemannian FC layer $\calF: \calN \to \calM$ for the input $X \in \calN$ returns the output $Y \in \calM$ by solving $m$ equations:
    \begin{equation} \label{eq:riem_fc_original}
        \sign \left( \left\langle \rielog^\calM _{E} (Y), B _k \right\rangle^\calM_E \right) \dist^\calM (Y, H^\calM _{B_k, E^\calM}) = \langle A_k, \rielog^\calN _{P_k} (X) \rangle^\calN _{P_k}, 1 \leq k \leq m,
    \end{equation}
    where $E^\calM \in \calM$ is the origin, and $\{ B_k \}_{k=1}^m$ is an orthonormal basis over $T_{E^\calM}\calM$. Here, $\dist^\calM$ denotes a chosen point-to-hyperplane distance, which can be either the true distance $\inf_{S\in H^\calM_{B_k,E^\calM}}d_\calM(Y,S)$ or a surrogate. Meanwhile, $\rielog_{P_i}^{\calN}$ and $\langle \cdot , \cdot \rangle_{P_i}^{\calN}$ are the Riemannian logarithm and metric over $\calN$. The pairs $P_k \in \calN$ and $A_k \in T_{P_k}\calN$ are the FC parameters.
\end{definition}

Our \cref{def:riem_fc} naturally extends previous FC layers over different geometries.
\begin{proposition}
    \label{prop:riem_fc_gen_euc_fc}
    \linktoproof{prop:riem_fc_gen_euc_fc}
When $\calN=\bbR{n}$ and $\calM=\bbR{m}$, \cref{def:riem_fc} with the true point-to-hyperplane distance reduces to the Euclidean FC layer. When $\calN=\pball{n}$, $\calM=\pball{m}$, the true point-to-hyperplane distance follows \citet[Thm. 5]{ganea2018hyperbolic}, and the LHS of \cref{eq:riem_fc_original} follows $v_k(\cdot)$ from \citet[Eq. (3)]{shimizu2021hyperbolic}. In this case, \cref{def:riem_fc} yields the Poincaré FC layer \citep[Sec. 3.2]{shimizu2021hyperbolic}. When $\calN=\spd{n}$, $\calM=\spd{m}$, and the point-to-hyperplane distances are pseudo-gyrodistances \citep[Thms. 2.23–2.25]{nguyen2023building}, \cref{def:riem_fc} recovers the corresponding gyro SPD FC layers \citep[Props. 3.4–3.6]{nguyen2024matrix}.
\end{proposition}

The crux of \cref{def:riem_fc} lies in the point-to-hyperplane distance and solving the resulting $m$ equations. Using the true point-to-hyperplane distance presents three difficulties: it has no general closed-form expression on arbitrary Riemannian manifolds; on a specific manifold, evaluating its infimum may require a difficult geometry-specific nonconvex optimization \citep[Sec. 3.1]{chen2024rmlr}; and having a computable distance alone does not ensure that the $m$ equations in \cref{def:riem_fc} admit a common output $Y$, as illustrated by the counterexample in \citet[App. D]{chen2026busemann}. To avoid these difficulties, we adopt the Riemannian-trigonometric point-to-hyperplane pseudo-distance of \citet[Def. 3.1 and Thm. 3.2]{chen2024rmlr}: $\dist(X,H_{A_k, P_k}) = \frac{|\langle \rielog_{P_k}(X), A_k \rangle_{P_k}|}{\| A_k \|_{P_k}}$, where $\norm{\cdot}_{P_k}$ denotes the norm induced by $\inner{\cdot}{\cdot}_{P_k}$. Under this pseudo-distance, the implicit \cref{def:riem_fc} admits the following explicit solution.

\begin{theorem}[Riemannian FC Layers]
    \label{thm:riem_fc}
    \linktoproof{thm:riem_fc}
    Following the notation in \cref{def:riem_fc} the Riemannian FC layer $\calF(\cdot): \calN \rightarrow \calM$ for the input $X \in \calN$ is $Y = \rieexp^{\calM}_E \left( \sum_{i=1}^{m} \langle \rielog^{\calN}_{P_i}(X), A_i \rangle^{\calN}_{P_i} B_i \right)$,
    where $\rieexp_{E}^{\calM}$ is the Riemannian exponential map on $\calM$.
\end{theorem}

The tangent FC layer \citep[Lem. 6]{ganea2018hyperbolic} uses a single tangent space and is written as $\rieexp _{E}(f(\rielog _E(X)))$, with $f$ being a Euclidean FC layer. In contrast, our formulation involves multiple tangent spaces, where each $\langle \rielog^{\calN}{P_i}(X), A_i \rangle^{\calN}_{P_i} = 0$ corresponds to a Riemannian hyperplane $H_{A_i, P_i}$. Moreover, our formulation naturally incorporates prior Riemannian FC layers without requiring additional geometric or algebraic structures, as summarized in \cref{tab:comp_spd_fc_layers}.
\begin{remark}
The point-to-hyperplane pseudo-distance adopted above coincides with the true geodesic point-to-hyperplane distance when the manifold is isometric to an Euclidean space. Specifically, if $\phi:(\calN,g)\to(\bbR{n},\langle\cdot,\cdot\rangle)$ is an isometry, extending \citet[Lem. 3.5]{chen2024rmlrspd} gives
$\inf_{Q\in H_{A,P}}\dist(X,Q)=\frac{|\langle\rielog_P(X),A\rangle_P|}{\|A\|_P}$. This includes Euclidean space as well as LEM and LCM on SPD manifolds.
\end{remark}

\begin{remark}
Theoretically, the origin $E^\calM$ in \cref{def:riem_fc} can be chosen arbitrarily. On a homogeneous space, any two choices and the FC layers defined through them are identified by an isometry (\cref{app:thm:fc_isometries}). In practice, we choose canonical origins that simplify the Riemannian operators and the resulting FC expressions. The examples in \cref{sec:examples} follow this computational criterion.
\end{remark}

\subsection{Riemannian convolutional layers}
\label{sec:riem_conv_layers}

As shown in \citep[Sec. 3.4]{shimizu2021hyperbolic}, Euclidean convolution first concatenates the features within each receptive field and then applies an FC transformation. The concatenated feature vector $x \in (\bbR{n})^c$ is an element of the product space $(\bbR{n})^c$, and its $k$-th output is the affine transformation $y_k=\left\langle a_k, x\right\rangle - b_k$. This view naturally extends to manifolds: concatenating manifold-valued features forms an element of the product manifold, on which we apply the Riemannian FC layer in \cref{subsec:riem_fc}.

\begin{wrapfigure}{r}{0.6\textwidth}
    \centering
    \includegraphics[width=\linewidth,trim=0.6cm 0.8cm 0cm 0cm]{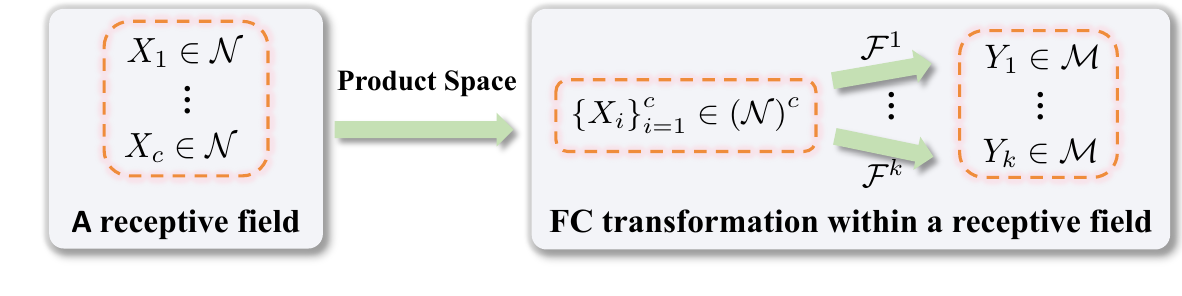}
    \caption{Riemannian convolution within a receptive field. Here, $\calF^k(\cdot)$ denotes the $k$-th FC transformation.}
    \label{fig:riem_conv}
    \vspace{-5mm}
\end{wrapfigure}
\textbf{Riemannian convolution.}
In each receptive field, the product-manifold element $X=(X_1,\ldots,X_c)\in(\calN)^c$ is fed into $k$ Riemannian FC layers, where $k$ is the number of kernels. Here, each Riemannian FC layer is implemented under the product geometry $(\calN)^{c}=\Pi_{i=1}^c \calN$, which is detailed in \cref{app:subsec:riem_fc_product}. \cref{fig:riem_conv} illustrates the above process.

\subsection{Parameter trivialization}
\label{subsec:fc_parameters}

\begin{wraptable}{r}{0.6\textwidth}
\centering
\vspace{-3mm}
\caption{Comparison of hyperbolic FC layers. An extended table can be found in \cref{app:sec:comp_hyperbolic_linear}.}
\label{tab:comp_hyperbolic_linear}
\resizebox{0.9\linewidth}{!}{
\begin{tabular}{cccc}
\toprule
Method & Space & Mechanism & References \\
\midrule
Möbius & $\pball{n}$ & Tangent & \citep[Def. 3.2]{ganea2018hyperbolic} \\
Einstein & $\klein{n}$ & Tangent & \citep[Thm. 9]{mao2024klein} \\
LFC & $\bbh{n}$ & Spacetime & \citep[Eq. (1)]{chen2022fully} \\
NestFC & $\bbh{n}$ & Nested projection & \citep[Eq. (14)]{fan2022nested} \\
Poincaré FC & $\pball{n}$ & Poincaré & \citep[Sec. 3.2]{shimizu2021hyperbolic} \\
\midrule
\rowcolor{HilightColor} Ours & $\pball{n},\klein{n},\bbh{n}$ & Riemannian & \cref{app:thm:pball_fc,app:thm:hyperboloid_fc} \\
\bottomrule
\end{tabular}
}
\vspace{-3mm}
\end{wraptable}
As convolution uses the FC layer as its prototype, we focus on the FC parameters. Since $P_i$ varies during training, $A_i \in T_{P_i}\calN$ cannot be updated directly by a Euclidean optimizer. As shown by \citet[Eqs. (12)--(13)]{chen2024rmlr}, it can be determined from the fixed tangent space at the origin $E^\calN \in \calN$ by\footnote{Although $\Gamma$ could be flexibly replaced by other maps between tangent spaces, such as vector transport and differential of group translation, we mainly use parallel transport.} $A_i = \pt{E^\calN}{P_i}(Z_i)$ with $Z_i \in T _{E^\calN} \calN$. Moreover, as shown by \citet[Sec. 3.1]{shimizu2021hyperbolic}, $p_k$ in \cref{eq:euc_fc} may be over-parameterized, since there are countless $p_k$ satisfying $\langle a_k, p_k\rangle =b_k$. Therefore, following \citet{shimizu2021hyperbolic}, each $P_i$ in the Riemannian FC layer is parameterized as $\rieexp ^\calN _{E^\calN} (\gamma_i [Z_i])$, where $\gamma_i \in \bbRscalar$, $Z_i \neq 0$, and $[Z_i]$ is the unit vector of $Z_i$. This use of the exponential map to optimize manifold-valued parameters is known as trivialization \citep[Sec. 4.1]{lezcano2019trivializations}.

Thus, instead of independently optimizing the two $n$-dimensional parameters $P_i$ and $A_i$, we parsimoniously optimize one tangent vector $Z_i$ and one scalar $\gamma_i$ for each output dimension. This reduces the number of parameters from $2n$ to $n+1$ and allows them to be updated by a Euclidean optimizer, instead of expensive Riemannian optimizers. 

\section{Examples}
\label{sec:examples}

Although \cref{thm:riem_fc} is geometry-agnostic, its instantiation can be further simplified under a specific geometry. We now instantiate the FC layer in \cref{thm:riem_fc} over different geometries, including three hyperbolic models, five SPD geometries, and two Grassmannian perspectives.

\subsection{Hyperbolic vector manifolds}

We focus on three hyperbolic models: the Poincaré ball, the Beltrami–Klein model, and the hyperboloid model. The resulting FC layers are denoted by HFC-P, HFC-K, and HFC-H, respectively. \cref{tab:comp_hyperbolic_linear} compares our hyperbolic FC layers against prior layers, where we refer to the hyperbolic feature transformation of \citet[Eq. (14)]{fan2022nested} as NestFC.

\textbf{Poincaré \& Beltrami–Klein.}
These two models admit Möbius and Einstein gyrovector spaces \citep{ungar2022analytic}, respectively. These structures further simplify the concrete HFC layers.

\begin{theorem}[HFC-P \& HFC-K]
    \label{app:thm:pball_fc}
    \linktoproof{app:thm:pball_fc}
    Let $\calH ^n \in \{\pball{n},\klein{n}\}$. Given $x \in \calH ^n$, the Riemannian FC layer $\calF(\cdot): \calH ^n \rightarrow \calH ^m$ is $y = \rieexp_{\bbzero} \left( \left( v_1(x),\cdots,v_m(x) \right)^\top \right)$. Here, $v_i(x)=\left\langle \rielog_\bbzero (-p_i \oplus _{\calH} x ), z_i \right\rangle$, with the zero vector $\bbzero$ as the origin and $p_i = \rieexp_{\bbzero}(\gamma_i [z_i])$. The FC parameters are $\{\gamma_i \in \bbRscalar \}_{i=1}^m$ and $\{z_i \in \bbR{n} \}_{i=1}^m$. The gyroaddition and Riemannian exponential and logarithmic maps can be found in \cref{app:subsec:geom_hyperbolic}, where $\rieexp_\bbzero$ ($\rielog_\bbzero$) shares the same expression in the two models.
\end{theorem}

Interestingly, the only difference between the HFC-P and HFC-K layers lies in the gyroaddition. Moreover, HFC-P takes an expression different from that of the Poincaré FC layer \citep[Sec. 3.2]{shimizu2021hyperbolic}, since their point-to-hyperplane distances and LHSs of \cref{eq:riem_fc_original} are different.

\textbf{Hyperboloid.}
The origin is defined as $e= \left(\nicefrac{1}{\sqrt{|K|}}, 0, \cdots, 0 \right)^\top$, which corresponds to the Poincaré origin under the stereographic projection \citep[Sec. 2.1]{skopek2019mixed}. Then, we have the following.

\begin{theorem}[HFC-H FC layer]
    \label{app:thm:hyperboloid_fc}
    \linktoproof{app:thm:hyperboloid_fc}
        The Riemannian FC layer $\calF(\cdot): \bbh{n} \rightarrow \bbh{m}$ for $x \in \bbh{n}$ is $y = \rieexp_e \left( \left( 0, v_1(x),\cdots,v_m(x) \right)^\top \right)$, where $v_i(x)=\left\langle \rielog_{p_i} ( x ), \pt{e}{p_i} (0,z_i) \right\rangle$, and $p_i = \rieexp_{e}(\gamma_i [(0,z_i^\top)^\top])$, with $\gamma_i \in \bbRscalar$ and $z_i \in \bbR{n}$ as parameters.
\end{theorem}

Substituting the hyperbolic operators into the preceding HFC layers gives the following expressions.
\begin{theorem}[Simplification]
    \label{thm:hfc_batch_closed_form}
    \linktoproof{thm:hfc_batch_closed_form}
    For the hyperboloid model, we write $x=(x_1,x_s)$. For each HFC layer, $v_i(x)$ can be further simplified:
    \begin{align*}
        \pball{n}: v_i(x)
        &=\resizebox{0.82\linewidth}{!}{$\displaystyle
        \|z_i\|
        \frac{\tanh^{-1}\left(\sqrt{|K|}Q_i^{\mathrm{P}}\right)}
        {\sqrt{|K|}Q_i^{\mathrm{P}}}
        \frac{
        \left[1+\tanh^2(\sqrt{|K|}\gamma_i)\right]\langle x,[z_i]\rangle
        -\frac{\tanh(\sqrt{|K|}\gamma_i)}{\sqrt{|K|}}
        \left(1+|K|\|x\|^2\right)}
        {D_i^{\mathrm{P}}}$},\\
        \klein{n}: v_i(x)
        &=\|z_i\|
        \frac{\tanh^{-1}\left(\sqrt{|K|}Q_i^{\mathrm{K}}\right)}
        {\sqrt{|K|}Q_i^{\mathrm{K}}}
        \frac{
        \langle x,[z_i]\rangle
        -\frac{\tanh(\sqrt{|K|}\gamma_i)}{\sqrt{|K|}}}
        {D_i^{\mathrm{K}}},\\
        \bbh{n}: v_i(x)
        &=\resizebox{0.82\linewidth}{!}{$\displaystyle
        \|z_i\|
        \frac{
        \cosh^{-1}\left(
        \sqrt{|K|}\left[
        \cosh(\sqrt{|K|}\gamma_i)x_1
        -\sinh(\sqrt{|K|}\gamma_i)\langle x_s,[z_i]\rangle
        \right]
        \right)}
        {\sqrt{
        |K|\left[
        \cosh(\sqrt{|K|}\gamma_i)x_1
        -\sinh(\sqrt{|K|}\gamma_i)\langle x_s,[z_i]\rangle
        \right]^2-1}}
        \left[
        \cosh(\sqrt{|K|}\gamma_i)\langle x_s,[z_i]\rangle
        -\sinh(\sqrt{|K|}\gamma_i)x_1
        \right]$}, \\
        D_i^{\mathrm{P}}
        &=1-2\sqrt{|K|}\tanh(\sqrt{|K|}\gamma_i)\langle x,[z_i]\rangle
        +|K|\tanh^2(\sqrt{|K|}\gamma_i)\|x\|^2,\\
        Q_i^{\mathrm{P}}
        &=\sqrt{\frac{
        \|x\|^2+\frac{\tanh^2(\sqrt{|K|}\gamma_i)}{|K|}
        -\frac{2\tanh(\sqrt{|K|}\gamma_i)}{\sqrt{|K|}}\langle x,[z_i]\rangle}
        {D_i^{\mathrm{P}}}},\\
        D_i^{\mathrm{K}}
        &=1-\sqrt{|K|}\tanh(\sqrt{|K|}\gamma_i)\langle x,[z_i]\rangle,\\
        Q_i^{\mathrm{K}}
        &=\sqrt{\frac{
        \|x\|^2+\frac{\tanh^2(\sqrt{|K|}\gamma_i)}{|K|}
        -\frac{2\tanh(\sqrt{|K|}\gamma_i)}{\sqrt{|K|}}\langle x,[z_i]\rangle
        -\tanh^2(\sqrt{|K|}\gamma_i)
        \left(\|x\|^2-\langle x,[z_i]\rangle^2\right)}
        {(D_i^{\mathrm{K}})^2}}.
    \end{align*}
\end{theorem}

Although \cref{thm:hfc_batch_closed_form} appears more complicated than the original formulations, it enables efficient batched computation and greatly reduces GPU memory consumption. Directly implementing \cref{app:thm:pball_fc,app:thm:hyperboloid_fc} requires either looping over the $m$ output dimensions or materializing a large $B\times m\times n$ intermediate tensor. As shown in \cref{thm:hfc_batch_closed_form}, the formulas mainly depend on $\langle x,[z_i]\rangle$ for HFC-P and HFC-K and $\langle x_s,[z_i]\rangle$ for HFC-H. Stacking $B$ inputs into $X\in\bbR{B\times n}$ (or their spatial components into $X_s\in\bbR{B\times n}$) and the unit directions into $U=([z_1],\ldots,[z_m])^\top\in\bbR{m\times n}$ allows all these inner products to be computed at once as $XU^\top$ (or $X_sU^\top$). The remaining terms can be evaluated elementwise.

\subsection{SPD matrix manifolds}
\label{subsec:spd_fc_conv}

We focus on five popular Riemannian metrics, \ie, LEM, AIM, PEM, LCM, and BWM. We define the identity matrix $I$ as the origin, since it corresponds to the zero matrix under the matrix logarithm.

\begin{theorem}[SPD FC Layers]
    \label{thm:spd_fc}
    \linktoproof{thm:spd_fc}
    Given an SPD matrix $S \in \spd{n}$, the outputs of the SPD FC layers $\calF(\cdot): \spd{n} \rightarrow \spd{m}$ under different Riemannian metrics are
    {\small
    \begin{align}
        \label{eq:lem_fc}
        \text{LEM}:
        &  Y = \exp \left( V^\LE \right),
        V^{\LE}_{ij} =
        \begin{cases}
        \frac{1}{\sqrt{\alpha}} v^{\LE}_{ii}(S) + \mu \sum_{k=1}^m v^{\LE}_{kk}(S) , & \text { if } i=j \\
        \frac{1}{\sqrt{2\alpha}} v^{\LE}_{ij}(S), & \text { if } i>j \\
        V^{\LE}_{ji}, & \text{ otherwise }
        \end{cases} \\
        \text{AIM}:
        &  Y = \exp \left( V^\AI \right),
        V^{\AI}_{ij} =
        \begin{cases}
        \frac{1}{\sqrt{\alpha}} v^{\AI}_{ii}(S) + \mu \sum_{k=1}^m v^{\AI}_{kk}(S) , & \text { if } i=j \\
        \frac{1}{\sqrt{2\alpha}} v^{\AI}_{ij}(S), & \text { if } i>j \\
        V^{\AI}_{ji}, & \text{ otherwise }
        \end{cases}\\
        \text{PEM}:
        & Y =  \left( I+ V^\PE \right)^{\frac{1}{\theta}}, V^{\PE}_{ij} =
        \begin{cases}
        \frac{1}{\sqrt{\alpha}} v^{\PE}_{ii}(S) + \mu \sum_{k=1}^m v^{\PE}_{kk}(S) , & \text { if } i=j \\
        \frac{1}{\sqrt{2\alpha}} v^{\PE}_{ij}(S), & \text { if } i>j \\
        V^{\PE}_{ji}, & \text{ otherwise }
        \end{cases} \\
        \label{eq:spd_fc_lcm}
        \text{LCM}:
        & Y = V^{\LC}(V^{\LC})^\top,
        V^{\LC}_{ij} =
        \begin{cases}
        \exp\left(v^{\LC}_{ii}(S)\right) , & \text { if } i=j \\
        v^{\LC}_{ij}(S), & \text { if } i>j\\
        0, & \text{ otherwise }
        \end{cases}\\
        \text{BWM}:
        & Y = \left(I+\frac{1}{2}V^\BW \right)^2,
        V^\BW_{ij} =
        \begin{cases}
        v^{\BW}_{ii}(S) , & \text { if } i=j \\
        \frac{1}{\sqrt{2}}v^{\BW}_{ij}(S), & \text { if } i>j \\
        V^{\BW}_{ji}, & \text{ otherwise }
        \end{cases}
    \end{align}
    }
    Here, $v_{ij}(S)$ under different metrics is defined as
    {\small
    \begin{align*}
        &\text{LEM}:
        \left\langle \log(S)-\log(P_{ij}), Z_{ij} \right\rangle^{\alphabeta}, \quad
        \text{AIM}:
        \left\langle \log(P_{ij}^{-\frac{1}{2}} S P_{ij}^{-\frac{1}{2}}), Z_{ij} \right\rangle^{\alphabeta}, \\
        &\text{PEM}:
        \left\langle S^\theta-P_{ij}^\theta, Z_{ij} \right\rangle^{\alphabeta}, \quad
        \text{LCM}:
        \left\langle \lfloor K\rfloor - \lfloor L_{ij} \rfloor + \dlog(\bbK\bbL_{ij}^{-1}), \lfloor Z_{ij} \rfloor + \frac{1}{2}\bbZ_{ij} \right\rangle,
    \end{align*}
    \begin{equation*}
        \text{BWM}:
        \left\langle \left(P_{ij}S \right)^{\frac{1}{2}} + \left(S P_{ij} \right)^{\frac{1}{2}} -2 P_{ij},  \calL_{P_{ij}}(L_{ij} Z_{ij} L_{ij}^\top) \right\rangle,
    \end{equation*}
    }
    The above notations are defined in the following.
    \begin{itemize}
        \item
        $Z_{ij} \in T_I\spd{n} \cong \sym{n}$ and $P_{ij} \in \spd{n}$ are the parameters for $1 \leq i \leq j \leq m$,
        \item
        $\log(\cdot)$ is the matrix logarithm.
        $\dlog(\cdot)$ is the diagonal element-wise logarithm.
        $\lfloor \cdot \rfloor$ is the strictly lower part of a square matrix.
        $\chol(\cdot)$ is the Cholesky decomposition. $\bbV$ is a diagonal matrix with diagonal elements of the square matrix $V$. $\calL_P(V)$ is the solution to the matrix linear system $\calL_P[V] P+ P \calL_P[V]=V$, known as the Lyapunov operator. $\mu=\frac{1}{n}\left( \frac{1}{\sqrt{\alpha+n\beta}}-\frac{1}{\sqrt{\alpha}} \right)$, $K=\chol(S)$ and $L_{ij}=\chol(P_{ij})$.
        \item
        $\inner{\cdot}{\cdot}$ and $\alphabetainner{\cdot}{\cdot}$ are the Frobenius inner product and the $\orth{n}$-invariant inner product defined in \cref{eq:oim_sym}.
        \item
        Due to the incompleteness of PEM and BWM, there are constraints on $V^\PE $ and $V^\BW$: $I+ \theta V^\PE \in \spd{m}$ and $I+\frac{1}{2}V^\BW \in \spd{n}$. Both constraints can be solved by numerical regularization, as detailed in \cref{app:rmk:spd_fc_constrains}.
    \end{itemize}
\end{theorem}

The Euclidean affine FC $y=Ax+b$ incorporates the linear map $y=Ax$, the most natural map between linear spaces. As shown by \citet[Sec. 4.4]{arsigny2005fast} and \citet[Thm. 1]{chen2024adaptive}, the SPD manifold admits two vector space structures with respect to LEM and LCM. Similar to the Euclidean FC layer, our SPD FC layer also incorporates linear maps over these vector structures. Denoting the addition and scalar product by $\oplusle$ ($\opluslc$) and $\odotle$ ($\odotlc$), respectively, which are detailed in \cref{app:subsec:prf:spd_fc_lie_hom}, we have the following result.

\begin{proposition}
    \label{prop:spd_fc_lie_hom}
    \linktoproof{prop:spd_fc_lie_hom}
    The LEM- and LCM-SPD FC layers incorporate the linear homomorphisms over the vector spaces $\{ \spd{n}, \oplusle, \odotle \}$ and $\{ \spd{n}, \opluslc, \odotlc\}$, respectively.
\end{proposition}

\begin{wraptable}{r}{0.6\textwidth}
  \centering
  \vspace{-4mm}
  \caption{Comparison with the existing SPD FC layers.}
  \label{tab:comp_spd_fc_layers}%
  \resizebox{\linewidth}{!}{
    \begin{tabular}{c|ccc}
    \toprule
    SPD FC Layer & Geometries & Requirement & Incorporated by Ours \\
    \midrule
    Gyro FC \citep{nguyen2024matrix} & AIM, LEM \& LCM & Gyrovector & \cmark (\cref{app:sec:relation_gyro_spd_fc}) \\
    \midrule
    Flat FC \citep{nguyen2025symmetric} & LEM \& LCM & Flat geometry & \cmark (\cref{app:sec:relation_flat_spd_fc}) \\
    \midrule
    Symmetric FC \citep{nguyen2025symmetric} & AIM & \makecell{Invariant metric \\ Symmetric space} & \na \\
    \midrule
    \rowcolor{HilightColor} Ours  & Riemannian spaces & Riemannian & \na \\
    \bottomrule
    \end{tabular}
    }
    \vspace{-5mm}
\end{wraptable}
\textbf{Comparison.}
As summarized in \cref{tab:comp_spd_fc_layers}, the three gyro SPD FC layers from \citet[Props. 3.4-3.6]{nguyen2024matrix} and the two flat SPD FC layers \citep{nguyen2025symmetric} are incorporated by our SPD FC layers.

\textbf{Simplification and convolution.}
Following the trivialization in \cref{subsec:fc_parameters}, the SPD FC layers under LEM, AIM, LCM, and PEM can be further simplified, as detailed in \cref{app:sec:simplified_sdp_fc}. The convolution is defined as \cref{sec:riem_conv_layers}, with $\calM=\spd{m}$ and $\calN=\spd{n}$.

\subsection{Grassmannian matrix manifolds}
We instantiate our FC layers over the ONB and PP Grassmannian and define \gls{GrConv} according to \cref{sec:riem_conv_layers}. We then compare GrConv with existing popular Grassmannian transformations, concluding that GrConv is more flexible in both dimensionality and geometry.

\textbf{ONB.}
We denote by $\idonb=
\begin{bmatrix}
    I_{p}, \bbzero
\end{bmatrix}^\top \in \bbR{n \times p}$, with $I_p$ being the $p \times p$ identity matrix. We define it as the Grassmannian origin, since it corresponds to $I_n \in \orth{n}$ in the quotient structure \citep[Sec. 2.2]{bendokat2024grassmann}. As in \cref{subsec:fc_parameters}, the FC parameters are modeled by parallel transport and the Riemannian exponential map at $\idonb$. The concrete ONB Grassmannian FC layer can be further simplified.

\begin{theorem}[ONB]
    \label{thm:gras_fc_onb}
    \linktoproof{thm:gras_fc_onb}
    Given $U \in \grasonb{p,n}$, the ONB Grassmannian FC layer $\calF(\cdot): \grasonb{p,n} \rightarrow \grasonb{q,m}$ is
    $Y =
    \left(\begin{array}{c}
    R \cos(\Sigma) R^\top \\
    O \sin(\Sigma) R^\top
    \end{array}\right)$,
    with $B^\onb \stackrel{\text { SVD }}{:=} O \Sigma R^\top  \in \bbR{(m-q) \times q}$. Each $(i,j)$ element of $B^\onb \in \bbR{(m-q) \times q}$ is $\inner{\rielog^{\onb}_{P_{ij}}(U) }{T_{ij} B_{Z_{ij}}}$, with
    $T_{ij}=
    \left(\begin{array}{c}
    -R_{ij} \sin(\Sigma_{ij}) O_{ij}^\top \\
    O_{ij} \cos(\Sigma_{ij}) O_{ij}^\top + I_{n-p} - O_{ij} O_{ij}^\top
    \end{array}\right)$.
    Here, $\gamma_{ij} [B_{Z_{ij}}] \stackrel{\text { SVD }}{:=} O_{ij} \Sigma_{ij} R_{ij}^\top$ is the SVD decomposition. The FC parameters are $B_{Z_{ij}} \in \bbR{(n-p) \times p}$ and $\gamma_{ij} \in \bbRscalar$ for $1 \leq i \leq m-q$ and $1 \leq j \leq q$.
\end{theorem}

\textbf{PP.}
We define the PP origin as $\idpp=\idonb \idonb^\top$, since it corresponds to $\idonb$ \citep[Eq. 2.11]{bendokat2024grassmann}. Similarly, we model the FC parameters by parallel transport and the Riemannian exponential map at $\idpp$. The PP Grassmannian FC layer can be further simplified. In addition, the Riemannian logarithm under the PP Grassmannian can be calculated using the ONB logarithm to support auto-differentiation \citep[Prop. 3.12]{nguyen2024matrix}. For more details, please refer to the proof of the following theorem.

\begin{theorem}[PP]
    \label{thm:gras_fc_pp}
    \linktoproof{thm:gras_fc_pp}
    Given $X \in \graspp{p,n}$, the PP Grassmannian  FC layer $\calF(\cdot): \graspp{p,n} \rightarrow \graspp{q,m}$ is $Y = \widetilde{U} \widetilde{U}^\top$, with
    $\widetilde{U} =
    \left( \exp \left(\left(\begin{array}{cc}
        0 & -(B^\pp)^T \\
        B^\pp & 0
    \end{array}\right) \right) \right)_{1:q}$,
    where $(\cdot)_{1:q}$ returns the first-$q$ columns of the input square matrix. Each $(i,j)$ element of $B^\pp \in \bbR{(m-q) \times q}$ is defined as $\frac{1}{2} \inner{ \pi_{*,\pi(P)} \left(\rielog^{\onb}_{(O_{ij})_{1:p}} (\pi^{-1}(X)) \right)}{O_{ij} Z_{ij} O_{ij}^\top}$, with
    $O_{ij} =
    \exp \left(\left(\begin{array}{cc}
            0 & -(\gamma_{ij} [B_{Z_{ij}}])^T \\
            \gamma_{ij} [B_{Z_{ij}}] & 0
    \end{array}\right) \right)$,
     where $\pi(U) = UU^\top$, and $\pi_{*,U} (V) = UV^\top + VU^\top$ is the differential map for all $U \in \grasonb{p,n}$ and $V \in T_{U}\grasonb{p,n}$. The FC parameters are $B_{Z_{ij}} \in \bbR{(n-p) \times p}$ and $\gamma_{ij} \in \bbRscalar$ for $1 \leq i \leq m-q$ and $1 \leq j \leq q$.
\end{theorem}

\textbf{Comparison.}
\citet{huang2018building} proposed FRMap + ReOrth layers to perform transformations over the ONB Grassmannian via left matrix product (FRMap) and QR decomposition (ReOrth). \citet{nguyen2022gyro} proposed PP scaling for the PP Grassmannian using the tangent space at the identity. \citet{nguyen2023building} extended PP scaling to the ONB Grassmannian. 
\begin{wraptable}{r}{0.45\textwidth}
    \centering
    \caption{Our GrConv against the existing Grassmannian transformation layers.}
    \label{tab:grconv_vs_others}%
    \resizebox{0.99\linewidth}{!}{
    \begin{tabular}{c|c|ccc}
        \toprule
        \multirow{2}{*}{Methods} & \multirow{2}{*}{Perspective} & \multicolumn{3}{c}{Flexible dimensions} \\
        \cmidrule{3-5} & & Subspace $p$ & Ambient $n$ & Channel \\
        \midrule
        \makecell{FRMap + ReOrth \\ \citep[Eqs. (2-4)]{huang2018building}}  & ONB & \xmark     &   \cmark  & \xmark \\
        \midrule
        \makecell {PP Scaling  \\ \citep[Sec. 4.2.2]{nguyen2022gyro}}
        &  PP &  \xmark    &  \xmark     & \xmark \\
        \midrule
        \makecell{ONB Scaling \\ \citep[Sec. 3.2]{nguyen2023building}}
        &  ONB &  \xmark    &  \xmark     & \xmark \\
        \midrule
        \makecell{GrTrans \\ \citep[Sec. 2.3.2]{nguyen2023building}}
        & ONB + PP & \xmark     &  \xmark     & \xmark \\
        \midrule
        \rowcolor{HilightColor} GrConv &  ONB + PP & \cmark      &  \cmark    & \cmark \\
        \bottomrule
    \end{tabular}%
    }
    \vspace{-8mm}
\end{wraptable}
In addition, \citet{nguyen2023building} used gyrogroup left translation (GrTrans) as a transformation. These layers are briefly recapped in \cref{app:sec:review_grass_trans_layers}. However, these previous layers fail to faithfully respect the Grassmannian geometries and lack flexibility regarding dimensions and perspectives. In contrast, given a $c$-channel Grassmannian $\grasonb{p,n}$ (or $\graspp{p,n}$) input, our GrConv can adjust all dimensions across both perspectives, enabling more flexibility. \cref{tab:grconv_vs_others} summarizes the above discussion.

\subsection{Manifold embedding}
In several applications \citep{chami2019hyperbolic,lopez2021vector,zhao2023modeling,nguyen2024matrix}, Euclidean features are embedded into the manifold via $\rieexp_{E}(Ax+b)$. As detailed in \cref{app:subsec:manifold_embedding}, our framework implies that this operation follows the Riemannian FC layer between the Euclidean space and the target manifold, \ie, $\calF(\cdot): \bbR{n} \to \calM$.

\section{Experiments}
\label{sec:experiments}

We evaluate the effectiveness of our layers on different manifolds. We refer the reader to \cref{app:subsec:exp_details_hyperbolic,app:subsec:exp_details_spd,app:subsec:exp_details_grass} for experimental details on the hyperbolic, SPD, and Grassmannian spaces, respectively.

\subsection{Experiments on the hyperbolic manifold}

We compare our HFC layers against other hyperbolic transformation layers, including Möbius \citep{ganea2018hyperbolic} and Einstein \citep{mao2024klein} transformations via tangent spaces, LFC \citep{chen2022fully} via spacetime, NestFC \citep{fan2022nested} via nested projections, and the Poincaré FC layer \citep{shimizu2021hyperbolic}. Following \citet{chami2019hyperbolic}, we adopt four graph datasets for the link prediction task: Disease \citep{anderson1991infectious}, Airport \citep{zhang2018link}, Pubmed \citep{namata2012query}, and Cora \citep{sen2008collective}.

\begin{table}[t]
  \centering
  \vspace{-3mm}
  \caption{Comparison of hyperbolic FC layers under the same two-layer HNN, reporting AUC (\%), full-model registered parameter elements, and peak allocated memory (MiB). The top 3 AUC results are highlighted with \ColFirst{red}, \ColSecond{blue}, and \ColThird{cyan}. The worst value for each resource is underlined.}
  \label{tab:hyperbolic_results}%
  \resizebox{\textwidth}{!}{
    \begin{tabular}{c|c|c|ccc|ccc|ccc|ccc}
    \toprule
    \multirow{2}[2]{*}{Method} & \multirow{2}[2]{*}{Mechanism} & \multirow{2}[2]{*}{Space} & \multicolumn{3}{c|}{Disease ($\delta=0$)} & \multicolumn{3}{c|}{Airport ($\delta=1$)} & \multicolumn{3}{c|}{Pubmed ($\delta=3.5$)} & \multicolumn{3}{c}{Cora ($\delta=11$)} \\
    & & & AUC & \#Param & MiB & AUC & \#Param & MiB & AUC & \#Param & MiB & AUC & \#Param & MiB \\
    \midrule
    Möbius \citep{ganea2018hyperbolic} & Tangent & $\pball{n}$ & 76.73 ± 4.86 & 464 & 77.20 & 93.26 ± 0.43 & 480 & 123.84 & \ColThird{94.95 ± 0.06} & 8,288 & 3273.19 & 90.75 ± 0.47 & 23,216 & 176.57 \\
    Einstein \citep{mao2024klein} & Tangent & $\klein{n}$ & 77.34 ± 2.56 & 464 & 77.32 & 92.72 ± 0.07 & 480 & 125.70 & \ColFirst{94.99 ± 0.13} & 8,288 & 3273.19 & 89.73 ± 0.21 & 23,216 & 176.57 \\
    LorentzTan & Tangent & $\bbh{n}$ & 76.54 ± 0.83 & 464 & 82.03 & 92.70 ± 0.19 & 480 & 125.46 & \ColSecond{94.99 ± 0.13} & 8,288 & 3653.47 & 89.73 ± 0.21 & 23,216 & 326.61 \\
    LFC \citep{chen2022fully} & Spacetime & $\bbh{n}$ & 78.00 ± 0.60 & 563 & 77.28 & 92.63 ± 0.27 & 580 & 118.01 & 94.22 ± 0.11 & 8,876 & 3653.48 & \ColSecond{91.74 ± 0.12} & 24,737 & 326.62 \\
    NestFC \citep{fan2022nested} & Nested projection & $\bbh{n}$ & 71.30 ± 3.50 & \underline{811} & 78.35 & 83.19 ± 0.90 & \underline{850} & 119.20 & 88.07 ± 0.75 & \underline{258,514} & \underline{3659.20} & 86.08 ± 0.84 & \underline{2,076,931} & \underline{1234.25} \\
    Poincaré FC \citep{shimizu2021hyperbolic} & Riemannian & $\pball{n}$ & 79.19 ± 2.05 & 528 & 78.72 & \ColThird{94.21 ± 0.43} & 544 & 127.80 & 94.30 ± 0.19 & 8,352 & 3294.38 & 87.16 ± 0.95 & 23,280 & 176.94 \\
    \midrule
    \rowcolor{HilightColor} HFC-P & Riemannian & $\pball{n}$ & \ColSecond{80.66 ± 1.35} & 496 & 83.42 & 94.13 ± 0.47 & 512 & 133.88 & 94.77 ± 0.28 & 8,320 & 3314.73 & \ColThird{90.76 ± 0.57} & 23,248 & 176.57 \\
    \rowcolor{HilightColor} HFC-K & Riemannian & $\klein{n}$ & \ColThird{80.58 ± 1.59} & 496 & 85.10 & \ColSecond{94.32 ± 0.27} & 512 & \underline{137.20} & 94.61 ± 0.27 & 8,320 & 3329.29 & 89.94 ± 0.49 & 23,248 & 176.57 \\
    \rowcolor{HilightColor} HFC-H & Riemannian & $\bbh{n}$ & \ColFirst{82.93 ± 0.89} & 496 & \underline{89.82} & \ColFirst{95.18 ± 0.18} & 512 & 134.76 & 93.99 ± 0.25 & 8,320 & 3653.48 & \ColFirst{92.20 ± 0.25} & 23,248 & 326.61 \\
    \bottomrule
    \end{tabular}%
    }
    \vspace{-5mm}
\end{table}%

\textbf{Results on HNN.}
Following the HNN implementations \citep{ganea2018hyperbolic,chami2019hyperbolic,mao2024klein}, we compare different transformation layers using the same two-layer HNN backbone. For every method, the first transformation maps the input feature dimension $n_{\mathrm{in}}$ to 16, and the second maps 16 to 16. Mimicking Möbius and Einstein transformations, we further implement the tangent transformation on the hyperboloid model, $\rielog_e (M\rielog_e(x))$, referred to as LorentzTan. \cref{tab:hyperbolic_results} presents the 5-fold average test AUC, number of parameters, and peak allocated GPU memory. We have the following key observations.
(1) \textbf{Effectiveness:}
Our HFC layers generally achieve superior performance over prior hyperbolic layers.
(2) \textbf{Hyperbolicity:}
Riemannian transformations outperform tangent or spacetime transformations on low-$\delta$ Disease and Airport. On the higher-$\delta$ datasets, tangent-space transformations remain competitive on Pubmed, whereas HFC-H performs best on Cora. These results suggest that the effectiveness of tangent-space approximations depends on data hyperbolicity.
(3) \textbf{Representation power \& metrics:}
The optimal models vary across datasets. Among our HFC variants, HFC-P performs best on Pubmed, whereas HFC-H performs best on Disease, Airport, and Cora, demonstrating the benefit of adapting the FC layer across hyperbolic models.
(4) \textbf{Parameter and memory efficiency.}
NestFC is the least parameter-efficient and the least memory-efficient for high-dimensional inputs, due to its matrix-manifold parameters. On Cora, it uses $\mathbf{89.3\times}$ as many parameters and $\mathbf{3.8\times}$ as much peak memory as HFC-H. Further analysis in \cref{app:subsubsec:hyperbolic_complexity} shows that our HFC layers have the same $O(Bnm)$ asymptotic complexity and comparable parameter dimensions as most existing hyperbolic FC layers, while NestFC has the highest asymptotic complexity and the largest parameter dimension.

\begin{table}[htbp]
  \centering
  \vspace{-3mm}
  \caption{Comparison of hyperbolic transformations under different settings of Poincaré RResNet.}
    \label{tab:exp_rresnet}%
    \resizebox{0.85\linewidth}{!}{
    \begin{tabular}{c|c|ccc|ccc}
    \toprule
    \multirow{2}[4]{*}{Dataset} & Number of horospheres & \multicolumn{3}{c|}{50} & \multicolumn{3}{c}{250} \\
\cmidrule{2-8}          & Dim   & 8     & 16    & 32    & 8     & 16    & 32 \\
    \midrule
    \multirow{4}[4]{*}{Disease} & RResNet \citep{katsman2024riemannian} & 76.0 ± 1.7 & 78.0 ± 2.2 & 77.4 ± 2.2 & 71.5 ± 5.1 & 78.1 ± 3.3 & 76.5 ± 2.5 \\
          & Möbius+RResNet & 74.6 ± 1.9 & 74.6 ± 5.7 & 75.1 ± 2.1 & 74.0 ± 2.7 & 71.0 ± 5.2 & 73.3 ± 3.4 \\
          & Poincaré FC+RResNet & 80.4 ± 0.7 & 79.1 ± 1.8 & 79.1 ± 1.6 & 80.6 ± 0.8 & 79.1 ± 0.7 & 80.1 ± 1.4 \\
\cmidrule{2-8}          & \cellcolor{HilightColor} HFC-P+RResNet & \cellcolor{HilightColor} \cellcolor{HilightColor} \ColFirst{81.1 ± 0.6} & \cellcolor{HilightColor} \cellcolor{HilightColor} \ColFirst{80.0 ± 0.4} & \cellcolor{HilightColor} \ColFirst{81.0 ± 0.6} & \cellcolor{HilightColor} \ColFirst{80.9 ± 0.6} & \cellcolor{HilightColor} \ColFirst{82.3 ± 0.6} & \cellcolor{HilightColor} \ColFirst{82.1 ± 0.3} \\
    \midrule
    \multirow{4}[4]{*}{Airport} & RResNet \citep{katsman2024riemannian} & 93.4 ± 1.1 &  92.6 ± 1.1 &  93.0 ± 0.2 & 93.0 ± 0.4 & 93.0 ± 1.6 & 89.6 ± 4.7 \\
          & Möbius+RResNet & 92.9 ± 0.5 & 93.0 ± 0.3 & 92.6 ± 0.3 & 92.9 ± 0.1 & 93.2 ± 0.2 & 92.9 ± 0.6 \\
          & Poincaré FC+RResNet & 92.8 ± 0.6 & 93.4 ± 0.6 & 93.8 ± 0.4 & 93.5 ± 0.4 & 93.1 ± 0.4 & 93.8 ± 0.7 \\
\cmidrule{2-8}          & \cellcolor{HilightColor} HFC-P+RResNet & \cellcolor{HilightColor} \ColFirst{94.1 ± 0.5} & \cellcolor{HilightColor} \ColFirst{93.5 ± 0.3} & \cellcolor{HilightColor} \ColFirst{94.8 ± 0.5} & \cellcolor{HilightColor} \ColFirst{94.1 ± 0.6} & \cellcolor{HilightColor} \ColFirst{94.0 ± 0.4} & \cellcolor{HilightColor} \ColFirst{94.3 ± 0.4} \\
    \midrule
    \multirow{4}[4]{*}{Cora} & RResNet \citep{katsman2024riemannian} & \ColFirst{86.7 ± 1.2} & 87.2 ± 1.4 & 82.4 ± 3.5 & 82.7 ± 3.0 & 84.0 ± 3.7 & 83.3 ± 1.6 \\
          & Möbius+RResNet & 84.6 ± 2.9 & 86.8 ± 2.1 & 83.1 ± 2.5 & 84.1 ± 2.4 & 83.2 ± 1.6 & 83.9 ± 2.9 \\
          & Poincaré FC+RResNet & 83.8 ± 2.4 & 84.6 ± 0.9 & 83.3 ± 2.7 & 82.8 ± 3.3 & 82.8 ± 3.6 & 83.3 ± 3.3 \\
\cmidrule{2-8}          & \cellcolor{HilightColor} HFC-P+RResNet & \cellcolor{HilightColor} 85.6 ± 0.8 & \cellcolor{HilightColor} \ColFirst{87.6 ± 0.8} & \cellcolor{HilightColor} \ColFirst{87.2 ± 1.8} & \cellcolor{HilightColor} \ColFirst{87.68 ± 1.81} & \cellcolor{HilightColor} \ColFirst{86.08 ± 1.72} & \cellcolor{HilightColor} \ColFirst{86.97 ± 1.04} \\
    \bottomrule
    \end{tabular}%
    }
    \vspace{-3mm}
\end{table}%

\textbf{Results on RResNet.}
We conduct ablations on the RResNet backbone \citep{katsman2024riemannian}. Since hyperbolic RResNet is built on the Poincaré ball, we compare Poincaré transformation layers, \ie, Möbius, Poincaré FC, and our HFC-P. In the vanilla RResNet, inputs are first projected to the target dimension using a Euclidean linear layer, followed by mapping to the hyperbolic space and processing with hyperbolic residual blocks. In contrast, we first map the input to the hyperbolic space and then apply a hyperbolic transformation layer before feeding it into the residual blocks. This transformation layer can be instantiated as Möbius, Poincaré FC, or our HFC-P layer. We perform experiments across various configurations of the residual blocks, varying both the hidden dimensions and the number of horospheres. \cref{tab:exp_rresnet} presents the 5-fold average AUC results. Our HFC-P generally outperforms other hyperbolic transformations, demonstrating its effectiveness.

\textbf{Results on NHGCN.}
We further evaluate HFC within NHGCN \citep{fan2022nested}, a hyperbolic graph convolutional network. \Cref{app:subsubsec:nhgcn_comparison} shows that HFC improves NHGCN with substantially fewer parameters.

\begin{table}[htbp]
  \centering
  \vspace{-3mm}
  \caption{Comparison of our SPDNNs against other SPD networks. The ones highlighted with \colorbox{HilightColortwo}{\phantom{ora}} are our special cases, while those marked with $^*$ are reproduced by us because official code is unavailable.}  \label{tab:results_spdnns}
  \resizebox{0.8\linewidth}{!}{
  \begin{tabular}{ccccc}
    \toprule
    Methods & Radar & HDM05 & FPHA  & NTU60 \\
    \midrule
    SPDNet \citep{huang2017riemannian} & 93.25 ± 1.10 & 64.57 ± 0.61 & 85.59 ± 0.72 & 66.36 ± 0.72 \\
    SPDNetBN \citep{brooks2019riemannian} & 94.85 ± 0.99 & 71.28 ± 0.79 & 89.33 ± 0.49 & 69.38 ± 0.84 \\
    RResNet-AIM \citep{katsman2024riemannian} & 95.71 ± 0.37 & 64.95 ± 0.82 & 86.63 ± 0.55 & 70.70 ± 3.81 \\
    RResNet-LEM \citep{katsman2024riemannian} & 95.89 ± 0.86 & 70.12 ± 2.45 & 85.07 ± 0.99 & 74.67 ± 2.89 \\
    SPDNetLieBN-AIM \citep{chen2024liebn} & 95.47 ± 0.90 & 71.83 ± 0.69 & 90.39 ± 0.66 & 73.34 ± 0.40 \\
    SPDNetLieBN-LCM \citep{chen2024liebn} & 94.80 ± 0.71 & 71.78 ± 0.44 & 86.33 ± 0.43 & 72.54 ± 1.09 \\
    SPDNetMLR \citep{chen2024rmlr} & 95.64 ± 0.83 & 65.90 ± 0.93 & 85.67 ± 0.69 & 74.18 ± 1.24 \\
    GyroLE$^*$ \citep{nguyen2023building} & 97.31 ± 0.68 & 73.17 ± 0.37 & 90.73 ± 0.92 & 82.65 ± 0.20 \\
    GyroLC$^*$ \citep{nguyen2023building} & 95.23 ± 0.66 & 67.53 ± 0.85 & 76.10 ± 0.63 & 78.32 ± 0.92 \\
    GyroAI$^*$ \citep{nguyen2023building} & 97.17 ± 0.67 & 72.34 ± 1.06 & 89.60 ± 0.37 & \textcolor{cyan}{\textbf{83.71 ± 0.32}} \\
    \midrule
    \rowcolor{HilightColortwo} GyroSPD++-AIM$^*$ \citep{nguyen2024matrix} & 97.01 ± 0.40 & 69.82 ± 1.79 & 89.50 ± 0.37 & 83.14 ± 0.87 \\
    \rowcolor{HilightColortwo} GyroSPD++-LEM$^*$ \citep{nguyen2024matrix} & 98.08 ± 0.26 & 77.63 ± 1.01 & 88.23 ± 0.62 & \textcolor{blue}{\textbf{85.48 ± 1.10}} \\
    \rowcolor{HilightColortwo} GyroSPD++-LCM$^*$ \citep{nguyen2024matrix} & 97.28 ± 0.78 & 75.36 ± 1.08 & 81.83 ± 0.93 & 74.64 ± 2.49 \\
    \midrule
    \rowcolor{HilightColor} SPDNN-LEM & \textcolor{cyan}{\textbf{98.27 ± 0.48}} & \textcolor{red}{\textbf{81.16 ± 0.93}} & \textcolor{red}{\textbf{91.83 ± 0.41}} & \textcolor{red}{\textbf{86.72 ± 0.14}} \\
    \rowcolor{HilightColor} SPDNN-AIM & 97.63 ± 0.50 & \textcolor{blue}{\textbf{80.12 ± 0.78}} & \textcolor{blue}{\textbf{91.57 ± 0.40}} & 82.44 ± 0.18 \\
    \rowcolor{HilightColor} SPDNN-PEM & \textcolor{blue}{\textbf{98.43 ± 0.44}} & \textcolor{cyan}{\textbf{78.77 ± 0.45}} & 90.33 ± 0.37 & 82.61 ± 0.37 \\
    \rowcolor{HilightColor} SPDNN-LCM & 97.65 ± 0.75 & 75.42 ± 0.95 & \textcolor{cyan}{\textbf{91.33 ± 0.24}} & 83.39 ± 0.10 \\
    \rowcolor{HilightColor} SPDNN-BWM & \textcolor{red}{\textbf{98.72 ± 0.14}} & 72.49 ± 2.02 & 87.80 ± 0.56 & 82.64 ± 0.35 \\
    \bottomrule
  \end{tabular}
  }
  \vspace{-4mm}
\end{table}

\subsection{Experiments on the SPD manifold}

Following \citet{huang2017deep,brooks2019riemannian,katsman2024riemannian}, we use the Radar dataset \citep{brooks2019riemannian} for radar classification, and the HDM05 \citep{muller2007documentation}, FPHA \citep{garcia2018first}, and NTU60 \citep{shahroudy2016ntu} datasets for human action recognition. In line with \citet{nguyen2024matrix}, we focus on the mutual action in NTU60. Following \citet{wang2024grassatt,nguyen2024matrix}, we model each sample sequence as multi-channel SPD covariance matrices of shape $[c,n,n]$.

\textbf{SPDNN.}
Our SPDNN has an MLR layer \citep{chen2024rmlr} stacked on top of a convolutional layer. We denote by SPDNN-[Metric] the SPDNN using convolution under the specified metric. For SPDNN-LEM, -PEM, and -LCM, the MLR is based on the same metric as the convolution. Since the MLRs under AIM and BWM are less efficient \citep{chen2024rmlr}, we apply LEM MLR for SPDNN-AIM and -BWM. Moreover, we trivialize the SPD parameters in the
MLR as \cref{subsec:fc_parameters}, which can be further simplified (detailed
in \cref{app:sec:simplified_sdp_mlr}). Consequently, all parameters in the SPDNNs can be optimized by a Euclidean optimizer. We compare our networks against the following SOTA SPD networks: SPDNet \citep{huang2017deep}, SPDNetBN \citep{brooks2019riemannian}, LieBN \citep{chen2024liebn}, RResNet \citep{katsman2024riemannian}, MLR \citep{chen2024rmlr}, Gyro \citep{nguyen2023building}, and GyroSPD++ \citep{nguyen2024matrix}.

\begin{wraptable}{r}{0.5\textwidth}
  \centering
  \caption{The average training time per epoch of our SPDNNs compared with GyroSPD++. A full comparison of efficiency can be found in \cref{app:subsec:efficiency_spd}.
  }
  \label{tab:efficiency_spd}
    \resizebox{\linewidth}{!}{
    \begin{tabular}{c|c|cccc}
    \toprule
    Geometry & Method & Radar & HDM05 & FPHA  & NTU60 \\
    \midrule
    \multirow{2}[2]{*}{AIM} & \cellcolor{HilightColortwo} GyroSPD++ & \cellcolor{HilightColortwo} 5.09  & \cellcolor{HilightColortwo} 103.57  & \cellcolor{HilightColortwo} 66.35  & \cellcolor{HilightColortwo} 125.05 \\
    \cmidrule{2-6} & \cellcolor{HilightColor} SPDNN & \cellcolor{HilightColor} 4.84 & \cellcolor{HilightColor} 101.80  & \cellcolor{HilightColor} 65.42  & \cellcolor{HilightColor} 124.41 \\
    \midrule
    \multirow{2}[1]{*}{LEM} & \cellcolor{HilightColortwo} GyroSPD++ & \cellcolor{HilightColortwo} 0.99  & \cellcolor{HilightColortwo} 0.95  & \cellcolor{HilightColortwo} 0.66  & \cellcolor{HilightColortwo} 7.58 \\
    \cmidrule{2-6} & \cellcolor{HilightColor} SPDNN & \cellcolor{HilightColor} 0.86  & \cellcolor{HilightColor} 0.74  & \cellcolor{HilightColor} 0.63  & \cellcolor{HilightColor} 5.79 \\
    \midrule
    \multirow{2}[1]{*}{LCM} & \cellcolor{HilightColortwo} GyroSPD++ & \cellcolor{HilightColortwo} 0.66  & \cellcolor{HilightColortwo} 0.70  & \cellcolor{HilightColortwo} 0.37  & \cellcolor{HilightColortwo} 5.74 \\
    \cmidrule{2-6} & \cellcolor{HilightColor} SPDNN & \cellcolor{HilightColor} 0.65  & \cellcolor{HilightColor} 0.59  & \cellcolor{HilightColor} 0.35  & \cellcolor{HilightColor} 3.72 \\
    \bottomrule
    \end{tabular}%
    }
  \vspace{-3mm}
\end{wraptable}
\textbf{Results.}
\Cref{tab:results_spdnns,tab:efficiency_spd} reports the classification results and training efficiency, respectively. Our SPDNNs consistently outperform other SPD models. Specifically, SPDNNs exceed the classic SPDNet by up to \textbf{5.47\%, 16.59\%, 6.24\%, and 20.36\%} on the four datasets, respectively. Despite sharing the same architecture, SPDNN generally outperforms GyroSPD++ in both accuracy and efficiency across LEM, LCM, and AIM. This advantage arises because our trivialization not only simplifies the expression of the FC and MLR layers but also mitigates the over-parameterization in GyroSPD++. In GyroSPD++, each output dimension of the FC layer requires two matrix parameters, whereas our approach uses only one matrix and one scalar parameter. This reduction in parameter complexity leads to improved training efficiency and generalization. Furthermore, the variation in optimal metrics across datasets underscores the flexibility of our methods.

\subsection{Experiments on the Grassmannian}
\begin{wraptable}{r}{0.6\textwidth}
    \centering
    \vspace{-3mm}
    \caption{Comparison of GrNNs against other Grassmannian networks on the Radar dataset. Those marked with $^*$ are reproduced by us because official code is unavailable.}
    \label{tab:results_grnns}
    \resizebox{\linewidth}{!}{
    \begin{tabular}{c|c|c|c}
        \toprule
        Method & Subspace dims & Ambient dims  & Mean ± Std \\
        \midrule
        GrNet \citep{huang2018building} & 4     & 20-->16 & 90.48 ± 0.76 \\
        GyroGr-Scaling$^*$ \citep{nguyen2023building} & 4     & 20-->20 & 88.88 ± 1.52 \\
        GyroGr$^*$ \citep{nguyen2023building} & 4     & 20-->20 & 90.64 ± 0.57 \\
        \midrule
        \multirow{4}{*}{GrNN-ONB}
        & \cellcolor{HilightColor}4-->4 & \cellcolor{HilightColor}20-->16 & \cellcolor{HilightColor}93.92 ± 0.74 \\
        & \cellcolor{HilightColor}4-->4 & \cellcolor{HilightColor}20-->20 & \cellcolor{HilightColor}92.83 ± 0.66 \\
        & \cellcolor{HilightColor}4-->6 & \cellcolor{HilightColor}20-->16 & \cellcolor{HilightColor}\textcolor{red}{\textbf{95.23 ± 0.96}} \\
        & \cellcolor{HilightColor}4-->8 & \cellcolor{HilightColor}20-->16 & \cellcolor{HilightColor}\textcolor{blue}{\textbf{94.77 ± 0.81}} \\
        \midrule
        \multirow{4}{*}{GrNN-PP}
        & \cellcolor{HilightColor}4-->4 & \cellcolor{HilightColor}20-->16 & \cellcolor{HilightColor}94.35 ± 0.42 \\
        & \cellcolor{HilightColor}4-->4 & \cellcolor{HilightColor}20-->20 & \cellcolor{HilightColor}\textcolor{cyan}{\textbf{94.56 ± 0.58}} \\
        & \cellcolor{HilightColor}4-->6 & \cellcolor{HilightColor}20-->16 & \cellcolor{HilightColor}94.51 ± 0.53 \\
        & \cellcolor{HilightColor}4-->8 & \cellcolor{HilightColor}20-->16 & \cellcolor{HilightColor}94.11 ± 0.58 \\
        \bottomrule
    \end{tabular}
    }
    \vspace{-3mm}
\end{wraptable}
We compare our Grassmannian convolutional layer against previous transformation layers, such as FRMap + ReOrth \citep{huang2018building}, scaling \citep{nguyen2023building}, and GrTrans \citep{nguyen2023building}. Compared with these previous layers, our transformation can more faithfully respect Grassmannian geometries while allowing greater flexibility with respect to dimensions and geometries. Following \citet{nguyen2023building}, each network consists of one transformation layer followed by classification. The corresponding models are denoted by GrNet \citep{huang2018building}, GyroGr-Scaling \citep{nguyen2023building}, GyroGr \citep{nguyen2023building}, GrNN-ONB, and GrNN-PP, respectively. Since our GrConv allows a more flexible change in dimensionality, we also perform ablations on the subspace and ambient dimensions of the output of the FC transformation. The experiments are conducted on the Radar dataset. Following \citet{wang2024grassatt}, we model each radar signal as a multi-channel Grassmannian tensor, \ie, $[c,n,p]$ for ONB and $[c,n,n]$ for PP. \cref{tab:results_grnns} reports the five-fold averages: GrConv outperforms the baselines, while varying the subspace dimension yields the two best results.

\section{Conclusion}
This paper extends fundamental FC and convolutional layers to operate on Riemannian manifolds. Our approach offers a naturally geometry-aware generalization that is more broadly applicable than previous work. Several existing Riemannian FC layers are subsumed within our framework as special cases. Empirically, we instantiate our framework across ten different geometries, including three hyperbolic models, five SPD geometries, and two Grassmannian formulations. Extensive experiments on radar classification, human action recognition, and graph link prediction demonstrate the effectiveness and flexibility of our approach. We expect this work to facilitate further advances in deep learning on Riemannian spaces.

\begin{ack}
This work was supported by a DAAD Research Grant in Germany (No. 57811724), an ELIZA PhD Mobility Scholarship, an ELSA Mobility Grant, and an ELIAS Mobility Grant. The author also acknowledges the support of CINECA and the EuroHPC Joint Undertaking for granting access to Leonardo at CINECA, Italy.
\end{ack}
\bibliographystyle{plainnat}
\bibliography{ref.bib}



\newpage
\appendix
\onecolumn
\startcontents[appendices]
\printcontents[appendices]{l}{1}{\section*{Appendix Contents}}
\newpage


\printglossary[type=acronym, title={List of acronyms}]

\section{Use of large language models}
\label{app:sec:llm-usage}
Large Language Models (LLMs) were used primarily for language polishing and minor text editing. In limited cases, they also assisted in translating certain mathematical formulations into PyTorch code. All generated outputs were carefully reviewed and, where necessary, corrected by the authors. The authors take full responsibility for the final content of this paper.

\section{Limitations}
\label{app:sec:limitations}
Our framework is designed for computationally tractable Riemannian manifolds, where closed-form expressions for exponential and logarithmic maps are available. This includes many commonly used manifolds, such as hyperbolic, SPD, and Grassmannian spaces. However, in cases where the underlying manifold structure is unknown or lacks tractable Riemannian operators, our approach may not be directly applicable. In such scenarios, future work could explore numerical approximations of Riemannian operators or develop new paradigms for constructing transformation layers for intractable geometries.

\section{Glossary of symbols}
\label{app:sec:notations}

\cref{app:tab:sum_notaitons} summarizes all the notation in the main paper.

\begin{table}[t!]
    \centering
    \caption{Summary of notation.}
    \label{app:tab:sum_notaitons}
    \resizebox{0.9\linewidth}{!}{
    \begin{tabular}{cc}
        \toprule
        Notation & Explanation  \\
        \midrule
        $\{\calN, g^\calN\}$ & Riemannian manifold $\calN$ with Riemannian metric $g^\calN$\\
        $\{\calM, g^\calM\}$ & Riemannian manifold $\calM$ with Riemannian metric $g^\calM$\\
        $E$ & Origin of the manifold of interest\\
        $T_P\calM$ & Tangent space at $P \in \calM$\\
        $g_p(\cdot ,\cdot)$ or $\langle \cdot, \cdot \rangle_P$ & Riemannian metric at $P$ \\
        $\| \cdot \|_P$ & The norm induced by $\langle \cdot, \cdot \rangle_P$ on $T_P\calM$ \\
        $\dist(X,Y)$ & Geodesic distance between points $X$ and $Y$\\
        $\dist(X,H)$ & Point-to-hyperplane pseudo-distance from $X$ to $H$\\
        $\rielog_P$ & Riemannian logarithm at $P$\\
        $\rieexp_P$ & Riemannian exponential map at $P$\\
        $\pt{P}{Q}$ & Parallel transport from $P$ to $Q$ along the geodesic \\
        $f_{*,P}$ & Differential map of the smooth map $f$ at $P \in \calM$\\
        $\{B_i\}_{i=1}^m$ & Standard orthonormal basis over the $m$-dimensional $T_E\calM$ \\
        \midrule
        $\pball{n}, \klein{n}$ and $\bbh{n}$ & Hyperbolic models of Poincaré ball, Beltrami--Klein, and hyperboloid ($K<0$)\\
        $\bbR{n}$ & Euclidean space of $n$-dimensional vectors\\
        $\Linner{\cdot}{\cdot}$ & Lorentz inner product \\
        $\Moplus$ and $\Motimes$ & Möbius gyro addition and scalar product\\
        $\Eoplus$ and $\Eotimes$ & Einstein gyro addition and scalar product\\
        $\pi_{\klein{n} \to \pball{n}}$ and $\pi_{\pball{n} \to \klein{n}}$ & Riemannian isometries between Beltrami--Klein and Poincaré ball\\
        \midrule
        $\spd{n}$ & Space of $n \times n$ SPD matrices \\
        $\sym{n}$ & Euclidean space of $n \times n$ symmetric matrices \\
        $\tril{n}$ & Euclidean space of $n \times n$ lower triangular matrices \\
        $\langle \cdot, \cdot \rangle$ & Standard Frobenius inner product\\
        $\langle \cdot, \cdot \rangle^{(\alpha, \beta)}$ & $\orth{n}$-invariant Euclidean metric on $\sym{n}$ s.t. $\min (\alpha, \alpha+n \beta)>0$ \\
        $\Fnorm{\cdot}$ &  Frobenius Norm \\
        $\log$ & Matrix logarithm \\
        $\exp$ & Matrix exponential \\
        $P^{\theta}$ & Matrix power for SPD matrix $P$\\
        $\calL_{P}[\cdot]$ & Lyapunov operator by $P \in \spd{n}$ \\
        $\scrL$ & Cholesky decomposition\\
        $\dlog$ & Diagonal element-wise logarithm \\
        $\lfloor \cdot \rfloor$ & Strictly lower triangular part of a square matrix \\
        $\bbD(\cdot)$  & A diagonal matrix with diagonal elements from a square matrix\\
        \midrule
        $\grasonb{p,n}$ & Grassmannian under the ONB perspective\\
        $\graspp{p,n}$ & Grassmannian under the projector perspective\\
        $\qr(\cdot)$ & Return an orthogonal matrix by QR decomposition \\
        $[\cdot,\cdot]$ & Matrix commutator\\
        $\idonb$ & Grassmannian identity under the ONB perspective\\
        $\idpp$ & Grassmannian identity under the projector perspective\\
        $I_n$ & $n \times n$ identity matrix\\
        $\pi$ & Riemannian isometry from $\grasonb{p,n}$ onto $\graspp{p,n}$\\
        $\overline{(\cdot)}$ &  $\overline{(\cdot)} = \widetilde{\rielog}_{\idpp}(\cdot)$ with $\widetilde{\rielog}$ being the Riemannian logarithm on $\graspp{p,n}$ \\
        $\bbzero$ & Zero matrix or vector\\
        $\stiefel{p, n}$ & Stiefel manifold of $n \times p$ column-wise orthogonal matrices\\
        $\GL{n}$ & General linear group of $n \times n$ invertible matrices\\
        $\orth{n}$ & Orthogonal group of $n \times n$ orthogonal matrices\\
        \bottomrule
    \end{tabular}
    }
\end{table}

\section{Geometries of the involved vector and matrix manifolds}
\label{app:sec:geom_spd_gras_hyperbolic}

\subsection{Geometries of the hyperbolic space}
\label{app:subsec:geom_hyperbolic}

\begin{table}[t!]
    \centering
    \caption{Riemannian operators on the Poincaré ball and hyperboloid $(K<0)$.}
    \label{app:tab:riem_operators_hyperbolic}
    \resizebox{\linewidth}{!}{
    \begin{tabular}{ccc}
        \toprule
        Operator & $\pball{n} =\left\{x \in \mathbb{R}^{n} \mid \norm{x}^2 < -\frac{1}{K} \right\}$ & \makecell{$\bbh{n}=\left\{x \in \mathbb{R}^{n+1} \mid \Lnorm{x}^2 = \frac{1}{K}, x_1 > 0 \right\}$, \\ with $\Lnorm{x}^2 = \sum_{i=2}^{n+1} x _i ^2 - x _1 ^2$} \\
        \midrule
        $g_x(v,w)$ & \makecell{$(\lambda^K_x)^2 \inner{v}{w}$ \\ $\lambda^K_x =\frac{2}{\left(1+K\|x\|^2\right)}$} & $\Linner{v}{w} = \sum_{i=2}^{n+1} v_i w_i - v_1 w_1$ \\
        \midrule
        $\dist(x,y)$ &  $\frac{2}{\sqrt{|K|}} \tanh ^{-1}\left(\sqrt{|K|}\|- x \Moplus  y\| \right)$ & $\frac{1}{\sqrt{|K|}} \cosh ^{-1}\left(K \Linner{x}{y} \right)$ \\
        \midrule
        $\rielog_x (y)$ & $\frac{2}{\sqrt{|K|} \lambda_{x}^K} \tanh ^{-1}\left(\sqrt{|K|}\left\|-x \Moplus y\right\| \right) \frac{-x \Moplus y}{\left\|-x \Moplus y\right\| }$ & $\frac{\cosh^{-1}\left(K\langle x, y\rangle_\calL \right)}{\sinh \left(\cosh^{-1}\left(K\langle x, y\rangle_\calL \right)\right)}\left(y-K\langle x, y\rangle_\calL x\right)$ \\
        \midrule
        $\Gamma_{x \rightarrow y} (v)$
        & $\frac{\lambda_x^K}{\lambda_y^K} \gyr[y,-x] v$ & $v-\frac{K \langle y, v\rangle_\calL}{1+ K \langle x, y\rangle_\calL}(x+y)$ \\
        \midrule
        $\rieexp_x (v)$ & $x \Moplus\left(\tanh \left(\sqrt{|K|} \frac{\lambda_x^K\|v\|}{2}\right) \frac{v}{\sqrt{|K|}\|v\|}\right)$ & $\cosh \left(\sqrt{|K|} \Lnorm{v} \right) x+\sinh \left(\sqrt{|K|} \Lnorm{v} \right) \frac{v}{\sqrt{|K|} \Lnorm{v}}$ \\
        \midrule
        References & \citep{ganea2018hyperbolic,skopek2019mixed,ungar2022gyrovector} & \citep{petersen2006riemannian,skopek2019mixed} \\
        \bottomrule
    \end{tabular}
    }
\end{table}

\begin{table}[t!]
    \centering
    \caption{Riemannian operators on the Beltrami--Klein model $(K<0)$.}
    \label{app:tab:riem_operators_klein}
    \begin{tabular}{cc}
        \toprule
        Operators & $\klein{n} = \{x \in \bbR{n} \mid \|x\|^2 < -\tfrac{1}{K}\}$ \\
        \midrule
        $g_x(v,w)$ & $\frac{\inner{v}{w}}{1+ K \norm{x} ^2} - \frac{K \inner{x}{v} \inner{x}{w}}{\left(1 + K \norm{x} ^2\right)^2}$\\
        \midrule
        $\dist(x,y)$
        & $\frac{2}{\sqrt{-K}} \tanh^{-1}\left(
            \sqrt{-K}  \frac{\| -x \Eoplus y \|}{1+\sqrt{1+K\| -x \Eoplus y \|^2}}
        \right)$ \\
        \midrule
        $\rieexp_x(v)$
        & $x \Eoplus \rieexp_{\bbzero}\left(
            \frac{1}{\sqrt{1+K\|x\|^2}} v
            - \frac{K \langle x,v\rangle}{(1+\sqrt{1+K\|x\|^2})(1+K\|x\|^2)} x
        \right)$ \\
        \midrule
        $\rielog_x(y)$
        & $\frac{1}{\lambda_{\tilde{x}}^K}
             (\pi_{\pball{n}\to\klein{n}})_{*,\tilde{x}}
            \left( \rielog_{\bbzero}(-x \Eoplus y) \right),
            \quad \tilde{x} = \pi_{\klein{n}\to\pball{n}}(x)$ \\
        \midrule
        References & \citep{ungar2022analytic,chen2025gyrobnextension} \\
        \bottomrule
    \end{tabular}
\end{table}

There are five models over the hyperbolic space \citep{cannon1997hyperbolic}. We focus on the Poincaré ball, Beltrami--Klein, and hyperboloid models:
\begin{align}
    \text{Poincaré ball: }& \pball{n} =\left\{x \in \mathbb{R}^{n} \mid \norm{x}^2 < -\frac{1}{K} \right\}, \\
    \text{Beltrami--Klein: }& \klein{n} =\left\{x \in \mathbb{R}^{n} \mid \norm{x}^2 < -\frac{1}{K} \right\}, \\
    \text{Hyperboloid: }& \bbh{n}=\left\{x \in \mathbb{R}^{n+1} \mid \Lnorm{x}^2 = \frac{1}{K}, x_1>0 \right\},
\end{align}
where $\Lnorm{x}^2 = \sum_{i=2}^{n+1} x _i ^2 - x _1 ^2$ is the Lorentz inner product, and $\norm{\cdot}$ is the standard $L_2$ norm induced by the standard inner product $\inner{\cdot}{\cdot}$. Here, $K<0$ is the constant curvature. Although the set of the Poincaré ball is identical to that of the Beltrami--Klein model, their Riemannian metrics are different. In fact, each of the above models has its own Riemannian metric:
\begin{align}
    g ^{\mathbb{P}}_{x}(v,w) &= (\lambda^K_x)^2 \inner{v}{w}, \\
    g ^{\mathbb{K}}_{x}(v,w) &= \frac{\inner{v}{w}}{1+ K \norm{x} ^2} - \frac{K \inner{x}{v} \inner{x}{w}}{\left(1 + K \norm{x} ^2\right)^2}, \\
    g ^{\mathbb{H}}_{x}(v,w) &= \Linner{v}{w} = \sum_{i=2}^{n+1} v_i w_i - v_1 w_1,
\end{align}
where $\lambda^K_x =\frac{2}{\left(1+K\|x\|^2\right)}$ is a conformal factor.

As shown by \citet{ungar2022analytic}, both the Poincaré ball and Beltrami--Klein models admit gyrovector structures, which are the natural manifold counterparts of vector spaces. The Poincaré ball admits a Möbius gyrovector space \citep[Ch. 6.14]{ungar2022analytic}, while the Beltrami--Klein model admits an Einstein gyrovector space \citep[Ch. 6.18]{ungar2022analytic}. Denoting $\calH \in \{\pball{n}, \klein{n}\}$, for any $x,y \in \calH$ and $r \in \bbRscalar$, the gyro operations are defined as
\begin{align}
    \label{app:eq:mobius_addition}
    \text{Möbius addition}
    &:x \Moplus y = \frac{\left(1-2 K\langle x, y\rangle-K\|y\|^2\right) x+\left(1+K\|x\|^2\right) y}{1-2 K\langle x, y\rangle+K^2\|x\|^2\|y\|^2}, \\
    \text{Möbius scalar multiplication}
    &: r \Motimes x
    =\frac{\tanh \left(r \tanh ^{-1}\left( \sqrt{-K} \|x\| \right)\right)}{\sqrt{-K}}  \frac{x}{\|x\|}, \\
    \text{Einstein addition}
    &: x \Eoplus y
    =\frac{1}{1 - K\inner{x}{y}}\left(x+\frac{1}{\gamma_{x}} y -K \frac{\gamma_{x}}{1+\gamma_{x}} \inner{x}{y} x\right), \\
    \text{Einstein scalar multiplication}
    &: r \Eotimes x
    =\frac{\tanh \left(r \tanh ^{-1}\left( \sqrt{-K} \|x\| \right)\right)}{\sqrt{-K}}  \frac{x}{\|x\|}.
\end{align}
where $\gamma_{x}=1 / \sqrt{1+ K \|x\|^2}$ is called the gamma factor. Interestingly, the scalar gyromultiplications are identical under the Möbius and Einstein gyrovector spaces.

The Poincaré ball and hyperboloid admit closed-form Riemannian operators, as summarized in \cref{app:tab:riem_operators_hyperbolic}. The parallel transport over the Poincaré ball requires the notion of gyration \citep{ungar2022analytic}:
\begin{equation}
    \gyr[x, y] z=\ominus_\mathrm{M} \left(x \Moplus y\right) \Moplus\left(x \Moplus\left(y \Moplus z\right)\right), \forall x,y,z \in \pball{n}.
\end{equation}

\citet[Sec. 5.6]{chen2025gyrobn} studied the Riemannian structure over the Beltrami--Klein ball. The Beltrami--Klein ball is isometric to the Poincaré ball by
\begin{align}
        \pi_{\klein{n} \to \pball{n}}
        &: x \in \klein{n} \longmapsto
        \frac{1}{ 1 +  \sqrt{1 + K \|x\|^2} } x \in \pball{n},\\
        \pi_{\pball{n} \to \klein{n}}
        &: x \in \pball{n} \longmapsto \frac{2}{1 - K \|x\|^2} x \in \klein{n}.
\end{align}
Using the above isometries, \citet[Sec. 5.6]{chen2025gyrobn} introduced closed-form expressions for the Riemannian operators on the Beltrami--Klein ball, as summarized in \cref{app:tab:riem_operators_klein}. In particular, the Riemannian exponential and logarithmic maps at the zero vector $\bbzero$ are identical under the Beltrami--Klein and Poincaré ball models:
\begin{align}
    \label{app:eq:exp_0_klein}
    \rieexp_{\bbzero}(v) &= \tanh (\sqrt{|K|}\|v\|) \frac{v}{\sqrt{|K|}\|v\|}, \quad \forall v \in T_\bbzero \calH,\\
    \label{app:eq:log_0_klein}
    \rielog_{\bbzero}(x) &= \tanh ^{-1}(\sqrt{|K|}\|x\|) \frac{x}{\sqrt{|K|}\|x\|}, \quad \forall x \in \calH,
\end{align}
with $\calH \in \{ \klein{n}, \pball{n} \}$.

As shown by \citet[Secs. 5.4 and 5.6]{chen2025gyrobnextension}, both the Möbius and Einstein gyrovector operations can be expressed by their Riemannian geometries
\begin{align}
    \label{app:eq:klein-poincare-add-riem}
    x \oplus _{\calH} y &= \rieexp _x \left(\pt{{\bbzero}}{x} (\rielog _{\bbzero} (y)) \right), \\
    \label{app:eq:klein-poincare-prod-riem}
    t \otimes _{\calH} x &= \rieexp _{\bbzero} ( t \rielog _{\bbzero} (x)),
\end{align}
where $\oplus _{\calH}$ and $\otimes _{\calH}$ are the gyroaddition and gyromultiplication under the corresponding model.

\subsection{Geometries of the SPD manifold}
\label{app:subsec:geom_spd}

\cref{app:tab:riem_operators_lem_aim_pem,app:tab:riem_operators_bwm_lc}  summarizes the associated Riemannian operators and properties. Following \cref{app:tab:sum_notaitons}, we further make the following notation. Given any SPD points $P, Q \in \spd{n}$ and tangent vectors $V, W \in T_P\spd{n}$, we denote $\widetilde{V}=\chol_{*,P}(V)$, $\widetilde{W}=\chol_{*,P}(W)$, $L=\chol{P}$, and $K=\chol{Q}$.
The corresponding diagonal matrices with their diagonal elements are denoted by $\widetilde{\bbV}, \widetilde{\bbW}, \bbL$, and $\bbK$, respectively.
For the parallel transport under the BWM, we only present the case where $P, Q$ are commuting matrices, \ie $P=U\Sigma U^\top$ and $Q=U \Delta U^\top$.

The $\orth{n}$-invariant Euclidean metric on $\sym{n}$ \citep{thanwerdas2023n} is
\begin{equation} \label{eq:oim_sym}
    \langle V,W \rangle^{(\alpha, \beta)}=\alpha \langle V,W \rangle + \beta \tr(V)\tr(W), \quad \text{ with } \min (\alpha, \alpha+n \beta)>0.
\end{equation}

\begin{remark}
    \label{app:rmk:incomplete_spd}
    We make the following remarks with respect to the geometries on the SPD manifold.
    \begin{itemize}
        \item
        \textbf{PEM \& EM.}
        When the power equals 1, the associated PEM is reduced to the Euclidean Metric (EM) \citep[Sec. 3.1]{thanwerdas2023n}.
        \item
        \textbf{Incompleteness \& Riemannian exponential maps.}
        As PEM and BWM are incomplete, their Riemannian exponential maps are locally defined. As shown by \citet[Prop. 9]{malago2018wasserstein} and implied by \citet{chen2024rmlr,thanwerdas2023n}, the restricted domains are
        \begin{equation}
            \begin{aligned}
                \text{PEM: }&  P^{\theta}+ P_{\theta*,P} (V) \in \spd{n},\\
                \text{BWM: }&  \calL_{P}[V] + I \in \spd{n}.\\
            \end{aligned}
        \end{equation}
        The above restriction can be addressed numerically, for example by ReEig \citep{huang2017deep}:
        \begin{equation}
            \widetilde{S} = U \max(\epsilon I, \Sigma) U^\top,
        \end{equation}
        where $S\stackrel{\text{Eig}}{:=} U \Sigma U^\top$ is the Eigendecomposition.
    \end{itemize}

\end{remark}

\begin{table}[htbp]
    \centering
    \caption{The Riemannian operators under LEM, AIM, and PEM on the SPD manifold.}
    \label{app:tab:riem_operators_lem_aim_pem}
    \resizebox{0.95\linewidth}{!}{
    \begin{tabular}{cccc}
        \toprule
        Operators & LEM & AIM & PEM \\
        \midrule
        $g_P(V,W)$ & $\langle \log_{*,P} (V), \log_{*,P} (W) \rangle^{\alphabeta}$ & $\langle P^{-1}V, W P^{-1} \rangle^{\alphabeta}$ & $\frac{1}{\theta^2}\langle \pow_{\theta*,P} (V), \pow_{\theta*,P} (W) \rangle^{\alphabeta}$ \\
        \midrule
        $\rielog_P Q$ & $(\log_{*,P})^{-1} \left[ \log(Q) - \log(P) \right]$ & $P^{\frac{1}{2}} \log \left(P^{-\frac{1}{2}} Q P^{-\frac{1}{2}}\right) P^{\frac{1}{2}}$ & $(P_{\theta*,P})^{-1} \left(Q^\theta-P^\theta \right)$ \\
        \midrule
        $\Gamma_{P \rightarrow Q} (V)$ & $(\log_{*,Q})^{-1} \circ \log_{*,P} (V)$ & $(Q P^{-1})^{\frac{1}{2}} V (P^{-1} Q)^{\frac{1}{2}}$ & $(\pow_{\theta *,Q})^{-1} \circ \pow_{\theta*,P} (V)$ \\
        \midrule
        $\rieexp_P(V)$ & $\exp \left( \log(P) + \log_{*,P}(V) \right)$ & $P^{\frac{1}{2}} \exp \left(P^{-\frac{1}{2}} V P^{-\frac{1}{2}}\right) P^{\frac{1}{2}}$ & $\left(P^{\theta}+ P_{\theta*,P} (V) \right)^{\frac{1}{\theta}}$  \\
        \midrule
        Invariance & \makecell{Lie group bi-invariance \\ $\orth{n}$-invariance}  & \makecell{Lie group left-invariance \\ $\GL{n}$-invariance} & $\orth{n}$-invariance \\
        \midrule
        References & \citep{arsigny2005fast,thanwerdas2023n}  & \citep{pennec2006riemannian,thanwerdas2019affine} & \citep{dryden2010power,thanwerdas2023n,chen2024rmlr} \\
        \bottomrule
    \end{tabular}
    }
\end{table}
\begin{table}[htbp]
    \centering
    \caption{The Riemannian operators under BWM and LCM on the SPD manifold.}
    \label{app:tab:riem_operators_bwm_lc}
    \resizebox{0.85\linewidth}{!}{
    \begin{tabular}{ccc}
        \toprule
        Operators & LCM & BWM \\
        \midrule
         $g_P(V,W)$ & $\langle \lfloor \widetilde{V} \rfloor , \lfloor \widetilde{W} \rfloor \rangle +  \langle \widetilde{\bbV}\widetilde{\bbL}^{-1}, \widetilde{\bbW}\widetilde{\bbL}^{-1} \rangle$ & $\frac{1}{2} \langle \calL_{P}[V], W \rangle$ \\
        \midrule
        $\rielog_P Q$ & $(\chol^{-1})_{*, L} \left[ \lfloor K\rfloor-\lfloor L\rfloor + \bbL \dlog (\bbL^{-1} \bbK ) \right]$ & $(P Q )^{\frac{1}{2}}+(Q P)^{\frac{1}{2}} -2 P$ \\
        \midrule
        $\Gamma_{P \rightarrow Q} (V)$ & $(\chol^{-1})_{*, K} \left[\lfloor \widetilde{V} \rfloor+\bbK \bbL^{-1} \widetilde{\bbV} \right]$ & $U\left[\sqrt{\frac{\delta_i+\delta_j}{\sigma_i+\sigma_j}}\left[U^{\top} V U  \right]_{i j}\right] U^{\top}$ \\
        \midrule
        $\rieexp_P(V)$ & $\chol^{-1} \left[ \lfloor L \rfloor + \lfloor \widetilde{V} \rfloor + \bbL \dexp (\bbL^{-1} \widetilde{V} ) \right]$ & $P + V + \calL_{P}[V] P \calL_{P}[V]$ \\
        \midrule
        Invariance & Lie group bi-invariance  & $\orth{n}$-invariance \\
        \midrule
        References & \citep{lin2019riemannian}  & \citep{bhatia2019bures,thanwerdas2023n} \\
        \bottomrule
    \end{tabular}
    }
\end{table}

\subsection{Geometries of the Grassmannian}
\label{app:subsec:geom_grass}

\begin{table}[th]
    \centering
    \caption{Riemannian operators on the Grassmannian.}
    \label{app:tab:riem_operators_gras}
    \resizebox{0.95\linewidth}{!}{
    \begin{tabular}{ccc}
        \toprule
        Operators & $\grasonb{p,n}$ & $\graspp{p,n}$ \\
        \midrule
        $g_P(V,W)$ & $\inner{V}{W}$ & $\frac{1}{2} \inner{V}{W}$ \\
        \midrule
        $\rielog_P Q$ & \makecell{$O \arctan(\Sigma) R^\top$ \\ $(I_n-PP^\top)Q(P^\top Q)^{-1}\stackrel{\mathrm{SVD}}{:=} O\Sigma R^\top$} & $\frac{1}{2} [\log \left(\left(I_n-2 Q\right)\left(I_n-2 P\right)\right),P]$ \\
        \midrule
        $\Gamma_{P \rightarrow Q} (V)$
        & \makecell{ $\left(\left(\begin{array}{cc}
            P R & O
            \end{array}\right)
            \left(\begin{array}{c}
            -\sin (\Sigma) \\
            \cos (\Sigma)
            \end{array}\right)
            O^T+\left(I-O O^T\right) \right) V$  \\ $\rielog_{P}(Q) \stackrel{\mathrm{SVD}}{:=} O \Sigma R^\top$
            }
        & $\exp(\left[\log_{P} (Q), P \right]) V \exp (-[\log_{P} (Q), P])$\\
        \midrule
        $\rieexp_P V$ & \makecell{$\left(\begin{array}{cc}
            P R & O
            \end{array}\right)
            \left(\begin{array}{c}
            \cos (\Sigma) \\
            \sin (\Sigma)
            \end{array}\right) R^\top$ \\ $V \stackrel{\mathrm{SVD}}{:=} O\Sigma R^\top$} & $\exp([V,P])P\exp(-[V,P])$ \\
        \midrule
        References & \citep{edelman1998geometry,bendokat2024grassmann} & \citep{batzies2015geometric,bendokat2024grassmann} \\
        \bottomrule
    \end{tabular}
    }
\end{table}

As the set of linear subspaces, the Grassmannian can naturally be represented by any orthonormal basis, which is called the OrthoNormal Basis (ONB) perspective. Under this perspective, the Grassmannian is the quotient of the Stiefel manifold \citep{bendokat2024grassmann}, denoted by $\grasonb{p,n} \cong \stiefel{p,n} / \orth{p}$. Each point is an equivalence class:
\begin{equation}
    \grasonb{p,n} = \{ [U] \mid [U]:= \{\widetilde{U} \in \stiefel{p, n} \mid \widetilde{U}=U R, R \in \orth{p}\} \}.
\end{equation}
By abuse of notation, we use $[U]$ and $U$ interchangeably for elements of $\grasonb{p,n}$. Each tangent space can be identified as a subspace of a corresponding tangent space on the Stiefel manifold, which is called the horizontal space. Therefore, every tangent vector can be identified with a tangent vector in the horizontal space, called a horizontal lift\footnote{In this paper, the tangent vector under the ONB perspective is always considered as the horizontal lift.}. Under this identification, each tangent vector $V \in T_{P}\grasonb{p,n}$ can be represented as
\begin{equation} \label{app:eq:tangent_vec_grasonb}
    V = P_{\perp} B, \text{ with } B \in \mathbb{R}^{(n-p) \times p},
\end{equation}
where $P_{\perp} \in \stiefel{n-p,n}$ is the orthogonal complement of $P$.

Another perspective is called the Projector Perspective (PP). As shown by \citet{bendokat2024grassmann}, the Grassmannian is an embedded submanifold of $\sym{n}$:
\begin{equation}
    \graspp{p,n} = \{P \in \sym{n} \mid P^2=P, \rank(P)=p \}.
\end{equation}
Therefore, each point can be represented as an $n \times n$ symmetric matrix. Under this perspective, any tangent vector $V \in T_{P}\graspp{p,n}$ at $P \in \graspp{p,n}$ can be represented as
\begin{equation} \label{app:eq:tangent_vec_pp}
    V=Q\left(\begin{array}{cc}
    0 & B^T \\
    B & 0
    \end{array}\right) Q^T, \text{ with } B \in \bbR{(n-p) \times p},
\end{equation}
where $Q \idpp Q^\top = P$.

Supposing $P$ and $Q$ are the points on the Grassmannian $\grasonb{p,n}$ ($\graspp{p,n}$), and $V$ and $W$ are the tangent vectors over $T_{P}\grasonb{p,n}$ ($T_{P}\graspp{p,n}$), \cref{app:tab:riem_operators_gras} summarizes the associated Riemannian operators following the notation in \cref{app:tab:sum_notaitons}.

\begin{remark}
    \label{app:rmk:cutlocus_gras}
    We make the following remarks with respect to the Riemannian operators over the Grassmannian.
    \begin{itemize}
        \item
        \textbf{Cut locus \& logarithm.}
        The Grassmannian Riemannian logarithm does not exist for every pair of $P$ and $Q$. As shown by \citet[Sec. 5]{bendokat2024grassmann}, $\rielog_{P} (Q)$ exists only if $P$ and $Q$ are not in each other's cut locus. However, this can be addressed numerically, for example by \citet[Alg. 5.3]{bendokat2024grassmann} or by using the Moore–Penrose inverse for the inverse in the ONB logarithm \citep{nguyen2022gyro}.
        \item
        \textbf{PP \& ONB logarithm.}
        The matrix logarithm shown in the PP logarithm does not support backpropagation, as it cannot be calculated by the SVD like the SPD matrix. However, the PP logarithm can be calculated via the ONB logarithm \citep[Prop. 3.12]{nguyen2024matrix}. The latter can be backpropagated through the SVD. In this way, the PP logarithm can be integrated into the PyTorch deep learning framework.
    \end{itemize}

\end{remark}

\section{Discussions on the Riemannian FC and convolutional layer}

\subsection{Additional discussions on the orthogonal basis}
\label{app:subsec:add_explanantion_orth_basis}

When the inner product $g_{E}$ on $T_E\calM$ is the standard inner product, we use the familiar orthonormal basis $\{e_i\}_{i=1}^m$. However, when $g_{E}$ is not standard, $\{e_i\}_{i=1}^m$ might not be orthonormal. In this case, we can always find a corresponding basis associated with $\{e_i\}_{i=1}^m$ by a linear isometry. We rewrite the inner product $g_{E}$ as
\begin{equation}
    g_E (V,W) = \langle f(V), f(W) \rangle=f(V)^\top f(W), \forall V,W \in T_E\calM \cong \bbR{m},
\end{equation}
where $f$ is the linear isometry that pulls back the standard inner product $\langle \cdot,\cdot \rangle$ to $g_{E}$. Then, $\{B_i\}_{i=1}^m = \{f^{-1}(e_i)\}_{i=1}^m$ is the standard orthonormal basis over $\{T_E\calM,g_E\}$.

\subsection{Riemannian fully connected layers under isometric geometry}
\label{app:subsec:isometric_fc_layers}

As isometric Riemannian metrics commonly arise in various geometries \citep{thanwerdas2022theoretically,chen2024adaptive,bendokat2024grassmann}, we discuss the construction of Riemannian FC layers under isometries. The following theorem demonstrates that a Riemannian FC layer under isometric metrics can be computed by the following procedure: mapping, applying the Riemannian FC layer, and remapping. This result will be applied in our concrete examples of the SPD and Grassmannian FC layers.

We denote the FC transformation by $Y = \calF \left( X; \bfA, \bfP \right)$, with $\bfP= \left\{ P_i \in \calN \right\}_{i=1}^m$ and $\bfA=\left\{A_i \in T_{P_i}\calN \right\}_{i=1}^m$ as the FC parameters.

\begin{theorem}[Isometric FC Layers]
    \label{app:thm:fc_isometries}
    Given $n$-dimensional Riemannian manifolds $\left\{\widetilde{\calN}, g^{\widetilde{\calN}}\right\}$ and $\left\{\calN,g^{\calN} \right\}$ with a Riemannian isometry $\phi^\calN: \widetilde{\calN} \rightarrow \calN$,  and $m$-dimensional Riemannian manifolds $\left\{\widetilde{\calM}, g^{\widetilde{\calM}}\right\}$ and $\left\{\calM,g^{\calM} \right\}$ with $\phi^{\calM}: \widetilde{\calM} \rightarrow \calM$ as a Riemannian isometry mapping the origin $E^{\widetilde{\calM}} \in \widetilde{\calM}$ to the origin $E \in \calM$, the Riemannian FC layer $\widetilde{\calF}: \widetilde{\calN} \rightarrow \widetilde{\calM}$ can be calculated by $\calF: \calN \rightarrow \calM$:
    \begin{equation}
        \widetilde{\calF} \left( \widetilde{X}; \widetilde{\bfP}, \widetilde{\bfA} \right) = \left( \phi^{\calM} \right)^{-1} \left(\calF \left( \phi^\calN (\widetilde{X}); \bfP, \bfA \right) \right),
    \end{equation}
    where $\widetilde{\bfP}= \left\{ \widetilde{P}_i \in \widetilde{\calN} \right\}_{i=1}^m$ and $\widetilde{\bfA}=\left\{\widetilde{A}_i \in T_{\widetilde{P}_i}\widetilde{\calN} \right\}_{i=1}^m$ are the FC parameters of $\widetilde{\calF}$, while $\bfP=\left\{ \phi^\calN(\widetilde{P}_i)\right\}_{i=1}^m$ and $\bfA=\left\{\phi^\calN _{*,\widetilde{P}_i}(\widetilde{A}_i)\right\}_{i=1}^m$ are the FC parameters of $\calF$.
\end{theorem}

\begin{proof}
    First, we show the correspondence between the standard orthonormal bases $\{\widetilde{B}_i \in \widetilde{\calM}\}$ and $\{B_i \in \calM\}$. The set $\{\widetilde{B}_i \in \widetilde{\calM}\}$ is orthonormal iff $\{B_i \in \calM\}$ is orthonormal. We only need to show standardness. The Riemannian metric $g^{\widetilde{\calM}}$ satisfies
    \begin{equation}
        \begin{aligned}
            g^{\widetilde{\calM}}_{\widetilde{E}}(V,W)
            &\stackrel{(1)}{=} g^\calM _{E} \left(\phi^\calM _{*,\widetilde{E}} (V),\phi^\calM _{*,\widetilde{E}} (V) \right)\\
            &= \left\langle f \circ \phi^\calM _{*,\widetilde{E}} (V),f \circ \phi^\calM _{*,\widetilde{E}} (V) \right\rangle,
        \end{aligned}
    \end{equation}
    where $f$ is the linear isomorphism that pulls back the standard Frobenius inner product to $g^{\calM}_{E}$. Here, (1) comes from the isometry. Therefore, for each $i$, we have
    \begin{equation}
        \begin{aligned}
            \widetilde{B}_i
            &= (f \circ \phi^\calM_{*,\widetilde{E}})^{-1} (E_i)\\
            &\stackrel{(1)}{=} \left( \phi^\calM_{*,\widetilde{E}} \right)^{-1} (B_i),
        \end{aligned}
    \end{equation}
    where (1) comes from $B_i = f^{-1}(E_i), \forall i=1,\cdots,n$.

    We now demonstrate the correspondence between the FC layers as follows:
    \begin{equation}
        \begin{aligned}
            Y &= \rieexp^{\widetilde{\calM}}_{\widetilde{E}} \left( \sum_{i=1}^{m} \left( \langle \rielog^{\widetilde{\calN}}_{ \widetilde{P}_i}(\widetilde{X}), \widetilde{A}_i \rangle_{\widetilde{P}_i}^{\widetilde{\calN}} {\widetilde{B}_i} \right) \right)\\
            &\stackrel{(1)}{=} \left( \phi^{\calM} \right)^{-1} \left( \rieexp^{\calM} _{E} \left( \phi^{\calM} _{*,\widetilde{E}}\left[ \sum_{i=1}^{m} \left( \langle \rielog^{\calN}_{P_i}(X), A_i \rangle^{\calN}_{P_i} \widetilde{B}_i \right) \right] \right) \right)\\
            &\stackrel{(2)}{=} \left( \phi^{\calM} \right)^{-1} \left( \rieexp^{\calM}_{E} \left( \sum_{i=1}^{m} \left( \langle \rielog^{\calN} _{P_i}(X), A_i \rangle^{\calN} _{P_i} B_i \right) \right) \right),
        \end{aligned}
    \end{equation}
    where $B_i =\phi^{\calM} _{*,\widetilde{E}} (\widetilde{B}_i)$, $A_i= \phi^{\calN} _{*,\widetilde{P}_i} (\widetilde{A}_i)$, $X=\phi^{\calN}(\widetilde{X})$, and $P_i=\phi^{\calN}(\widetilde{P}_i)$.
    The above derivation comes from the following.
    \begin{enumerate}[label=(\arabic*)]
        \item
        The isometry of $\phi^{\calM}$ and $\phi^{\calN}$;
        \item
        The linearity of $\phi^{\calM}_{*,\widetilde{E}}$.
    \end{enumerate}
\end{proof}

\subsection{Riemannian fully connected layers under product geometry}
\label{app:subsec:riem_fc_product}

Now, we discuss \cref{thm:riem_fc} under product geometry.

\begin{theorem}
    Following the notation in \cref{thm:riem_fc}, the Riemannian FC layer $\calF(\cdot): (\calN)^c \rightarrow \calM$ for the input $(X_1 \in \calN, \cdots ,X_c \in \calN) = X \in (\calN)^c$ is
    \begin{equation}
        Y = \rieexp^{\calM}_E \left( \sum_{i=1}^{m} \sum_{j=1}^{c} \langle \rielog^{\calN}_{P_{ij}}(X), A_{ij} \rangle^{\calN}_{P_{ij}} B_i \right),
    \end{equation}
    where $P_{ij} \in \calN$ and $A_{ij} \in T_{P_{ij}}\calN$ are the FC parameters.
\end{theorem}
\begin{proof}
    By product geometry, we have
    \begin{align}
        (\calN)^c \ni P_i  &= (P_{i1} \in \calN,\cdots,P_{ic} \in \calN),\\
        T_{P_i} (\calN)^c \ni A_i  &= (A_{i1} \in T_{P_{i1}} \calN, \cdots, A_{ic} \in T_{P_{ic}} \calN).
    \end{align}
    The above implies that
    \begin{equation}
        \langle \rielog^{(\calN)^c}_{P_i}(X), A_i \rangle^{(\calN)^c}_{P_i}
        = \sum_{j=1}^{c} \langle \rielog^{\calN}_{P_{ij}}(X), A_{ij} \rangle^{\calN}_{P_{ij}}.
    \end{equation}
\end{proof}

\subsection{Riemannian fully connected layers and manifold embedding}
\label{app:subsec:manifold_embedding}

In several applications \citep{chami2019hyperbolic,lopez2021vector,zhao2023modeling,nguyen2024matrix}, embedding Euclidean features into non-Euclidean manifolds often yields superior results. A common approach can be expressed as $\rieexp_{E}(Ax+b)$, which maps Euclidean features to the tangent space at the origin via a linear layer, followed by applying the exponential map at the origin. This method has been adopted in various embeddings, including hyperbolic \citep{chami2019hyperbolic,fu2024hyperbolic}, SPD \citep{zhao2023modeling}, and Grassmannian spaces \citep[Sec. 3.4.2]{nguyen2024matrix}. Our framework offers a novel intrinsic interpretation, showing that this operation respects the Riemannian FC layer between the Euclidean space and the target manifold.

\begin{proposition}
    \label{prop:r2manifold_fc}
    The Riemannian FC layer from a standard Euclidean space $\bbR{n}$ to an $m$-dimensional target manifold $\calM$, namely $\calF(\cdot): \bbR{n} \rightarrow \calM$, is given by
    \begin{equation}
        \calF(x) = \rieexp_{E}(Ax+b),
    \end{equation}
    where $A \in \bbR{n \times m}$ and $b \in \bbR{m}$ are the transformation matrix and bias vector, respectively.
\end{proposition}

\begin{proof}
    By \cref{thm:riem_fc}, we have the following
        \begin{equation}
            \begin{aligned}
                Y
                &\stackrel{(1)}{=} \rieexp^{\calM}_E \left( \sum_{i=1}^{m} \left( \langle \rielog^{\mathrm{Euc}}_{p_i}(x), a_i \rangle^{\mathrm{Euc}}_{p_i} B _i \right) \right), \\
                &\stackrel{(2)}{=} \rieexp^{\calM}_E \left( \sum_{i=1}^{m} \left( \langle x - p_i, a_i \rangle B _i \right) \right), \\
                &\stackrel{(3)}{=} \rieexp^{\calM}_E \left( \sum_{i=1}^{m} \left( \langle x - p_i, a_i \rangle f^{-1} (e _i) \right)
                \right), \\
                &\stackrel{(4)}{=} \rieexp^{\calM}_E \left( f^{-1}
                \left( \sum_{i=1}^{m}  \langle x - p_i, a_i \rangle e _i \right)
                \right), \\
                &\stackrel{(5)}{=} \rieexp^{\calM}_E \left( f^{-1}
                \left( \bar{A}x + \bar{b} \right)
                \right), \\
                &\stackrel{(6)}{=} \rieexp^{\calM}_E \left( Ax+b
                \right). \\
            \end{aligned}
    \end{equation}
    The above follows from the following.
    \begin{enumerate}[label=(\arabic*)]
        \item
        $p_i, a_{i} \in \bbR{n}$, and $\{B_i\}$ is an orthonormal basis over $\{T_E\calM, g_E \}$;
        \item
        The Euclidean logarithm and metric become the familiar vector operation:
        \begin{equation*}
            \begin{aligned}
                \rielog^{\mathrm{Euc}}_{p_i}(x)
                &=x-p_i\\
                \inner{v}{w}^{\mathrm{Euc}}_{p}
                &= \inner{v}{w}, \forall p \in \bbR{n}, \forall v,w \in T_p \bbR{n};
            \end{aligned}
        \end{equation*}
        \item
        $f$ is the linear isomorphism pulling the standard inner product back to $g_E$; $\{e_i\}$ is the standard orthonormal basis over the standard inner product;
        \item
        Linearity of $f^{-1}$;
        \item
        $\sum_{i=1}^{m}  \langle x - p_i, a_i \rangle e _i$ has the form of an affine transformation;
        \item
        As $f^{-1}$ has matrix representation, $f^{-1}(x)= \tilde{A}x$, we have
        \begin{equation}
            \begin{aligned}
                f^{-1} \left( \bar{A}x + \bar{b} \right)
                &= \tilde{A}\left( \bar{A}x + \bar{b} \right)\\
                &= \tilde{A}\bar{A}x+ \tilde{A}\bar{b}.
            \end{aligned}
        \end{equation}
        Setting $A=\tilde{A}\bar{A}$ and $b=\tilde{A}\bar{b}$, one can obtain the result.
    \end{enumerate}
\end{proof}

\subsection{Relation with the convolution in ManifoldNet}
\citet{chakraborty2020manifoldnet} also proposed a convolution operation for manifolds. However, since its formulation is based on the weighted Fréchet mean, it is unable to alter the manifold dimension, for example through dimensionality reduction. In contrast, our framework allows modifications in both the channel and manifold dimensions, providing greater flexibility.

\section{Comparison of our hyperbolic FC layers against previous ones}
\label{app:sec:comp_hyperbolic_linear}

\cref{app:tab:comp_hyperbolic_linear} extends \cref{tab:comp_hyperbolic_linear} by comparing our hyperbolic FC layers against previous hyperbolic linear layers.

\begin{table}[htbp]
\centering
\caption{Comparison of hyperbolic linear layers. Here, we consider the transformation from an $n$-dimensional hyperbolic space to an $m$-dimensional one.}
\label{app:tab:comp_hyperbolic_linear}
\resizebox{\linewidth}{!}{
\begin{tabular}{cccccc}
\toprule
Method & Model & Mechanism & Formulation & Parameters & References \\
\midrule
Möbius & $\pball{n}$ & Tangent & $\rieexp _\bbzero (M \rielog _\bbzero (x))$ & $M \in \bbR{m \times n}$ & \citep[Def. 3.2]{ganea2018hyperbolic} \\
\midrule
Klein & $\klein{n}$  & Tangent & $\rieexp _\bbzero (M \rielog _\bbzero (x))$ & $M \in \bbR{m \times n}$ & \citep[Thm. 9]{mao2024klein} \\
\midrule
LFC & $\bbh{n}$ & Spacetime &
$\displaystyle \left[\begin{array}{c}
\frac{ \sqrt{ \| W x \|^2 - 1/K } }{ v^{\top} x }  v^{\top} \\
W
\end{array}\right] x$ & \makecell{$M \in \bbR{m \times (n+1)}$ \\ $v \in \bbR{n+1}$}  & \citep[Sec 3.1]{chen2022fully} \\
\midrule
NestFC & $\bbh{n}$ & Nested projection & $\displaystyle y=\frac{Wx}{\|Wx\|_L},\quad WJ_nW^\top=J_m$ & \makecell{$Q\in\mathrm{SO}(n)$ \\ $\widetilde P\in\mathrm{St}(m,n)$ \\ $\alpha\in\bbRscalar$} & \citep[Eq. (14) and Sec. 3.3]{fan2022nested} \\
\midrule
Poincaré FC & $\pball{n}$ & Poincaré & \makecell{$w\left(1+\sqrt{1-K\|w\|^2}\right)^{-1}$ \\ $w = \left((-K)^{-\frac{1}{2}} \sinh \left(\sqrt{-K}  v_k(x)\right)\right)_{k=1}^m$ \\
$v_k$ is defined by \citet[Eq. (6)]{shimizu2021hyperbolic}} & \makecell{$\{z_i \in \bbR{n}\}_{i=1}^m$ \\ $\{\gamma _i \in \bbRscalar \}_{i=1}^m$} &\citep[Sec. 3.2]{shimizu2021hyperbolic} \\
\midrule
\rowcolor{HilightColor} Ours & $\pball{n}, \klein{n},\bbh{n}$ & Riemannian  & \cref{app:thm:pball_fc,app:thm:hyperboloid_fc} & \makecell{$\{z_i \in \bbR{n}\}_{i=1}^m$ \\ $\{\gamma _i \in \bbRscalar \}_{i=1}^m$} & \cref{app:thm:pball_fc,app:thm:hyperboloid_fc} \\
\bottomrule
\end{tabular}
}
\end{table}

\section{Additional details on the SPD fully connected layers}

\subsection{Relation with the gyro SPD fully connected layers}
\label{app:sec:relation_gyro_spd_fc}

This subsection demonstrates that our SPD FC layers subsume three gyro SPD FC layers under LEM, AIM, and LCM. This follows directly from \cref{prop:riem_fc_gen_euc_fc}, as one can readily verify that the point-to-hyperplane pseudo-distance we used is identical to the corresponding pseudo-gyrodistances under these three metrics. To clarify this relationship more clearly, we compare the final expressions.

We first review some related SPD gyro structures \citep{nguyen2023building}. Given $P$, $Q$ in $\{\spd{n}, g\}$ with $g$ being AIM, LEM, or LCM, and $t \in \bbRscalar$, the gyro structures induced by $g$ are defined as follows:
\begin{align}
    \label{eq:gyro_addtion}
    \text{Gyro addition: } P \oplus Q &= \rieexp_{P}\left(\pt{I}{P} \left(\rielog _{I}(Q)\right)\right), \\
    \label{eq:gyro_scalar_product}
    \text{Scalar gyromultiplication: } t \otimes P &= \rieexp_{I}\left(t \rielog _{I}(P)\right), \\
    \label{eq:gyro_inverse}
    \text{Gyro inverse: } \ominus P &= -1 \otimes P = \rieexp_{I}\left(- \rielog _{I}(P)\right), \\
    \label{eq:gyro_inner_product}
    \text{Gyro inner product: } \gyrinner{P}{Q}&=\left\langle \rielog _I (P), \rielog _I (Q)\right\rangle_{I},
\end{align}
where $\rielog _I$ and $\langle \cdot,\cdot \rangle_I$ are the Riemannian logarithm and metric at the identity matrix $I$. As shown by \citet{nguyen2022gyro}, the gyro addition and scalar product under AIM, LEM, and LCM form gyrovector spaces.

Based on these gyro structures, \citet{nguyen2024matrix} introduced the gyro SPD FC layers under AIM, LEM, and LCM, respectively. We review their results in the following.
\begin{theorem} [Gyro SPD FC Layers \citep{nguyen2024matrix}]
    \label{app:thm:gyro_spd_fc_layer}
    The gyro SPD FC layers under standard LEM, AIM, and LCM are
    \begin{align}
        \text{LEM}:
        &  Y = \exp \left( V^\LE \right),
        V^{\LE}_{ij} =
        \begin{cases}
        v^{\LE}_{ii}(S) , & \text { if } i=j \\
        \frac{1}{\sqrt{2}}v^{\LE}_{ij}(S), & \text { if } i>j \\
        V^{\LE}_{ji}, & \text{ otherwise }
        \end{cases} \\
        \text{AIM}:
        &  Y = \exp \left( V^\AI \right),
        V^{\AI}_{ij} =
        \begin{cases}
        v^{\AI}_{ii}(S) + \eta \sum_{k=1}^m v^{\AI}_{kk}(S) , & \text { if } i=j \\
        \frac{1}{\sqrt{2}}v^{\AI}_{ij}(S), & \text { if } i>j \\
        V^{\AI}_{ji}, & \text{ otherwise }
        \end{cases} \\
        \text{LCM}:
        & Y = V^{\LC}(V^{\LC})^\top,
        V^{\LC}_{ij} =
        \begin{cases}
        \exp\left(v^{\LC}_{ii}(S)\right) , & \text { if } i=j \\
        v^{\LC}_{ij}(S), & \text { if } i>j\\
        0, & \text{ otherwise }
        \end{cases}
    \end{align}
    where $\eta=\frac{1}{n}\left(\frac{1}{\sqrt{1+n \beta}}-1\right)$,  and $v_{ij}^g = \gyrinner{\ominus P_{ij} \oplus S}{W_{ij}}$ with $g$ as LEM, AIM, or LCM. Here, $P_{ij}, W_{ij} \in \spd{n}, \forall i \geq j, i,j = 1, \cdots, m$.
\end{theorem}

\begin{proposition}
    Our LEM ($\alphabeta=(1,0)$), AIM ($\alphabeta=(1,\beta)$), and LCM SPD FC layers incorporate the LEM, AIM, and LCM gyro SPD FC layers, respectively.
\end{proposition}
\begin{proof}
    Comparing \cref{app:thm:gyro_spd_fc_layer} with our \cref{thm:spd_fc}, we only need to show the equality of $v_{ij}$ in the gyro and our framework:
    \begin{equation}
        \begin{aligned}
            v_{ij}^g \stackrel{(1)}{=} \inner{\rielog_{P_{ij}}\left( S \right)}{\pt{I}{P_{ij}}\left(\rielog_{I} (W_{ij})\right)}_{P_{ij}},
        \end{aligned}
    \end{equation}
    where (1) has been proved in \cref{prop:riem_fc_gen_euc_fc}. Setting $A_{ij} = \pt{I}{P}\left(\rielog_{I} (W_{ij})\right) \in T_{P_{ij}} \spd{n}$, we recover \cref{app:eq:v_ij_lem,app:eq:v_ij_aim,app:eq:v_ij_lcm} for each metric.
\end{proof}

\subsection{Relation with the flat SPD fully connected layers}
\label{app:sec:relation_flat_spd_fc}
\citet{nguyen2025symmetric} proposed two SPD FC layers based on flat LEM and LCM. However, as shown by \citet[App. B. 2.2]{nguyen2025symmetric}, they have the same formulations as the LEM and LCM gyro SPD FC layers, respectively.

\subsection{Trivialized SPD fully connected layers}
\label{app:sec:simplified_sdp_fc}

\begin{theorem}[Trivialized SPD FC Layers]
    \label{app:thm:trivilized_spd_fc}
    Trivializing each $P_{ij}$ in \cref{thm:spd_fc} as $\rieexp _{I} (\gamma _{ij} [Z_{ij}])$, $v_{ij}(S)$ under different metrics can be further simplified:
    \begin{align}
        \text{LEM}:
        &  \left\langle \log(S), Z_{ij} \right\rangle^{\alphabeta} - \gamma_{ij} \norm{Z_{ij}}^{\alphabeta}, \\
        \text{AIM}:
        &\left\langle \log \left( \exp \left( -\frac{\gamma_{ij}}{2} [Z_{ij}] \right)  S \exp \left( -\frac{\gamma_{ij}}{2} [Z_{ij}] \right) \right), Z_{ij} \right\rangle^{\alphabeta}, \\
        \text{PEM}:
        & \left\langle S^\theta- \left( I + \theta \gamma_{ij} [Z_{ij}] \right), Z_{ij} \right\rangle^{\alphabeta}, \\
        \text{LCM}:
        & \left\langle \lfloor K\rfloor + \dlog(\bbK) - \left( \gamma_{ij} \lfloor [Z_{ij}] \rfloor + \frac{1}{2} \gamma_{ij} \bbD([Z_{ij}])\right),  \lfloor Z_{ij} \rfloor + \frac{1}{2}\bbZ_{ij} \right\rangle,
    \end{align}
    where $\norm{\cdot}^{\alphabeta}$ is the norm induced by $\inner{\cdot}{\cdot}^{\alphabeta}$, and $\bbD(\cdot)$ returns a diagonal matrix with diagonal elements from the input square matrix.
\end{theorem}
\begin{proof}
    \textbf{LEM:}
    \begin{equation}
        \begin{aligned}
             \left\langle \log(S)-\log(P_{ij}), Z_{ij} \right\rangle^{\alphabeta}
             &\stackrel{(1)}{=} \left\langle \log(S)- \gamma_{ij} [Z_{ij}], Z_{ij} \right\rangle^{\alphabeta}\\
             &\stackrel{(2)}{=} \left\langle \log(S), Z_{ij} \right\rangle^{\alphabeta} - \gamma_{ij} \norm{Z_{ij}}^{\alphabeta},
        \end{aligned}
    \end{equation}
        The above comes from the following.
    \begin{enumerate}[label=(\arabic*)]
        \item
        \cref{app:eq:exp_i_lem_aim};
        \item
        $[Z_{ij}] = \frac{Z_{ij}}{\norm{Z_{ij}}^{\alphabeta}}$.
    \end{enumerate}

    \textbf{AIM:}
    This can be obtained by the following:
    \begin{equation}
        \begin{aligned}
             \exp \left( \gamma_{ij} [Z_{ij}] \right)^{-\frac{1}{2}} = \exp \left( -\frac{\gamma_{ij}}{2} [Z_{ij}] \right).
        \end{aligned}
    \end{equation}

    \textbf{PEM:}
    This can be obtained by \cref{app:eq:exp_i_pem}.

    \textbf{LCM:}
    \begin{equation}
        \begin{aligned}
            &\left\langle \lfloor K\rfloor - \lfloor L_{ij} \rfloor + \dlog(\bbK\bbL_{ij}^{-1}), \lfloor Z_{ij} \rfloor + \frac{1}{2}\bbZ_{ij} \right\rangle \\
            &= \left\langle \lfloor K\rfloor + \dlog(\bbK) - \left( \lfloor L_{ij} \rfloor +  \dlog(\bbL_{ij}) \right), \lfloor Z_{ij} \rfloor + \frac{1}{2}\bbZ_{ij} \right\rangle \\
            &\stackrel{(1)}{=} \left\langle \lfloor K\rfloor + \dlog(\bbK) - \left( \gamma_{ij} \lfloor [Z_{ij}] \rfloor + \frac{1}{2} \gamma_{ij} \bbD([Z_{ij}])\right), \lfloor Z_{ij} \rfloor + \frac{1}{2}\bbZ_{ij} \right\rangle,
        \end{aligned}
    \end{equation}
    where (2) comes from \cref{app:eq:exp_i_lcm}.
\end{proof}
\begin{remark} \label{app:rmk:spd_fc_constrains}
    Due to the incompleteness of PEM and BWM, their exponential maps at $I$, $\rieexp_{I}(V)$, are well-defined locally:
    \begin{equation}
        \begin{aligned}
            \text{PEM: }&  I + \theta V \in \spd{n},\\
            \text{BWM: }&  I + \frac{1}{2} V \in \spd{n}.\\
        \end{aligned}
    \end{equation}
    The above restriction can be addressed numerically, for example by ReEig \citep{huang2017deep}:
    \begin{equation}
        \widetilde{S} = U \max(\epsilon I, \Sigma) U^\top,
    \end{equation}
    where $S\stackrel{\text{Eig}}{:=} U \Sigma U^\top$ is the eigendecomposition.
\end{remark}

\subsection{Trivialized SPD multinomial logistic regression}
\label{app:sec:simplified_sdp_mlr}

In our implementation, we trivialize the SPD parameters in the SPD MLR as in \cref{subsec:fc_parameters}. The SPD MLRs proposed by \citet{chen2024rmlr} under five geometries can be further simplified. For simplicity, we do not involve the power deformation \citep{chen2024rmlr}.

\begin{theorem}[Trivialized SPD MLRs]
    \label{thm:spdmlrs}
    \linktoproof{thm:spdmlrs_trivilized}
    Given $C$ classes and an SPD feature $S$, the SPD MLRs, $p(y=k \mid S \in \spd{n})$, are proportional to
    \begin{align}
        \text{LEM}:
        & \exp \left[ \left\langle \log(S), Z_{k} \right\rangle^{\alphabeta} - \gamma_{k} \norm{Z_{k}}^{\alphabeta} \right], \\
        \text{AIM}:
        &\left[ \exp  \left\langle \log \left( \exp \left( -\frac{\gamma_{k}}{2} [Z_{k}] \right)  S \exp \left( -\frac{\gamma_{k}}{2} [Z_{k}] \right) \right), Z_{k} \right\rangle^{\alphabeta} \right], \\
        \text{PEM}:
        & \frac{1}{\theta} \exp \left[ \left\langle S^\theta- \left( I + \theta \gamma_{k} [Z_{k}] \right), Z_{k} \right\rangle^{\alphabeta} \right], \\
        \text{LCM}:
        & \exp \left[ \left\langle \lfloor K\rfloor + \dlog(\bbK) - \left( \gamma_{k} \lfloor [Z_{k}] \rfloor + \frac{1}{2} \gamma_{k} \bbD([Z_{k}])\right),  \lfloor Z_{k} \rfloor + \frac{1}{2}\bbZ_{k} \right\rangle \right],\\
        \text{BWM: }
        & \exp \left[ \frac{1}{2} \left\langle \left(P_{k}S \right)^{\frac{1}{2}} + \left(S P_{k} \right)^{\frac{1}{2}} -2 P_{k}, \calL_{P_{k}}(L_{k} Z_{k} L_{k}^\top) \right\rangle \right],
    \end{align}
    where $Z_k \in T_I\spd{n} \backslash \{0\}$ is a symmetric matrix, $L_k=\chol(P_k)$ is the Cholesky factor of $P_k$ with $P_k = (I +\frac{1}{2} \gamma_{k} [Z_{k}])^2$. Here $\{Z_k \in \sym{n} \}_{k=1}^C$ and $\{\gamma_{k} \in \bbRscalar \}_{k=1}^C$ are the MLR parameters.
\end{theorem}
\begin{proof}
    For each class $k$,  the expression of $v_{k}$ in the SPD MLR \citep[Thm. 4.2]{chen2024rmlr} has been reviewed in \cref{app:subsec:prf_spd_fc}.
    For MLR under each metric $g$, we parameterize each parameter $P_{k} \in \spd{n}$ by $Z_k$ and $\gamma_{k}$ by
    \begin{equation}
        P_k = \rieexp^g_{I}( \gamma_{k} [Z_{k}]),
    \end{equation}
    with $[Z_{k}]$ as the unit vector of $Z_k$.
    Under this parameterization, the MLRs under LEM, AIM, PEM, and LCM can be further simplified, which has been implied by \cref{app:thm:trivilized_spd_fc}.
\end{proof}

\begin{remark}
    Similar to the SPD FC layer, due to the incompleteness of PEM and BWM, the associated parameterization should follow
    \begin{align}
        \text{PEM: } I + \theta \gamma_{k} [Z_{k}] \in \spd{n}, \\
        \text{BWM: } I +\frac{1}{2} \gamma_{k} [Z_{k}] \in \spd{n}.
    \end{align}
\end{remark}

\subsection{Covariant equivariance}
\label{app:sec:equivariance}

\textbf{Equivariance in ManifoldNet \citep{chakraborty2020manifoldnet}.}
For a learned scalar kernel $w$ with positive weights that sum to one, ManifoldConv \citep[Eq. (8)]{chakraborty2020manifoldnet} is defined as
\begin{equation}
    (f*w)(y)=\mathop{\operatorname{argmin}}_{Z\in\calM}\sum_{x\in\mathcal{K}_y}w(x-y)d_{\calM}^2(f(x),Z),
\end{equation}
where $\mathcal{K}_y$ is the receptive field centered at $y$. Since an isometry preserves Riemannian distances and hence commutes with the weighted Fr\'{e}chet mean, ManifoldConv satisfies the fixed-parameter equivariance
\begin{equation}
    ((\phi\circ f)*w)(y)=\phi((f*w)(y)).
\end{equation}

\textbf{Equivariance of our layers.}
Our FC and convolutional layers are defined through Riemannian operators that are compatible with isometries. In contrast to ManifoldNet's fixed-parameter equivariance, they admit parameter-covariant equivariance. Since our convolutional layer is composed of local FC transformations, we only consider the FC layer below for notational convenience. Consider the equal-manifold and equal-dimension setting $\calN=\calM$ in \cref{thm:riem_fc}. For an isometry $\phi:\calM\to\calM$, we have
\begin{equation}
    \calF_{A,P,B,E}(\phi(X))
    =\rieexp_{E}\left(\sum_{i=1}^m\left\langle\rielog_{P_i}(\phi(X)),A_i\right\rangle_{P_i}B_i\right)
    =\phi\left(\calF_{\bar A,\bar P,\bar B,\bar E}(X)\right),
\end{equation}
with
\begin{equation}
    \bar P_i=\phi^{-1}(P_i),\quad \bar A_i=(\phi^{-1})_{*,P_i}(A_i),\quad \bar E=\phi^{-1}(E),\quad \bar B_i=(\phi^{-1})_{*,E}(B_i).
\end{equation}
Since $\phi$ is an isometry, $\{\bar B_i\}_{i=1}^m$ remains an orthonormal basis of $T_{\bar E}\calM$. If $\phi(E)=E$ further fixes the origin $E$, and the equivariance identity simplifies to
\begin{equation}
    \calF_{A,P}(\phi(X))=\phi\left(\calF_{\bar A,\bar P}(X)\right).
\end{equation}

\section{Review of previous Grassmannian transformation layers}
\label{app:sec:review_grass_trans_layers}

This section briefly reviews several popular Grassmannian transformation layers.

\textbf{FRMap + ReOrth.}
Given an input Grassmannian $X \in \grasonb{p,q}$, \citet{huang2018building} used Full Rank Map (FRMap) to transform the input orthonormal matrices of subspaces into new matrices through a linear mapping function, and then applied QR decomposition to recover orthogonality:
\begin{equation}
   Y = \qr(WX),
\end{equation}
where $W \in \bbR{m \times n}$ is a row-wise orthogonal parameter, and $\qr(\cdot)$ returns the orthogonal matrix in the QR decomposition.

\textbf{PP \& ONB Scaling.}
\citet{nguyen2022gyro,nguyen2023building} proposed matrix scaling for the PP and ONB Grassmannian, respectively. Given $P=XX^\top \in \graspp{p,n}$ with $X \in \grasonb{p,n}$, the operations are defined as
\begin{align}
     &\textbf{PP: }
    Y = \exp \left(\left[\begin{array}{cc}
    0 & W * B \\
    -(W * B)^T & 0
    \end{array}\right]\right) \idpp \exp \left(-\left[\begin{array}{cc}
    0 & W * B \\
    -(W * B)^T & 0
    \end{array}\right]\right),\\
    &\textbf{ONB: }
    Y = \exp \left(\left[\begin{array}{cc}
    0 & W * B \\
    -(W * B)^T & 0
    \end{array}\right]\right) \idonb,
\end{align}
where $*$ denotes the Hadamard product and $B \in \bbR{(n-p) \times p}$ is a Euclidean parameter. Here, $X = \exp \left(\left[\begin{array}{cc}
    0 & B \\
    -B^T & 0
    \end{array}\right]\right)\idonb$.

\textbf{GrTrans.}
\citet{nguyen2023building} adopted Grassmannian gyrogroup translation (GrTrans) to transform the ONB and PP Grassmannian features. Given $X \in \graspp{p,n}$ (or $X \in \grasonb{p,n}$), the operation is defined as
\begin{align}
    Y = W \oplus X,
\end{align}
where $\oplus$ is the Grassmannian PP (ONB) gyro addition \citep[Sec. 2.3]{nguyen2023building}, and $W \in \graspp{p,n}$ (or $W \in \grasonb{p,n}$) is a Grassmannian parameter.

\section{Additional experimental details and results}
\label{app:sec:exp_details}

\subsection{Hyperbolic spaces}
\label{app:subsec:exp_details_hyperbolic}

\subsubsection{Datasets}

\textbf{Disease \citep{anderson1991infectious}. }
It represents a disease propagation tree, simulating the SIR disease transmission model, with each node representing either an infection or a non-infection state.

\textbf{Airport \citep{zhang2018link}. }
It is a transductive dataset where nodes represent airports and edges represent airline routes from OpenFlights.org.

\textbf{Pubmed \citep{namata2012query}. }
This is a standard benchmark describing citation networks where nodes represent scientific papers in the area of medicine, edges are citations between them, and node labels are academic (sub)areas.

\textbf{Cora \citep{sen2008collective}. }
It is a citation network where nodes represent scientific papers in the area
of machine learning, edges are citations between them, and node labels are academic (sub)areas.

\subsubsection{Implementation details}

We follow the official implementations of HNN\footnote{\url{https://github.com/dalab/hyperbolic_nn}} \citep{ganea2018hyperbolic}, HNN++\footnote{\url{https://github.com/mil-tokyo/hyperbolic_nn_plusplus}} \citep{shimizu2021hyperbolic}, and HyboNet \footnote{\url{https://github.com/chenweize1998/fully-hyperbolic-nn}} \citep{chen2022fully} to conduct the experiments. For the Einstein transformation in the Beltrami--Klein model, we carefully implement it according to the original paper \citep{mao2024klein}. We adopt the HGCN\footnote{\url{https://github.com/HazyResearch/hgcn}} settings \citep{chami2019hyperbolic} for the link prediction task.

\textbf{Details on main experiments.}
Following the HNN implementations \citep{ganea2018hyperbolic,chami2019hyperbolic,mao2024klein}, the baseline encoder consists of two transformation layers: the first maps the input feature dimension to 16, and the second maps 16 to 16. The transformation layers can be our HFC layers or alternatives such as Möbius, Einstein, Poincaré FC, LorentzTan, or LFC. On Disease, Airport, and Pubmed, each transformation is followed by an activation layer $\rieexp_o (\operatorname{ReLU} (\rielog_o(x)))$, where $o$ is the origin in each model. On Cora, no activation layer is used for any method. Following HNN, we also adopt the bias translation after each HFC layer, \ie, $x \oplus b=\rieexp_x (\pt{o}{x} \rielog_o(b))$, except NestFC. We use the Adam optimizer \citep{kingma2015adam} with a learning rate of $1e^{-2}$. Within each dataset, all methods use the same hyperparameters except for weight decay and dropout, which are tuned for each method. The weight decay and dropout configurations for our HFC models are reported in \cref{app:tab:hyperbolic_hyperparameters}.

\begin{center}
  \captionsetup{type=table,hypcap=false}
  \caption{Weight decay and dropout configurations for the HFC models. Each entry reports weight decay / dropout.}
  \label{app:tab:hyperbolic_hyperparameters}
  \begin{tabular}{c|cccc}
    \toprule
    Method & Disease & Airport & Pubmed & Cora \\
    \midrule
    HFC-P & $1e^{-5} / 0$ & $1e^{-3} / 0.1$ & $1e^{-5} / 0$ & $1e^{-4} / 0$ \\
    HFC-K & $5e^{-5} / 0$ & $1e^{-3} / 0$ & $0 / 0$ & $5e^{-5} / 0.1$ \\
    HFC-H & $0 / 0$ & $1e^{-3} / 0$ & $1e^{-5} / 0.1$ & $1e^{-3} / 0.1$ \\
    \bottomrule
  \end{tabular}
\end{center}

\textbf{Details on ablations on the RResNet.}
We employ a hyperbolic transformation layer to map each input vector into an 8-dimensional vector in the Poincaré ball. The network consists of two residual blocks, each configured with different hidden dimensions and varying numbers of horospheres. We use the Adam optimizer \citep{kingma2015adam} and fine-tune hyperparameters, such as the learning rate and weight decay.

\subsubsection{Complexity and parameter analysis}
\label{app:subsubsec:hyperbolic_complexity}

\begin{table}[ht]
  \centering
  \caption{Forward complexity and parameter dimensions of an $n$-to-$m$ hyperbolic FC layer for a batch of $B$ samples in the reduction setting $m\le n$. The worst entry in each metric is underlined.}
  \label{app:tab:hyperbolic_fc_complexity_parameters}
  \begin{tabular}{l|c|c|c}
    \toprule
    Method & Space & Complexity & Parameter dimension \\
    \midrule
    Möbius \citep{ganea2018hyperbolic} & $\pball{n}$ & $O(Bnm)$ & $nm$ \\
    Einstein \citep{mao2024klein} & $\klein{n}$ & $O(Bnm)$ & $nm$ \\
    LorentzTan & $\bbh{n}$ & $O(Bnm)$ & $nm$ \\
    LFC \citep{chen2022fully} & $\bbh{n}$ & $O(Bnm)$ & $m(n+1)+m+(n+1)+2$ \\
    NestFC \citep{fan2022nested} & $\bbh{n}$ & $\underline{O(n^3+Bn^2)}$ & $\underline{\frac{n(n-1)}{2}+mn-\frac{m(m+1)}{2}+1}$ \\
    Poincaré FC \citep{shimizu2021hyperbolic} & $\pball{n}$ & $O(Bnm)$ & $nm+m$ \\
    \midrule
    \rowcolor{HilightColor} HFC-P & $\pball{n}$ & $O(Bnm)$ & $nm+m$ \\
    \rowcolor{HilightColor} HFC-K & $\klein{n}$ & $O(Bnm)$ & $nm+m$ \\
    \rowcolor{HilightColor} HFC-H & $\bbh{n}$ & $O(Bnm)$ & $nm+m$ \\
    \bottomrule
  \end{tabular}
\end{table}

We compare the computational complexity and parameter dimensions of the hyperbolic FC layers.

\textbf{Analysis.}
We consider a hyperbolic FC layer that maps a batch of $B$ samples from an $n$-dimensional hyperbolic space to an $m$-dimensional hyperbolic space in the reduction setting $m\le n$. NestFC uses two matrix-manifold parameters, $Q\in\mathrm{SO}(n)$ and $\widetilde P\in\mathrm{St}(m,n)$ \citep[Sec. 3.3]{fan2022nested}, which introduce additional complexity and parameters.
\begin{itemize}
    \item \textbf{Asymptotic complexity.}
    As summarized in \cref{app:tab:hyperbolic_fc_complexity_parameters}, all other FC layers have complexity $O(Bnm)$, whereas NestFC has the highest complexity, $O(n^3+Bn^2)$. Its additional $O(n^3)$ term comes from the matrix exponentials used to construct its matrix-manifold parameters.
    \item \textbf{Parameter dimension.}
    As summarized in \cref{app:tab:hyperbolic_fc_complexity_parameters}, NestFC has the largest and fastest-growing parameter dimension among the compared FC layers when $m$ is fixed and $n$ increases. Specifically, $Q$ contributes $n(n-1)/2$ dimensions and $\widetilde P$ contributes $mn-m(m+1)/2$ dimensions.
\end{itemize}

\subsubsection{Comparison under the NHGCN architecture}
\label{app:subsubsec:nhgcn_comparison}

\begin{table}[t]
  \centering
  \caption{Hyperparameter-search procedure.}
  \label{app:tab:nhgcn_search}
  \resizebox{\textwidth}{!}{
  \begin{tabular}{c|l|l}
    \toprule
    Step & Hyperparameter & Candidates \\
    \midrule
    1 & Feature normalization (\texttt{normalize\_feats}) & $0, 1$ \\
    2 & Activation & None, ReLU, LeakyReLU, ELU, Tanh \\
    3 & Learning rate & $0.001, 0.005, 0.008, 0.01, 0.02$ \\
    4 & Gyro bias & $0, 1$ for HFC-H only. NestFC skips this step \\
    5 & Gradient clipping & None, $0.5, 1.0$ \\
    6 & Attention and local aggregation & $(0,0), (1,1)$ for (\texttt{use\_att}, \texttt{local\_agg}) \\
    7 & Weight decay & $0, 0.0001, 0.0005, 0.001, 0.002, 0.005, 0.01, 0.05, 0.1$ \\
    8 & Dropout & $0, 0.1, 0.2, 0.3, 0.4, 0.5, 0.6, 0.7, 0.8, 0.9$ \\
    \bottomrule
  \end{tabular}}
\end{table}

\begin{table}[t]
  \centering
  \caption{Link-prediction test AUC (\%).}
  \label{app:tab:nhgcn_auc}
  \begin{tabular}{l|cccc}
    \toprule
    Method & Disease & Airport & Pubmed & Cora \\
    \midrule
    NHGCN-NestFC & 78.28 ± 0.73 & 94.49 ± 0.04 & 94.70 ± 0.21 & \ColFirst{93.11 ± 0.40} \\
    \midrule
    \rowcolor{HilightColor} NHGCN-HFC-H & \ColFirst{97.37 ± 0.22} & \ColFirst{96.14 ± 0.27} & \ColFirst{96.18 ± 0.06} & 92.85 ± 0.24 \\
    \bottomrule
  \end{tabular}
\end{table}

\begin{table}[t]
  \centering
  \caption{Full-model trainable parameters, peak memory (MiB), and FitTime (ms/epoch) for NHGCN-NestFC and NHGCN-HFC-H.}
  \label{app:tab:nhgcn_resources}
  \resizebox{\textwidth}{!}{
  \begin{tabular}{l|rrr|rrr|rrr|rrr}
    \toprule
    \multirow{2}{*}{Method} & \multicolumn{3}{c|}{Disease} & \multicolumn{3}{c|}{Airport} & \multicolumn{3}{c|}{Pubmed} & \multicolumn{3}{c}{Cora} \\
    & \#Param & MiB & FitTime & \#Param & MiB & FitTime & \#Param & MiB & FitTime & \#Param & MiB & FitTime \\
    \midrule
    NHGCN-NestFC & 738 & \ColFirst{81.10} & 27.039 & 776 & \ColFirst{120.29} & \ColFirst{43.292} & 257,952 & 3496.40 & 105.094 & 2,075,436 & 1266.51 & 94.458 \\
    \midrule
    \rowcolor{HilightColor} NHGCN-HFC-H & \ColFirst{450} & 92.88 & \ColFirst{19.854} & \ColFirst{465} & 134.34 & 44.638 & \ColFirst{7,785} & \ColFirst{3444.32} & \ColFirst{91.064} & \ColFirst{21,780} & \ColFirst{253.68} & \ColFirst{24.007} \\
    \bottomrule
  \end{tabular}}
\end{table}

The main experiments compare different hyperbolic FC layers within the HNN framework, and we further compare HFC-H and NestFC within the NHGCN \citep{fan2022nested} framework, which is based on the hyperboloid model.

\textbf{Setting.}
We replace only NestFC in NHGCN with HFC-H for feature transformation \citep[Sec. 3.2]{fan2022nested}. All other NHGCN modules remain unchanged. Both methods use the same two-layer NHGCN architecture with a hidden dimension of 16 and identical input and output dimensions in each layer. For link prediction, the public paper and code\footnote{\url{https://github.com/cvgmi/Nested-Hyperbolic-DimReduc-and-HNN}} specify neither the hyperparameter-search procedure nor the final dataset-specific hyperparameters. For a fair comparison, we conduct two rounds of hyperparameter search for both methods, as shown in \cref{app:tab:nhgcn_search}.

\textbf{Results.}
\Cref{app:tab:nhgcn_auc} compares the best HFC-H and NestFC results.
\begin{itemize}
    \item \textbf{Better performance.}
    HFC-H exceeds NestFC by 19.09, 1.65, and 1.48 AUC points on Disease, Airport, and Pubmed, respectively. On Cora, the two selected outcomes differ by only 0.26 point.
    \item \textbf{Lower resource use.}
    Although both methods use the same two-layer architecture and identical input and output dimensions in every layer, HFC-H uses fewer trainable parameters than NestFC. On the high-dimensional Pubmed and Cora inputs, it also requires less peak memory and FitTime. Particularly on Cora, NestFC uses $\mathbf{95.3\times}$ as many parameters, $\mathbf{5.0\times}$ as much peak memory, and $\mathbf{3.9\times}$ as much FitTime as HFC-H.
\end{itemize}

\subsection{SPD manifolds}
\label{app:subsec:exp_details_spd}

\subsubsection{Datasets}

\textbf{Radar\footnote{\url{https://www.dropbox.com/s/dfnlx2bnyh3kjwy/data.zip?dl=0}}  \citep{brooks2019riemannian}. }
It consists of 3,000 synthetic radar signals equally distributed across 3 classes.

\textbf{HDM05\footnote{\url{https://resources.mpi-inf.mpg.de/HDM05/}} \citep{muller2007documentation}.}
It consists of 2,343 skeleton-based motion capture sequences executed by different actors. Each frame consists of 3D coordinates of 31 joints. We remove the under-represented clips, trimming the dataset down to 2,326 instances scattered throughout 122 classes. We randomly select 50\% of the samples from each category for training and the remaining 50\% for testing.

\textbf{FPHA\footnote{\url{https://github.com/guiggh/hand_pose_action}} \citep{garcia2018first}.}  It includes 1,175 skeleton-based first-person hand gesture videos of 45 different categories with 600 clips for training and 575 for testing. Each frame contains the 3D coordinates of 21 hand joints.

For the HDM05 and FPHA datasets, we preprocess each sequence using the code\footnote{\url{https://ravitejav.weebly.com/kbac.html}} provided by \citet{vemulapalli2014human} to normalize body part lengths and ensure invariance to scale and view.

\subsubsection{SPD modeling}
For our SPDNNs, we follow \citet{wang2024grassatt,nguyen2024matrix} to model each sample into a multi-channel SPD tensor. For the Radar dataset, we follow \citet{wang2024grassatt} to use the temporal convolution followed by a covariance pooling layer to obtain a multi-channel covariance $[c,20,20]$ tensor. For the HDM05 and FPHA datasets, we follow \citet[Sec. D.2.2]{nguyen2024matrix} to model each skeleton sequence into a multi-channel covariance tensor $[c,n,n]$. Specifically, we first identify the closest left (right) neighbor of every joint based on their distance to the hip (wrist) joint, and then combine the 3D coordinates of each joint and those of its left (right) neighbor to create a feature vector for the joint. For a given frame $t$, we compute its Gaussian embedding \citep{lovric2000multivariate}:
\begin{equation}
    Y_t=(\operatorname{det} \Sigma_{t})^{-\frac{1}{n+1}}
    \left[\begin{array}{cc} \Sigma_t+\mu_t\left(\mu_t\right)^T & \mu_t \\
    \left(\mu_t\right)^T & 1
\end{array}\right],
\end{equation}
where $\mu_t$ and $\Sigma_t$ are the mean vector and covariance matrix computed from the set of feature vectors within the frame. The lower part of matrix $\log \left(Y_t\right)$ is flattened to obtain a vector $\tilde{v}_t$. All vectors $\tilde{v}_t$ within a time window $[t, t+c-1]$, where $c$ is determined from a temporal pyramid representation of the sequence (the number of temporal pyramids is set to 2 in our experiments), are used to compute a covariance matrix as
\begin{equation}
    Z_t=\frac{1}{c} \sum_{i=t}^{t+c-1}\left(\tilde{v}_i-\overline{v}_t\right)\left(\tilde{v}_i-\overline{v}_t\right)^T,
\end{equation}
where $\overline{v}_t=\frac{1}{c} \sum_{i=t}^{t+c-1} \tilde{v}_i$. The resulting $\{Z_t\}$ is the input covariance tensor. On the FPHA dataset, we generate the covariance based on three sets of neighbors: left, right, and vertical (bottom) neighbors.

For GyroLE, GyroAI, GyroLC, and GyroSPD++, the inputs are similar to those of our SPDNNs. For other SPD baselines, such as SPDNet, SPDNetBN, LieBN, MLR, and RResNet, each sequence is represented by a global covariance representation \citep{huang2017riemannian,brooks2019riemannian}. The sizes of the covariance matrices are $20 \times 20$, $93 \times 93$, and $63 \times 63$ for the Radar, HDM05, and FPHA datasets, respectively.

\begin{table}[htbp]
  \centering
  \caption{Training hyperparameters in SPDNNs.}
    \begin{tabular}{c|c|ccc}
    \toprule
    Dataset & Model & $\theta$ & Optimizer & Learning Rate  \\
    \midrule
    \multirow{5}[1]{*}{Radar} & SPDNN-LEM & \na    & AMSGrad & $5e^{-3}$ \\
    & SPDNN-AIM & 0.25  & AMSGrad & $5e^{-4}$  \\
    & SPDNN-PEM & \na    & AMSGrad & $1e^{-2}$  \\
    & SPDNN-LCM & 0.25  & AMSGrad & $5e^{-4}$  \\
    & SPDNN-BWM & \na    & AMSGrad & $5e^{-4}$  \\
    \midrule
    \multirow{5}[0]{*}{HDM05} & SPDNN-LEM & \na    & SGD   & $5e^{-3}$ \\
    & SPDNN-AIM & \na    & SGD   & $5e^{-3}$  \\
    & SPDNN-PEM & \na    & AMSGrad & $1e^{-3}$  \\
    & SPDNN-LCM & \na    & AMSGrad & $1e^{-3}$  \\
    & SPDNN-BWM & \na    & AMSGrad & $1e^{-3}$ \\
    \midrule
    \multirow{5}[1]{*}{FPHA} & SPDNN-LEM & \na    & AMSGrad & $1e^{-4}$\\
    & SPDNN-AIM & \na    & AMSGrad & $1e^{-4}$  \\
    & SPDNN-PEM & \na    & AMSGrad & $1e^{-3}$  \\
    & SPDNN-LCM & -0.25 & AMSGrad & $1e^{-3}$  \\
    & SPDNN-BWM & -0.25 & AMSGrad & $1e^{-4}$  \\
    \midrule
    \multirow{5}[1]{*}{NTU60} & SPDNN-LEM & \na    & SGD & $1e^{-3}$\\
    & SPDNN-AIM & \na    & AMSGrad & $1e^{-4}$  \\
    & SPDNN-PEM & \na    & AMSGrad & $5e^{-4}$  \\
    & SPDNN-LCM & 0.25 & AMSGrad & $5e^{-4}$  \\
    & SPDNN-BWM & 0.25 & AMSGrad & $1e^{-3}$  \\
    \bottomrule
    \end{tabular}%
  \label{app:tab:spd_hyperparameters}%
\end{table}%

\subsubsection{Implementation details}
\textbf{Comparative methods.}
We follow the official PyTorch code of
SPDNetBN\footnote{\url{https://proceedings.neurips.cc/paper\_files/paper/2019/file/6e69ebbfad976d4637bb4b39de261bf7-Supplemental.zip}} to implement SPDNet and SPDNetBN.
For LieBN\footnote{\url{https://github.com/GitZH-Chen/LieBN}}, we focus on the instantiations under AIM and LCM, while for RResNet\footnote{\url{https://github.com/CUAI/Riemannian-Residual-Neural-Networks}}, we implement the variants induced by LEM and AIM. For SPD MLR\footnote{\url{https://github.com/GitZH-Chen/SPDMLR}}, we implement the variant induced by LCM. For GyroLE, GyroAI, GyroLC, and GyroSPD++, we re-implemented them based on the original papers \citep{nguyen2023building,nguyen2024matrix}.

\textbf{SPDNNs.}
On all datasets, we employ a single convolutional kernel for global convolution, \ie, applying a global receptive field across the channel dimension. The output dimensions of the SPD convolutional layer are $8 \times 8$, $34 \times 34$, $22 \times 22$, and $11 \times 11$ for the Radar, HDM05, FPHA, and NTU60 datasets, respectively. We primarily use the AMSGrad \citep{reddi2018convergence} optimizer, except for SPDNN-LEM and SPDNN-AIM on the HDM05 dataset and SPDNN-LEM on NTU60, where SGD \citep{robbins1951stochastic} is employed. Weight decay is set to zero except for SPDNN-PEM and SPDNN-BWM on the FPHA dataset, where it is $5e^{-4}$ and $1e^{-4}$, respectively. The matrix power in SPDNN-PEM is set to 0.5 for Radar and 0.25 for the other three datasets. Since matrix power can deform the latent Riemannian metric \citep[Fig. 1]{chen2024rmlr}, we also apply matrix power $(\cdot)^\theta$ before the convolutional layer in SPDNN-AIM, -LCM, and -BWM to activate the latent geometries. The batch size is set to 30, and training runs for 150 epochs with early stopping. \cref{app:tab:spd_hyperparameters} summarizes the training hyperparameters.

\subsubsection{Reproduction Fidelity and Controlled Comparisons}
\label{app:subsec:reproduction_fidelity}

Gyro \citep{nguyen2023building} and GyroSPD++ \citep{nguyen2024matrix} baselines are repproduced due to the unavailability of their official code.

\textbf{Reproduction fidelity.} Two factors may explain gaps between our reproduced and published results:
\begin{itemize}
    \item \textbf{Input preprocessing.} We follow the original papers \citep{nguyen2023building,nguyen2024matrix}, but the authors did not release processed inputs, and the papers omit the raw-skeleton preprocessing, ambiguous joint-to-neighbor cases \citep[Supp. Sec.~6.1]{nguyen2022gyro}, and the temporal-pyramid partition. Thus, our inputs may not exactly match theirs:
    \begin{equation*}
    \begin{aligned}
    \text{raw skeletons}
    &\rightarrow \boxed{\text{unreported raw-data preprocessing}} \\
    &\rightarrow \boxed{\text{underdetermined joint-neighbor selection}} \\
    &\rightarrow \boxed{\text{unspecified temporal-pyramid partition}} \\
    &\rightarrow \text{SPD inputs}.
    \end{aligned}
    \end{equation*}
    \item \textbf{Network implementation.} Because code for Gyro \citep{nguyen2023building} and GyroSPD++ \citep{nguyen2024matrix} is unavailable, we reimplemented them. For a controlled comparison, these baselines and our SPDNNs share inputs, pipeline, and code for common matrix operations, including Riemannian operators and matrix functions. We also tested the official training settings in Gyro \citep[App.~A.1]{nguyen2023building} and GyroSPD++ \citep[App.~D.2.1]{nguyen2024matrix}, but they did not recover the published accuracies. We therefore report each baseline's best selected configuration in this shared pipeline.
\end{itemize}

\textbf{Published vs. reproduced results.} Among the three datasets, only FPHA uses the same official split. On HDM05, they use 130 classes and a subject split, whereas we follow \citet[App.~G.2.2]{chen2024rmlr} to use 122 classes and a per-class 50/50 split. On NTU60, they use cross-subject, whereas we use cross-view. \cref{app:tab:reproduction_fidelity} therefore reports only FPHA. All published FPHA results exceed our reproduced results. Even SPDNet and SPDNetBN, which have official code, show gaps of $3.20$ and $1.69$ points.

\begin{table}[htbp]
  \centering
  \caption{FPHA accuracy (\%): Source \citep{nguyen2023building,nguyen2024matrix}, Ours, and Gap (Ours $-$ Source; percentage points).}
  \label{app:tab:reproduction_fidelity}
  \small
  \begin{tabular}{cccc}
    \toprule
    \multirow{2}{*}{Method} & \multicolumn{3}{c}{FPHA} \\
    \cmidrule(lr){2-4}
    & Source & Ours & Gap \\
    \midrule
    SPDNet (official code) & $88.79 \pm 0.36$ & $85.59 \pm 0.72$ & $-3.20$ \\
    SPDNetBN (official code) & $91.02 \pm 0.25$ & $89.33 \pm 0.49$ & $-1.69$ \\
    GyroLE & $94.61$ & $90.73 \pm 0.92$ & $-3.88$ \\
    GyroLC & $82.43$ & $76.10 \pm 0.63$ & $-6.33$ \\
    GyroAI & $93.39$ & $89.60 \pm 0.37$ & $-3.79$ \\
    GyroSPD++-AIM (source: AI-LE) & $96.84 \pm 0.27$ & $89.50 \pm 0.37$ & $-7.34$ \\
    GyroSPD++-LEM (source: LE-LE) & $94.72 \pm 0.25$ & $88.23 \pm 0.62$ & $-6.49$ \\
    \bottomrule
  \end{tabular}
\end{table}

\textbf{Controlled fairness.} This shared pipeline controls our internal comparison, and the Gyro \citep{nguyen2023building} and GyroSPD++ \citep{nguyen2024matrix} architectures follow the original papers. SPDNN and GyroSPD++ also use the same architecture, consisting of one SPD convolutional layer followed by an SPD MLR classifier.

\subsubsection{Training efficiency}
\label{app:subsec:efficiency_spd}

\begin{table}[htbp]
  \centering
  \caption{Training efficiency (seconds per epoch).}
    \begin{tabular}{c|c|cccc}
    \toprule
    Method & Geometry & Radar & HDM05 & FPHA & NTU60 \\
    \midrule
    SPDNet & \na  & 0.66  & 0.50  & 0.28  & 3.08 \\
    SPDNetBN & AIM   & 1.25  & 0.94  & 0.58  & 6.14 \\
    SPDResNet-AIM & AIM   & 0.96  & 1.23  & 0.69  & 6.84 \\
    SPDResNet-LEM & LEM   & 0.77  & 0.55  & 0.30  & 3.17 \\
    SPDNetLieBN-AIM & AIM   & 1.21  & 1.15  & 0.97  & 8.85 \\
    SPDNetLieBN-LCM & LCM   & 1.10  & 1.11  & 0.59  & 5.96 \\
    SPDNetMLR & LCM   & 0.66  & 5.46  & 0.88  & 4.94 \\
    GyroLE & LEM   & 0.79  & 2.86  & 1.59  & 10.57 \\
    GyroLC & LCM   & 0.66  & 1.49  & 0.78  & 5.99 \\
    GyroAI & AIM   & 0.99  & 22.80  & 12.62  & 26.76 \\
    \midrule
    \rowcolor{HilightColortwo} GyroSPD++-AIM & AIM   & 5.09  & 103.57  & 66.35  & 125.05 \\
    \rowcolor{HilightColortwo} GyroSPD++-LEM & LEM & 0.99  & 0.95  & 0.66  & 7.58 \\
    \rowcolor{HilightColortwo} GyroSPD++-LCM & LCM & 0.66  & 0.70  & 0.37  & 5.74 \\
    \midrule
    \rowcolor{HilightColor} SPDNN-LEM & LEM   & 0.86  & 0.74  & 0.63  & 5.79 \\
    \rowcolor{HilightColor} SPDNN-AIM & AIM   & 4.84 & 101.80  & 65.42  & 124.41 \\
    \rowcolor{HilightColor} SPDNN-PEM & PEM   & 1.09  & 7.10  & 1.57  & 8.71 \\
    \rowcolor{HilightColor} SPDNN-LCM & LCM   & 0.65  & 0.59  & 0.35  & 3.72 \\
    \rowcolor{HilightColor} SPDNN-BWM & BWM   & 6.07  & 110.51  & 71.67  & 139.48 \\
    \bottomrule
    \end{tabular}%
  \label{app:tab:efficiency}
\end{table}

\cref{app:tab:efficiency} presents the average training time per epoch of each SPD network. We have the following observations:
\begin{itemize}
    \item
    \textbf{The efficiency of SPDNN varies across metrics.}
    The most efficient metric is LCM, where our model even achieves comparable efficiency to the vanilla SPDNet. However, AIM and BWM demonstrate significant computational burden, primarily due to their complex Riemannian computations.
    \item
    \textbf{Our trivialization improves efficiency.}
    Compared with the LCM-based SPDNetMLR, SPDNN-LCM achieves much lower training time. This improvement can be partially attributed to our trivialization, which simplifies the final expression of MLR (\cref{app:sec:simplified_sdp_mlr}) and eliminates the need for computationally expensive Riemannian optimization. Moreover, SPDNN consistently outperforms GyroSPD++ under LEM, LCM, and AIM in terms of efficiency. This advantage arises because our trivialization not only simplifies the expression of the FC and MLR layers, but also reduces the number of parameters.
\end{itemize}

\subsubsection{Comparison with ManifoldNet and MVC-Net}
\label{app:subsec:local_spd_convolution}

\begin{table}[t]
  \centering
  \caption{SPD CNN architectures with local receptive fields and shared kernels.}
  \label{app:tab:local_spd_architectures}
  \resizebox{\textwidth}{!}{
  \begin{tabular}{c|c|c|c|c|c|c}
    \toprule
    Input shape & \#Conv & Kernels & Kernel sizes & Strides & Shape changes & Final SPD features \\
    \midrule
    $[40,3,3]$ & 3 & $[4,8,16]$ & $[3,3,3]$ & $[2,2,2]$ & $[40,3,3]\to[19,4,3,3]\to[9,8,3,3]\to[4,16,3,3]$ & $[64,3,3]$ \\
    $[20,6,6]$ & 3 & $[4,8,16]$ & $[2,2,2]$ & $[2,2,1]$ & $[20,6,6]\to[10,4,6,6]\to[5,8,6,6]\to[4,16,6,6]$ & $[64,6,6]$ \\
    $[5,12,12]$ & 2 & $[8,20]$ & $[2,2]$ & $[1,1]$ & $[5,12,12]\to[4,8,12,12]\to[3,20,12,12]$ & $[60,12,12]$ \\
    \bottomrule
  \end{tabular}}
\end{table}

\begin{table}[t]
  \centering
  \caption{Radar classification accuracy (\%) and FitTime (s/epoch) under different architectures.}
  \label{app:tab:local_spd_results}
  \resizebox{\textwidth}{!}{
  \begin{tabular}{l|cc|cc|cc}
    \toprule
    \multirow{2}{*}{Method} & \multicolumn{2}{c|}{$[40,3,3]$} & \multicolumn{2}{c|}{$[20,6,6]$} & \multicolumn{2}{c}{$[5,12,12]$} \\
    & Accuracy & FitTime & Accuracy & FitTime & Accuracy & FitTime \\
    \midrule
    ManifoldNet & 77.89 ± 1.56 & 13.12 & 73.27 ± 0.68 & 11.58 & 80.75 ± 1.23 & 11.50 \\
    MVC-Net & 81.04 ± 1.72 & 7.88 & 78.05 ± 0.23 & 8.18 & 81.15 ± 0.34 & 7.87 \\
    \midrule
    \rowcolor{HilightColor} SPDConvNet-LEM & \ColFirst{97.68 ± 0.32} & 1.58 & \ColFirst{95.86 ± 0.95} & 1.55 & 96.91 ± 0.18 & \ColFirst{$1.03$} \\
    \rowcolor{HilightColor} SPDConvNet-LCM & 96.51 ± 0.66 & \ColFirst{$1.54$} & 94.40 ± 0.50 & \ColFirst{$1.22$} & \ColFirst{97.20 ± 0.65} & 1.25 \\
    \rowcolor{HilightColor} SPDConvNet-PEM & 97.44 ± 0.56 & 1.82 & 94.45 ± 0.78 & 1.53 & 96.16 ± 0.91 & 2.34 \\
    \bottomrule
  \end{tabular}}
\end{table}

\begin{table}[t]
  \centering
  \caption{ManifoldNet learning-rate ablation under the $[40,3,3]$ input setting.}
  \label{app:tab:manifoldnet_lr_ablation}
  \resizebox{\textwidth}{!}{
  \begin{tabular}{c|cccccccc}
    \toprule
    Learning rate & $5\times10^{-5}$ & $10^{-4}$ & $2\times10^{-4}$ & $5\times10^{-4}$ & $10^{-3}$ & $2\times10^{-3}$ & $5\times10^{-3}$ & $10^{-2}$ \\
    \midrule
    Accuracy (\%) & 55.20 ± 2.58 & 59.73 ± 1.31 & 63.52 ± 1.44 & 69.55 ± 2.46 & 73.28 ± 1.42 & 76.08 ± 1.08 & \ColFirst{77.89 ± 1.56} & 77.12 ± 1.30 \\
    \bottomrule
  \end{tabular}}
\end{table}

To explicitly evaluate local receptive fields and kernel sharing, we conduct additional 1D convolutional experiments on Radar and compare our method with ManifoldNet and MVC-Net. Our main SPD experiments follow Gyro \citep{nguyen2023building} and GyroSPD++ \citep{nguyen2024matrix}, representing each sample with only a few SPD descriptors and applying a single global receptive field across them. In this regime, convolution reduces to an FC layer on the product manifold and therefore does not directly evaluate local convolution. ManifoldNet and MVC-Net instead operate on manifold-valued fields $f:U\to\calM$, where $U\subset\bbZ^d$ indexes spatial or temporal positions and each position stores a point on the manifold $\calM$ \citep{chakraborty2020manifoldnet,bouza2020mvc}. Their convolutional layers slide local windows over these positions and share kernel weights across positions.

\textbf{Settings.}
Local receptive fields and kernel sharing require sufficiently many spatial or temporal positions. We therefore use Radar and split each $[2,1000]$ signal into $C$ consecutive temporal blocks. Covariance pooling within each block produces a 1D CNN input of shape $[C,n,n]$, where $C$ is the number of temporal positions and each $[n,n]$ slice is an SPD matrix. We focus on three settings: $[40,3,3]$, $[20,6,6]$, and $[5,12,12]$. \cref{app:tab:local_spd_architectures} gives the corresponding convolutional architectures. We refer to our SPD convolutional network as SPDConvNet and apply tangent ReLU after every convolutional layer. Following its original implementation, ManifoldNet uses no activation. MVC-Net includes tangent ReLU in its original formulation \citep{bouza2020mvc}, but in our preliminary runs it caused numerical instability and yielded NaNs. We therefore omit tangent ReLU for MVC-Net.

\textbf{Experiments.}
We compare ManifoldNet \citep{chakraborty2020manifoldnet} and MVC-Net \citep{bouza2020mvc} with SPDConvNet instantiated with LEM, LCM, and PEM. All methods use the matched architectures in \cref{app:tab:local_spd_architectures}. As shown in \cref{app:tab:local_spd_results}, all three SPDConvNet variants achieve higher accuracy and train faster across the three settings. The efficiency difference arises from the repeated computation of weighted Fr\'echet means in ManifoldNet \citep[Eq. (8)]{chakraborty2020manifoldnet} and Fr\'echet means in MVC-Net \citep[Def. 1]{bouza2020mvc} during convolution.

\textbf{Learning-rate ablation.}
We additionally scan eight learning rates for ManifoldNet under the $[40,3,3]$ input setting. Its best result is $77.89\%$, confirming that the comparison is not caused by an untuned learning rate.

\subsection{Grassmannian manifolds}
\label{app:subsec:exp_details_grass}

\textbf{Grassmannian modeling.}
As Grassmannian descriptors can be derived by the SVD of the covariance \citep{huang2018building,nguyen2023building}, we map the multi-channel Radar covariance into a $[c,n,p]$ ONB Grassmannian tensor via the SVD. The PP Grassmannian features can be derived from the ONB Grassmannian features via the isometry $\pi(\cdot):\grasonb{p,n} \rightarrow \graspp{p,n}$:
\begin{equation}
    \pi(U)=U U ^\top, \forall U \in \grasonb{p,n}.
\end{equation}

\textbf{Implementation details.}
Since GrNet \citep{huang2018building} is officially implemented in Matlab, we carefully re-implemented it using PyTorch. Additionally, since both GyroGr and GyroGr-Scaling do not release official code, we re-implemented them based on the original papers \citep{nguyen2023building}. For all comparative methods, we use SGD with a learning rate of $5e^{-2}$. For training our ONB and PP GrNNs, we use AMSGrad with a learning rate of $5e^{-3}$. The batch size is set to 30, and training runs for 150 epochs.

\subsection{Hardware}
\label{app:subsec:hardware}

On the HDM05 and FPHA datasets, SPDNet, RResNet, SPDNetBN, SPDNetLieBN, and MLR require SVD operations on relatively large matrices, which are more efficiently executed on a CPU. As a result, these methods are implemented on a CPU, whereas all other cases are executed on a single A6000 GPU.

\section{Proofs}
\label{app:sec:proofs}

\linkofproof{prop:riem_fc_gen_euc_fc}
\subsection{Proof of \cref{prop:riem_fc_gen_euc_fc}}
\begin{proof}
    \textbf{Euclidean spaces.}
    We first review the following facts about Euclidean space:
    \begin{itemize}
        \item
        the origin is the zero vector $\bbzero$;
        \item
        the standard orthonormal basis over $T_{0}\bbR{m} \cong \bbR{m}$ is $\{e_i\}_{i=1}^m$;
        \item
        $\rielog _p(x) = x-p$ and $\inner{\cdot}{\cdot}_p = \inner{\cdot}{\cdot}$ for any $x,p \in \bbR{n}$;
        \item
        point-to-hyperplane distance is $d(x,H_{a, p}) = \frac{|\langle  x-p, a \rangle |}{\| a \|}$.
    \end{itemize}
    Putting the above together, one can recover \cref{eq:euc_fc}.

    \textbf{Poincaré balls.}
    This exactly corresponds to the derivation of the Poincaré FC layer \citep[App. D.3]{shimizu2021hyperbolic}.

    \textbf{SPD gyrovector spaces.}
    \citet{nguyen2024matrix} proposed three gyro SPD FC layers based on the gyrovector structures under LEM, LCM, and AIM, respectively. The discussion below summarizes the proof in \citet[Apps. J-L]{nguyen2024matrix}.

    In the SPD gyro FC layer, the origin of the SPD manifold is the identity matrix $I$. Given a metric among LEM, LCM, and AIM, let $\{B_i\}_{i=1}^d$ be an orthonormal basis over $T _I \spd{m}$, where $d=\nicefrac{n(n+1)}{2}$ is the dimension of $\spd{n}$. The SPD gyro FC layer is defined by solving the following equations
    \begin{equation} \label{app:eq:gyro_fc}
        \sign \left( \left\langle \rielog _{I} (Y), B _k \right\rangle _I \right) \dist_{\mathrm{gyo}} (Y, H _{B_k, I}) = \langle W_k,  \ominus P_k \oplus X \rangle _{gr}, \quad 1 \leq k \leq m,
    \end{equation}
    where $\dist_{\mathrm{gyo}}$ denotes the corresponding pseudo-gyrodistance, $\ominus$ and $\oplus$ are gyro operations \citep[Apps. G.2-G.4]{nguyen2024matrix}, and $\inner{\cdot}{\cdot}_{gr}$ is the gyro inner product \citep[App. G.7]{nguyen2024matrix}. Here, each $W_k \in \spd{n}$ and $P_k \in \spd{n}$ are FC parameters.
    By Prop. 3.2 in \citet{nguyen2024matrix}, the RHS of \cref{app:eq:gyro_fc} is equal to $\inner{\rielog _{P_k} (X)}{\pt{I}{P_{k}}\left(\rielog_{I} (W_{k})\right)}_{P_{k}}$. Setting $A_k = \pt{I}{P_{k}}\left(\rielog_{I} (W_{k})\right) \in T _{P_k} \spd{n}$, one can recover \cref{eq:riem_fc_original}.
\end{proof}

\linkofproof{thm:riem_fc}
\subsection{Proof of \cref{thm:riem_fc}}
\begin{proof}
    The point-to-hyperplane pseudo-distance from a point $Y \in \calM$ to a Riemannian hyperplane over $\calM$ is
    \begin{equation}
         \dist(Y,H_{A, P})
        = \frac{|\langle  \rielog^\calM_P Y, A \rangle^\calM _P|}{\| A \|^\calM _P},
    \end{equation}
    where $H_{A, P}$ is a Riemannian hyperplane parameterized by $P \in \calM$ and $A \in T _P\calM$. Therefore, the signed point-to-hyperplane pseudo-distance from $Y$ to $H_{B_i, E}$ is
    \begin{equation}
    \label{app:eq:singed_dist}
    \begin{aligned}
        \sign \left( \left\langle \rielog^\calM_E (Y), B_i \right\rangle^\calM_E \right) \dist(Y,H_{B_i, E})
        &= \frac{\langle  \rielog^\calM_E (Y), B_i \rangle^\calM _E}{\| B_i \|^\calM_E}\\
        &\stackrel{(1)}{=}\langle  \rielog^\calM _E (Y), B_i \rangle^\calM _E \\
    \end{aligned}
    \end{equation}
    where (1) comes from the orthonormality of $B_i$.

    Setting \cref{app:eq:singed_dist} equal to $v_i(X)$, we have
    \begin{equation}
             \langle  \rielog^\calM _E (Y), B_i \rangle^\calM _E = \langle \rielog^\calN _{P_i}(X), A_i \rangle^\calN _{P_i}.
    \end{equation}
    The above equation indicates
    \begin{equation}
        \rielog^\calM _{E}(Y) = \sum_{i=1}^{m} \left( \langle \rielog^\calN _{P_i}(X), A_i \rangle^\calN _{P_i} B_i \right).
    \end{equation}
\end{proof}

\subsection{Proof of \cref{app:thm:pball_fc}}
\linkofproof{app:thm:pball_fc}

To simplify the Riemannian FC layer with the gyro structure, we first prove a useful lemma.
\begin{lemma} \label{app:lem:inner_equivalence}
    We assume that the manifold $\calM$ admits a gyrogroup \citep[Def. 2.2]{nguyen2022gyro} defined by\footnote{We assume all the involved Riemannian operators are well-defined.}
    \begin{equation} \label{app:eq:gyro_add_def}
        x \oplus y = \rieexp_{x} \left( \pt{e}{x} \left( \rielog_e \left( y   \right)\right)\right), \forall x,y \in \calM,
    \end{equation}
    where $e \in \calM$ is the origin of the manifold. Then, we have the following
    \begin{equation} \label{app:eq:gyro_inner}
        \left\langle \rielog_p (x ), a \right\rangle _{p}
        =\left\langle \rielog_e ( \ominus p \oplus x ), \pt{p}{e}(a) \right\rangle_{e}, \quad \forall x,p \in \calM \text{ and } \forall a \in T_p\calM.
    \end{equation}
\end{lemma}
\begin{proof}
    \textbf{Credit of the proof:}
    \cref{app:eq:gyro_add_def} comes from \citet[Eq. (1)]{nguyen2023building}, who demonstrated that several geometries admit gyrogroups based on this definition. The prototype of \cref{app:eq:gyro_inner} comes from App. I by \citet{nguyen2024matrix}, which only deals with SPD matrices. Here, we further extend the result into general gyrogroups.

    Denoting $\ominus p$ as the gyro inverse of $p$ ($\ominus p \oplus p=e$), we have
    \begin{equation}
        \begin{aligned}
          &x \stackrel{(1)}{=}  p \oplus ( \ominus p \oplus x) \stackrel{(2)}{=} \rieexp_{p} \left( \pt{e}{p} \left( \rielog_e \left( \ominus p \oplus x   \right)\right)\right)\\
          \stackrel{(3)}{\Rightarrow}  & \rielog_p(x) = \pt{e}{p} \left( \rielog_e \left( \ominus p \oplus x   \right)\right).
        \end{aligned}
    \end{equation}
        The above follows from the following.
    \begin{enumerate}[label=(\arabic*)]
        \item
        Left cancellation law of the gyrogroup \citep[Thms. 1.13]{ungar2022gyrovector}.
        \item
        Definition of gyro addition.
        \item
        Applying $\rielog_p(\cdot)$ to both sides.
    \end{enumerate}
    By the last equation, we have
    \begin{equation}
        \begin{aligned}
            \left\langle \rielog_p (x ), a \right\rangle _{p}
            &= \left\langle \pt{e}{p} \left( \rielog_e \left( \ominus p \oplus x   \right)\right), a \right\rangle _p\\
            &\stackrel{(1)}{=} \left\langle \rielog_e \left( \ominus p \oplus x   \right), \pt{p}{e} (a) \right\rangle _{e},\\
        \end{aligned}
    \end{equation}
    where (1) comes from
    \begin{itemize}
        \item
        Parallel transport preserves the norm \citep[Sec. 3.1]{do1992riemannian}.
        \item
        $\pt{p}{e} \circ \pt{e}{p}(v)=v, \forall v \in T_e\calM$.
    \end{itemize}
\end{proof}
Now we begin to prove \cref{app:thm:pball_fc}.
\begin{proof}[Proof of \cref{app:thm:pball_fc}]
    In both geometries, the origin is defined as the zero vector, as it is the identity element in each gyrovector space. We first deal with the Poincaré ball, followed by the Beltrami--Klein model.

    \textbf{Poincaré ball: }
    The Riemannian metric at the identity element is
    \begin{equation} \label{app:eq:pball_metric_at_zero}
        \inner{v}{w}_\bbzero = 4 \inner{v}{w}, \forall v,w \in T_\bbzero \pball{m}.
    \end{equation}
    Obviously, $\{\frac{1}{2} e_i\}_{i=1}^m$ is an orthonormal basis. \cref{app:lem:inner_equivalence} implies
    \begin{equation} \label{app:prf:eq:inner_poincare_fc}
        \begin{aligned}
            \left\langle \rielog_{p_i} (x), a_i \right\rangle _{p_i} \frac{1}{2} e_i
            &\stackrel{(1)}{=} \left\langle \rielog_\bbzero ( - {p_i} \Moplus x ), \pt{{p_i}}{\bbzero}({a_i}) \right\rangle_{\bbzero} \frac{1}{2} e_i \\
            &\stackrel{(2)}{=} 2\left\langle \rielog_\bbzero ( - {p_i} \Moplus x ), \pt{{p_i}}{\bbzero}({a_i}) \right\rangle e_i \\
            &\stackrel{(3)}{=} \left\langle \rielog_\bbzero ( - {p_i} \Moplus x ), z_i \right\rangle e_i.
        \end{aligned}
    \end{equation}
    The above comes from the following.
    \begin{enumerate}[label=(\arabic*)]
        \item
        \cref{app:lem:inner_equivalence} and $\Mominus p = -p, \quad \forall p \in \pball{n}$.
        \item
        \cref{app:eq:pball_metric_at_zero}.
        \item
        $a_{i}=\pt{\bbzero}{p_i}(z_i/2)$. This reparameterization preserves $p_i$ because $[z_i/2]=[z_i]$.
    \end{enumerate}

    \textbf{Beltrami--Klein model: }
    The Riemannian metric at the identity element is
    \begin{equation}
        \inner{v}{w}_\bbzero = \inner{v}{w}, \forall v,w \in T_\bbzero \klein{m}.
    \end{equation}
    Obviously, $\{ e_i\}_{i=1}^m$ is an orthonormal basis. \cref{app:lem:inner_equivalence,app:eq:klein-poincare-add-riem} implies that the above reasoning for the Poincaré ball can be transferred into the Beltrami--Klein model:
    \begin{equation}
        \begin{aligned}
            \left\langle \rielog_{p_i} (x), a_i \right\rangle _{p_i} e_i
            &\stackrel{(1)}{=} \left\langle \rielog_\bbzero ( - {p_i} \Eoplus x ), \pt{{p_i}}{\bbzero}({a_i}) \right\rangle_{\bbzero} e_i \\
            &= \left\langle \rielog_\bbzero ( - {p_i} \Eoplus x ), \pt{{p_i}}{\bbzero}({a_i}) \right\rangle e_i \\
            &\stackrel{(2)}{=} \left\langle \rielog_\bbzero ( - {p_i} \Eoplus x ), z_i \right\rangle e_i,
        \end{aligned}
    \end{equation}
    The above comes from the following.
    \begin{enumerate}[label=(\arabic*)]
        \item
        \cref{app:lem:inner_equivalence}, \cref{app:eq:klein-poincare-add-riem}, and $\Eominus p = -p, \quad$.
        \item
        $a_{i}=\pt{\bbzero}{p_i}(z_i)$.
    \end{enumerate}
\end{proof}

\subsection{Proof of \cref{app:thm:hyperboloid_fc}}
\linkofproof{app:thm:hyperboloid_fc}

\begin{proof}
    We only need to show the origin, the tangent space at the origin, and the inner product and an orthonormal basis over the tangent space at the origin.

    The hyperboloid is isometric to the Poincaré ball by the following diffeomorphism \citep{lee2006riemannian}:
    \begin{equation}
        \pi_{\pball{n} \rightarrow \bbh{n}}(x)=\left(\frac{1}{\sqrt{|K|}} \frac{1-K\|x\|^2}{1+K\|x\|^2} ; \frac{2 x^T}{1+K\|x\|^2}\right)^\top.
    \end{equation}
    The origin of hyperboloid is therefore defined as
    \begin{equation}
        e:= \pi_{\pball{n} \rightarrow \bbh{n}}(\bbzero)= \left(\frac{1}{\sqrt{|K|}}, 0 \cdots, 0 \right)^\top.
    \end{equation}

    The Riemannian metric and tangent space at $e$ are
    \begin{align}
        T_e{\bbh{n}} &=\{ (0,v^\top)^\top| v \in \bbR{n} \}, \\
        \langle (0,v^\top)^\top, (0,w^\top)^\top \rangle_e &= \langle v, w \rangle, \quad \forall (0,v^\top)^\top, (0,w^\top)^\top \in T_e{\bbh{n}}.
    \end{align}
    Likewise, $\{(0,e_i^\top)^\top\}_{i=1}^m$ is an orthonormal basis of $T_e{\bbh{m}}$ with $e_i \in \bbR{m}$.

    Combining the above with \cref{app:tab:riem_operators_hyperbolic}, we can instantiate \cref{thm:riem_fc} in the hyperboloid geometry.
\end{proof}

\linkofproof{thm:hfc_batch_closed_form}
\subsection{Proof of \cref{thm:hfc_batch_closed_form}}

\begin{proof}
    \textbf{Poincaré ball.}
    \begin{equation*}
        \begin{aligned}
            v_i(x)&\stackrel{(1)}{=}\left\langle\rielog_{\bbzero}((-p_i)\Moplus x),z_i\right\rangle\\
            &\stackrel{(2)}{=}\frac{\tanh^{-1}\left(\sqrt{|K|}\|(-p_i)\Moplus x\|\right)}
            {\sqrt{|K|}\|(-p_i)\Moplus x\|}
            \left\langle(-p_i)\Moplus x,z_i\right\rangle\\
            &\stackrel{(3)}{=}\resizebox{0.78\textwidth}{!}{$\displaystyle
            \frac{\tanh^{-1}\left(
            \sqrt{|K|}\left\|
            \frac{(1+K\|p_i\|^2)x-
            \left(1+2K\langle p_i,x\rangle-K\|x\|^2\right)p_i}
            {1+2K\langle p_i,x\rangle+K^2\|p_i\|^2\|x\|^2}
            \right\|\right)}
            {\sqrt{|K|}\left\|
            \frac{(1+K\|p_i\|^2)x-
            \left(1+2K\langle p_i,x\rangle-K\|x\|^2\right)p_i}
            {1+2K\langle p_i,x\rangle+K^2\|p_i\|^2\|x\|^2}
            \right\|}
            \left\langle
            \frac{(1+K\|p_i\|^2)x-
            \left(1+2K\langle p_i,x\rangle-K\|x\|^2\right)p_i}
            {1+2K\langle p_i,x\rangle+K^2\|p_i\|^2\|x\|^2},z_i
            \right\rangle$}\\
            &\stackrel{(4)}{=}\resizebox{0.78\textwidth}{!}{$\displaystyle
            \frac{
            \tanh^{-1}\left(
            \sqrt{\frac{|K|\|x-p_i\|^2}
            {1+2K\langle p_i,x\rangle+K^2\|p_i\|^2\|x\|^2}}
            \right)}
            {\sqrt{\frac{|K|\|x-p_i\|^2}
            {1+2K\langle p_i,x\rangle+K^2\|p_i\|^2\|x\|^2}}}
            \frac{
            (1+K\|p_i\|^2)\langle x,z_i\rangle
            -\left(1+2K\langle p_i,x\rangle-K\|x\|^2\right)\langle p_i,z_i\rangle}
            {1+2K\langle p_i,x\rangle+K^2\|p_i\|^2\|x\|^2}$}\\
            &\stackrel{(5)}{=}\resizebox{0.78\textwidth}{!}{$\displaystyle
            \frac{
            \tanh^{-1}\left(
            \sqrt{\frac{|K|\|x-p_i\|^2}
            {1+2K\langle p_i,x\rangle+K^2\|p_i\|^2\|x\|^2}}
            \right)}
            {\sqrt{\frac{|K|\|x-p_i\|^2}
            {1+2K\langle p_i,x\rangle+K^2\|p_i\|^2\|x\|^2}}}
            \frac{
            (1-K\|p_i\|^2)\langle x,z_i\rangle
            -(1-K\|x\|^2)\langle p_i,z_i\rangle}
            {1+2K\langle p_i,x\rangle+K^2\|p_i\|^2\|x\|^2}$}\\
            &\stackrel{(6)}{=}\resizebox{0.78\textwidth}{!}{$\displaystyle
            \frac{
            \tanh^{-1}\left(
            \sqrt{\frac{|K|\|x-p_i\|^2}
            {1+2K\langle p_i,x\rangle+K^2\|p_i\|^2\|x\|^2}}
            \right)
            \left[
            (1-K\|p_i\|^2)\langle x,z_i\rangle
            -(1-K\|x\|^2)\langle p_i,z_i\rangle
            \right]}
            {\sqrt{|K|\|x-p_i\|^2
            \left(1+2K\langle p_i,x\rangle+K^2\|p_i\|^2\|x\|^2\right)}}
            $}\\
            &\stackrel{(7)}{=}\resizebox{0.78\textwidth}{!}{$\displaystyle
            \|z_i\|
            \frac{\tanh^{-1}\left(\sqrt{|K|}Q_i^{\mathrm{P}}\right)}
            {\sqrt{|K|}Q_i^{\mathrm{P}}}
            \frac{
            \left[1+\tanh^2(\sqrt{|K|}\gamma_i)\right]\langle x,[z_i]\rangle
            -\frac{\tanh(\sqrt{|K|}\gamma_i)}{\sqrt{|K|}}
            \left(1+|K|\|x\|^2\right)}
            {D_i^{\mathrm{P}}}$}.
        \end{aligned}
    \end{equation*}

    The above comes from the following.
    \begin{enumerate}[label=(\arabic*)]
        \item The first equality is the HFC-P expression in \cref{app:thm:pball_fc}.
        \item From \cref{app:eq:log_0_klein},
        \begin{equation*}
            \rielog_{\bbzero}((-p_i)\Moplus x)=
            \frac{\tanh^{-1}(\sqrt{|K|}\|(-p_i)\Moplus x\|)}
            {\sqrt{|K|}\|(-p_i)\Moplus x\|}((-p_i)\Moplus x).
        \end{equation*}
        Taking its inner product with $z_i$ gives the second equality.
        \item Substituting $-p_i$ and $x$ into the Möbius addition in \cref{app:eq:mobius_addition} gives
        \begin{equation*}
            (-p_i)\Moplus x=
            \frac{
            (1+K\|p_i\|^2)x
            -\left(1+2K\langle p_i,x\rangle-K\|x\|^2\right)p_i}
            {1+2K\langle p_i,x\rangle+K^2\|p_i\|^2\|x\|^2}.
        \end{equation*}
        Inserting this vector into both occurrences of $(-p_i)\Moplus x$ gives the third equality.
        \item Expanding the squared norm of the numerator in (3) gives
        \begin{equation*}
            \begin{aligned}
                &\left\|
                (1+K\|p_i\|^2)x
                -\left(1+2K\langle p_i,x\rangle-K\|x\|^2\right)p_i
                \right\|^2\\
                &=(1+K\|p_i\|^2)^2\|x\|^2
                +\left(1+2K\langle p_i,x\rangle-K\|x\|^2\right)^2\|p_i\|^2\\
                &\quad-2(1+K\|p_i\|^2)
                \left(1+2K\langle p_i,x\rangle-K\|x\|^2\right)
                \langle p_i,x\rangle\\
                &=\|x-p_i\|^2
                \left(1+2K\langle p_i,x\rangle+K^2\|p_i\|^2\|x\|^2\right).
            \end{aligned}
        \end{equation*}
        Therefore,
        \begin{equation*}
            \|(-p_i)\Moplus x\|^2=
            \frac{\|x-p_i\|^2}
            {1+2K\langle p_i,x\rangle+K^2\|p_i\|^2\|x\|^2}.
        \end{equation*}
        The corresponding inner product is
        \begin{equation*}
            \left\langle(-p_i)\Moplus x,z_i\right\rangle=
            \frac{
            (1+K\|p_i\|^2)\langle x,z_i\rangle
            -\left(1+2K\langle p_i,x\rangle-K\|x\|^2\right)\langle p_i,z_i\rangle}
            {1+2K\langle p_i,x\rangle+K^2\|p_i\|^2\|x\|^2}.
        \end{equation*}
        Substituting these two identities into (3) gives the fourth equality.
        \item Since $p_i=\rieexp_{\bbzero}(\gamma_i[z_i])$, we have $p_i\parallel z_i$ and hence
        \begin{equation*}
            \langle p_i,x\rangle\langle p_i,z_i\rangle
            =\|p_i\|^2\langle x,z_i\rangle.
        \end{equation*}
        It follows that
        \begin{equation*}
            \begin{aligned}
                &(1+K\|p_i\|^2)\langle x,z_i\rangle
                -\left(1+2K\langle p_i,x\rangle-K\|x\|^2\right)\langle p_i,z_i\rangle\\
                &=(1-K\|p_i\|^2)\langle x,z_i\rangle
                -(1-K\|x\|^2)\langle p_i,z_i\rangle,
            \end{aligned}
        \end{equation*}
        which gives the fifth equality.
        \item The scalar denominators satisfy
        \begin{equation*}
            \begin{aligned}
                &\frac{1}
                {\sqrt{\frac{|K|\|x-p_i\|^2}
                {1+2K\langle p_i,x\rangle+K^2\|p_i\|^2\|x\|^2}}}
                \frac{1}
                {1+2K\langle p_i,x\rangle+K^2\|p_i\|^2\|x\|^2}\\
                &=\frac{1}
                {\sqrt{|K|\|x-p_i\|^2
                \left(1+2K\langle p_i,x\rangle+K^2\|p_i\|^2\|x\|^2\right)}}.
            \end{aligned}
        \end{equation*}
        This yields the sixth equality.
        \item The exponential map at the origin gives
        \begin{equation*}
            p_i=\rieexp_{\bbzero}(\gamma_i[z_i])
            =\frac{\tanh\left(\sqrt{|K|}\gamma_i\right)}{\sqrt{|K|}}[z_i].
        \end{equation*}
        Since $[z_i]=z_i/\|z_i\|$, we have
        \begin{equation*}
            \begin{aligned}
                \|p_i\|^2&=\frac{\tanh^2(\sqrt{|K|}\gamma_i)}{|K|},\\
                \langle p_i,x\rangle&=\frac{\tanh(\sqrt{|K|}\gamma_i)}{\sqrt{|K|}}\langle[z_i],x\rangle,\\
                \langle p_i,z_i\rangle&=\frac{\tanh(\sqrt{|K|}\gamma_i)}{\sqrt{|K|}}\|z_i\|,\\
                \langle x,z_i\rangle&=\|z_i\|\langle x,[z_i]\rangle,\\
                \|x-p_i\|^2&=\|x\|^2+\frac{\tanh^2(\sqrt{|K|}\gamma_i)}{|K|}
                -\frac{2\tanh(\sqrt{|K|}\gamma_i)}{\sqrt{|K|}}\langle x,[z_i]\rangle.
            \end{aligned}
        \end{equation*}
        Define
        \begin{equation*}
            \begin{aligned}
                D_i^{\mathrm{P}}
                &=1-2\sqrt{|K|}\tanh(\sqrt{|K|}\gamma_i)\langle x,[z_i]\rangle
                +|K|\tanh^2(\sqrt{|K|}\gamma_i)\|x\|^2,\\
                Q_i^{\mathrm{P}}
                &=\sqrt{\frac{
                \|x\|^2+\frac{\tanh^2(\sqrt{|K|}\gamma_i)}{|K|}
                -\frac{2\tanh(\sqrt{|K|}\gamma_i)}{\sqrt{|K|}}\langle x,[z_i]\rangle}
                {D_i^{\mathrm{P}}}}.
            \end{aligned}
        \end{equation*}
        When $Q_i^{\mathrm{P}}=0$, the quotient $\tanh^{-1}(\sqrt{|K|}Q_i^{\mathrm{P}})/(\sqrt{|K|}Q_i^{\mathrm{P}})$ is defined as $1$ by continuity.
        Substituting these identities into (6) and using $K=-|K|$ gives the seventh equality.
    \end{enumerate}

    \textbf{Beltrami--Klein model.}
    \begin{equation*}
        \begin{aligned}
            v_i(x)&\stackrel{(1)}{=}
            \left\langle\rielog_{\bbzero}((-p_i)\Eoplus x),z_i\right\rangle\\
            &\stackrel{(2)}{=}
            \frac{\tanh^{-1}\left(\sqrt{|K|}\|(-p_i)\Eoplus x\|\right)}
            {\sqrt{|K|}\|(-p_i)\Eoplus x\|}
            \left\langle(-p_i)\Eoplus x,z_i\right\rangle\\
            &\stackrel{(3)}{=}
            \resizebox{0.78\textwidth}{!}{$\displaystyle\frac{\tanh^{-1}\left(
            \sqrt{|K|}
            \left\|
            \frac{
            \sqrt{1+K\|p_i\|^2}x
            -\left(1+\frac{K\langle p_i,x\rangle}
            {1+\sqrt{1+K\|p_i\|^2}}\right)p_i}
            {1+K\langle p_i,x\rangle}
            \right\|
            \right)}
            {\sqrt{|K|}
            \left\|
            \frac{
            \sqrt{1+K\|p_i\|^2}x
            -\left(1+\frac{K\langle p_i,x\rangle}
            {1+\sqrt{1+K\|p_i\|^2}}\right)p_i}
            {1+K\langle p_i,x\rangle}
            \right\|}
            \times
            \left\langle
            \frac{
            \sqrt{1+K\|p_i\|^2}x
            -\left(1+\frac{K\langle p_i,x\rangle}
            {1+\sqrt{1+K\|p_i\|^2}}\right)p_i}
            {1+K\langle p_i,x\rangle},z_i
            \right\rangle$}\\
            &\stackrel{(4)}{=}
            \resizebox{0.78\textwidth}{!}{$\displaystyle\frac{\tanh^{-1}\left(
            \frac{\sqrt{|K|\left[\|x-p_i\|^2
            +K\left(\|p_i\|^2\|x\|^2-\langle p_i,x\rangle^2\right)\right]}}
            {1+K\langle p_i,x\rangle}
            \right)}
            {\frac{\sqrt{|K|\left[\|x-p_i\|^2
            +K\left(\|p_i\|^2\|x\|^2-\langle p_i,x\rangle^2\right)\right]}}
            {1+K\langle p_i,x\rangle}}
            \times
            \frac{1}{1+K\langle p_i,x\rangle}
            \left[
            \sqrt{1+K\|p_i\|^2}\langle x,z_i\rangle
            -\left(1+\frac{K\langle p_i,x\rangle}
            {1+\sqrt{1+K\|p_i\|^2}}\right)\langle p_i,z_i\rangle
            \right]$}\\
            &\stackrel{(5)}{=}
            \frac{\tanh^{-1}\left(
            \frac{\sqrt{|K|\left[\|x-p_i\|^2
            +K\left(\|p_i\|^2\|x\|^2-\langle p_i,x\rangle^2\right)\right]}}
            {1+K\langle p_i,x\rangle}
            \right)}
            {\frac{\sqrt{|K|\left[\|x-p_i\|^2
            +K\left(\|p_i\|^2\|x\|^2-\langle p_i,x\rangle^2\right)\right]}}
            {1+K\langle p_i,x\rangle}}
            \frac{\langle x-p_i,z_i\rangle}
            {1+K\langle p_i,x\rangle}\\
            &\stackrel{(6)}{=}
            \frac{\tanh^{-1}\left(
            \frac{\sqrt{|K|\left[\|x-p_i\|^2
            +K\left(\|p_i\|^2\|x\|^2-\langle p_i,x\rangle^2\right)\right]}}
            {1+K\langle p_i,x\rangle}
            \right)}
            {\sqrt{|K|\left[\|x-p_i\|^2
            +K\left(\|p_i\|^2\|x\|^2-\langle p_i,x\rangle^2\right)\right]}}
            \langle x-p_i,z_i\rangle\\
            &\stackrel{(7)}{=}
            \|z_i\|
            \frac{\tanh^{-1}\left(\sqrt{|K|}Q_i^{\mathrm{K}}\right)}
            {\sqrt{|K|}Q_i^{\mathrm{K}}}
            \frac{
            \langle x,[z_i]\rangle
            -\frac{\tanh(\sqrt{|K|}\gamma_i)}{\sqrt{|K|}}}
            {D_i^{\mathrm{K}}}.
        \end{aligned}
    \end{equation*}

    The above comes from the following.
    \begin{enumerate}[label=(\arabic*)]
        \item The first equality is the HFC-K expression in \cref{app:thm:pball_fc}.
        \item Applying the logarithmic map at $\bbzero$ in \cref{app:eq:log_0_klein} to $(-p_i)\Eoplus x$ gives the second equality.
        \item The gamma factor of $-p_i$ is
        \begin{equation*}
            \gamma_{-p_i}=\frac{1}{\sqrt{1+K\|p_i\|^2}}.
        \end{equation*}
        Substituting $-p_i$ and $x$ into the Einstein addition in \cref{app:subsec:geom_hyperbolic} and replacing $\gamma_{-p_i}$ gives
        \begin{equation*}
            (-p_i)\Eoplus x=
            \frac{
            \sqrt{1+K\|p_i\|^2}x
            -\left(1+\frac{K\langle p_i,x\rangle}
            {1+\sqrt{1+K\|p_i\|^2}}\right)p_i}
            {1+K\langle p_i,x\rangle}.
        \end{equation*}
        Inserting this vector into both occurrences of $(-p_i)\Eoplus x$ gives the third equality.
        \item Expanding the squared norm of the numerator in (3) gives
        \begin{equation*}
            \begin{aligned}
                &\left\|
                \sqrt{1+K\|p_i\|^2}x
                -\left(1+\frac{K\langle p_i,x\rangle}
                {1+\sqrt{1+K\|p_i\|^2}}\right)p_i
                \right\|^2\\
                &=(1+K\|p_i\|^2)\|x\|^2
                +\left(1+\frac{K\langle p_i,x\rangle}
                {1+\sqrt{1+K\|p_i\|^2}}\right)^2\|p_i\|^2\\
                &\quad-2\sqrt{1+K\|p_i\|^2}
                \left(1+\frac{K\langle p_i,x\rangle}
                {1+\sqrt{1+K\|p_i\|^2}}\right)\langle p_i,x\rangle\\
                &=\|x-p_i\|^2
                +K\left(\|p_i\|^2\|x\|^2-\langle p_i,x\rangle^2\right).
            \end{aligned}
        \end{equation*}
        Therefore,
        \begin{equation*}
            \|(-p_i)\Eoplus x\|^2=
            \frac{\|x-p_i\|^2
            +K\left(\|p_i\|^2\|x\|^2-\langle p_i,x\rangle^2\right)}
            {(1+K\langle p_i,x\rangle)^2}.
        \end{equation*}
        Expanding the inner product of the vector in (3) with $z_i$ gives
        \begin{equation*}
            \begin{aligned}
                \left\langle(-p_i)\Eoplus x,z_i\right\rangle
                &=\frac{1}{1+K\langle p_i,x\rangle}
                \left[
                \sqrt{1+K\|p_i\|^2}\langle x,z_i\rangle\right.\\
                &\quad\left.-\left(1+\frac{K\langle p_i,x\rangle}
                {1+\sqrt{1+K\|p_i\|^2}}\right)\langle p_i,z_i\rangle
                \right].
            \end{aligned}
        \end{equation*}
        Substituting these two identities into (3) gives the fourth equality.
        \item Since $p_i\parallel z_i$,
        \begin{equation*}
            \langle p_i,x\rangle\langle p_i,z_i\rangle
            =\|p_i\|^2\langle x,z_i\rangle.
        \end{equation*}
        Using this identity, the directional term in (4) becomes
        \begin{equation*}
            \begin{aligned}
                &\sqrt{1+K\|p_i\|^2}\langle x,z_i\rangle
                -\left(1+\frac{K\langle p_i,x\rangle}
                {1+\sqrt{1+K\|p_i\|^2}}\right)\langle p_i,z_i\rangle\\
                &=\left(
                \sqrt{1+K\|p_i\|^2}
                -\frac{K\|p_i\|^2}{1+\sqrt{1+K\|p_i\|^2}}
                \right)\langle x,z_i\rangle-\langle p_i,z_i\rangle\\
                &=\langle x,z_i\rangle-\langle p_i,z_i\rangle
                =\langle x-p_i,z_i\rangle.
            \end{aligned}
        \end{equation*}
        Here, the second equality uses
        \begin{equation*}
            \sqrt{1+K\|p_i\|^2}
            -\frac{K\|p_i\|^2}{1+\sqrt{1+K\|p_i\|^2}}=1.
        \end{equation*}
        This gives the fifth equality.
        \item The remaining scalar factors satisfy
        \begin{equation*}
            \begin{aligned}
                \frac{1}
                {\frac{\sqrt{|K|\left[\|x-p_i\|^2
                +K\left(\|p_i\|^2\|x\|^2-\langle p_i,x\rangle^2\right)\right]}}
                {1+K\langle p_i,x\rangle}}
                \frac{1}{1+K\langle p_i,x\rangle}
                =\frac{1}
                {\sqrt{|K|\left[\|x-p_i\|^2
                +K\left(\|p_i\|^2\|x\|^2-\langle p_i,x\rangle^2\right)\right]}}.
            \end{aligned}
        \end{equation*}
        This gives the sixth equality.
        \item The exponential map at the origin gives
        \begin{equation*}
            p_i=\frac{\tanh(\sqrt{|K|}\gamma_i)}{\sqrt{|K|}}[z_i]
        \end{equation*}
        Define
        \begin{equation*}
            \begin{aligned}
                D_i^{\mathrm{K}}
                &=1-\sqrt{|K|}\tanh(\sqrt{|K|}\gamma_i)\langle x,[z_i]\rangle,\\
                Q_i^{\mathrm{K}}
                &=\sqrt{\frac{
                \|x\|^2+\frac{\tanh^2(\sqrt{|K|}\gamma_i)}{|K|}
                -\frac{2\tanh(\sqrt{|K|}\gamma_i)}{\sqrt{|K|}}\langle x,[z_i]\rangle
                -\tanh^2(\sqrt{|K|}\gamma_i)
                \left(\|x\|^2-\langle x,[z_i]\rangle^2\right)}
                {(D_i^{\mathrm{K}})^2}}.
            \end{aligned}
        \end{equation*}
        When $Q_i^{\mathrm{K}}=0$, the quotient $\tanh^{-1}(\sqrt{|K|}Q_i^{\mathrm{K}})/(\sqrt{|K|}Q_i^{\mathrm{K}})$ is defined as $1$ by continuity. Substituting $p_i$ into (6) gives
        \begin{equation*}
            \begin{aligned}
                1+K\langle p_i,x\rangle&=D_i^{\mathrm{K}},\\
                \langle x-p_i,z_i\rangle
                &=\|z_i\|\left(
                \langle x,[z_i]\rangle
                -\frac{\tanh(\sqrt{|K|}\gamma_i)}{\sqrt{|K|}}
                \right),\\
                \|(-p_i)\Eoplus x\|&=Q_i^{\mathrm{K}}.
            \end{aligned}
        \end{equation*}
        These identities give the seventh equality.
    \end{enumerate}

    \textbf{Hyperboloid model.}
    For $x=(x_1,x_s)$ and $p_i=(p_{i1},p_{is})$, we have
    \begin{equation*}
        \begin{aligned}
            v_i(x)&\stackrel{(1)}{=}
            \left\langle\rielog_{p_i}(x),\pt{e}{p_i}(0,z_i)\right\rangle_{\calL}\\
            &\stackrel{(2)}{=}
            \frac{\cosh^{-1}\left(K\langle p_i,x\rangle_{\calL}\right)}
            {\sqrt{\left(K\langle p_i,x\rangle_{\calL}\right)^2-1}}
            \left\langle
            x-K\langle p_i,x\rangle_{\calL}p_i,
            \pt{e}{p_i}(0,z_i)
            \right\rangle_{\calL}\\
            &\stackrel{(3)}{=}
            \frac{\cosh^{-1}\left(K\langle p_i,x\rangle_{\calL}\right)}
            {\sqrt{\left(K\langle p_i,x\rangle_{\calL}\right)^2-1}}
            \left\langle x,\pt{e}{p_i}(0,z_i)\right\rangle_{\calL}\\
            &\stackrel{(4)}{=}
            \frac{\cosh^{-1}\left(K\langle p_i,x\rangle_{\calL}\right)}
            {\sqrt{\left(K\langle p_i,x\rangle_{\calL}\right)^2-1}}
            \left\langle
            x,
            \left(
            \sqrt{|K|}\langle p_{is},z_i\rangle,
            z_i+\frac{|K|\langle p_{is},z_i\rangle}
            {1+\sqrt{|K|}p_{i1}}p_{is}
            \right)
            \right\rangle_{\calL}\\
            &\stackrel{(5)}{=}
            \frac{\cosh^{-1}\left(K\langle p_i,x\rangle_{\calL}\right)}
            {\sqrt{\left(K\langle p_i,x\rangle_{\calL}\right)^2-1}}
            \left[
            \langle x_s,z_i\rangle
            +\frac{|K|\langle p_{is},z_i\rangle\langle x_s,p_{is}\rangle}
            {1+\sqrt{|K|}p_{i1}}
            -\sqrt{|K|}x_1\langle p_{is},z_i\rangle
            \right]\\
            &\stackrel{(6)}{=}
            \frac{\cosh^{-1}\left(K\langle p_i,x\rangle_{\calL}\right)}
            {\sqrt{\left(K\langle p_i,x\rangle_{\calL}\right)^2-1}}
            \left[
            \left(1+\frac{|K|\|p_{is}\|^2}
            {1+\sqrt{|K|}p_{i1}}\right)\langle x_s,z_i\rangle
            -\sqrt{|K|}x_1\langle p_{is},z_i\rangle
            \right]\\
            &\stackrel{(7)}{=}
            \sqrt{|K|}
            \frac{\cosh^{-1}\left(K\langle p_i,x\rangle_{\calL}\right)}
            {\sqrt{\left(K\langle p_i,x\rangle_{\calL}\right)^2-1}}
            \left(p_{i1}\langle x_s,z_i\rangle-x_1\langle p_{is},z_i\rangle\right)\\
            &\stackrel{(8)}{=}\resizebox{0.89\textwidth}{!}{$\displaystyle
            \|z_i\|
            \frac{
            \cosh^{-1}\left(
            \sqrt{|K|}\left[
            \cosh(\sqrt{|K|}\gamma_i)x_1
            -\sinh(\sqrt{|K|}\gamma_i)\langle x_s,[z_i]\rangle
            \right]
            \right)}
            {\sqrt{
            |K|\left[
            \cosh(\sqrt{|K|}\gamma_i)x_1
            -\sinh(\sqrt{|K|}\gamma_i)\langle x_s,[z_i]\rangle
            \right]^2-1}}
            \left[
            \cosh(\sqrt{|K|}\gamma_i)\langle x_s,[z_i]\rangle
            -\sinh(\sqrt{|K|}\gamma_i)x_1
            \right]$}.
        \end{aligned}
    \end{equation*}

    The above comes from the following.
    \begin{enumerate}[label=(\arabic*)]
        \item The first equality is the HFC-H expression in \cref{app:thm:hyperboloid_fc}.
        \item From \cref{app:tab:riem_operators_hyperbolic},
        \begin{equation*}
            \rielog_{p_i}(x)=
            \frac{\cosh^{-1}\left(K\langle p_i,x\rangle_{\calL}\right)}
            {\sinh\left(\cosh^{-1}\left(K\langle p_i,x\rangle_{\calL}\right)\right)}
            \left(x-K\langle p_i,x\rangle_{\calL}p_i\right).
        \end{equation*}
        Since $\sinh(\cosh^{-1}(t))=\sqrt{t^2-1}$ for $t\geq1$, substituting this logarithmic map into (1) gives the second equality.
        \item Parallel transport maps $(0,z_i)\in T_e\bbh{n}$ to $T_{p_i}\bbh{n}$. Therefore,
        \begin{equation*}
            \left\langle p_i,\pt{e}{p_i}(0,z_i)\right\rangle_{\calL}=0.
        \end{equation*}
        Consequently,
        \begin{equation*}
            \begin{aligned}
                &\left\langle
                x-K\langle p_i,x\rangle_{\calL}p_i,
                \pt{e}{p_i}(0,z_i)
                \right\rangle_{\calL}\\
                &=\left\langle x,\pt{e}{p_i}(0,z_i)\right\rangle_{\calL},
            \end{aligned}
        \end{equation*}
        which gives the third equality.
        \item The parallel transport in \cref{app:tab:riem_operators_hyperbolic} gives
        \begin{equation*}
            \pt{e}{p_i}(0,z_i)
            =(0,z_i)-
            \frac{K\langle p_i,(0,z_i)\rangle_{\calL}}
            {1+K\langle e,p_i\rangle_{\calL}}(e+p_i).
        \end{equation*}
        Using $e=(1/\sqrt{|K|},0)$, we have
        \begin{equation*}
            \langle p_i,(0,z_i)\rangle_{\calL}=\langle p_{is},z_i\rangle,
            \qquad
            1+K\langle e,p_i\rangle_{\calL}=1+\sqrt{|K|}p_{i1}.
        \end{equation*}
        Hence,
        \begin{equation*}
            \pt{e}{p_i}(0,z_i)=
            \left(
            \sqrt{|K|}\langle p_{is},z_i\rangle,
            z_i+\frac{|K|\langle p_{is},z_i\rangle}
            {1+\sqrt{|K|}p_{i1}}p_{is}
            \right),
        \end{equation*}
        which gives the fourth equality.
        \item Expanding the Lorentz inner product in (4) gives
        \begin{equation*}
            \begin{aligned}
                &\left\langle
                x,
                \left(
                \sqrt{|K|}\langle p_{is},z_i\rangle,
                z_i+\frac{|K|\langle p_{is},z_i\rangle}
                {1+\sqrt{|K|}p_{i1}}p_{is}
                \right)
                \right\rangle_{\calL}\\
                &=\langle x_s,z_i\rangle
                +\frac{|K|\langle p_{is},z_i\rangle\langle x_s,p_{is}\rangle}
                {1+\sqrt{|K|}p_{i1}}
                -\sqrt{|K|}x_1\langle p_{is},z_i\rangle.
            \end{aligned}
        \end{equation*}
        This gives the fifth equality.
        \item Since $p_i=\rieexp_e(\gamma_i[(0,z_i^\top)^\top])$, the spatial component satisfies $p_{is}\parallel z_i$. Thus,
        \begin{equation*}
            \langle p_{is},z_i\rangle\langle x_s,p_{is}\rangle
            =\|p_{is}\|^2\langle x_s,z_i\rangle,
        \end{equation*}
        which gives the sixth equality.
        \item Since $p_i\in\bbh{n}$,
        \begin{equation*}
            \|p_{is}\|^2-p_{i1}^2=\frac{1}{K}=-\frac{1}{|K|}.
        \end{equation*}
        Therefore,
        \begin{equation*}
            1+\frac{|K|\|p_{is}\|^2}{1+\sqrt{|K|}p_{i1}}
            =1+\frac{|K|p_{i1}^2-1}{1+\sqrt{|K|}p_{i1}}
            =\sqrt{|K|}p_{i1}.
        \end{equation*}
        Substituting this identity into (6) gives the seventh equality.
        \item The exponential map at $e$ gives
        \begin{equation*}
            p_i=\left(
            \frac{\cosh(\sqrt{|K|}\gamma_i)}{\sqrt{|K|}},
            \frac{\sinh(\sqrt{|K|}\gamma_i)}{\sqrt{|K|}}[z_i]
            \right).
        \end{equation*}
        Hence,
        \begin{equation*}
            \begin{aligned}
                K\langle p_i,x\rangle_{\calL}
                &=\sqrt{|K|}\left[
                \cosh(\sqrt{|K|}\gamma_i)x_1
                -\sinh(\sqrt{|K|}\gamma_i)\langle x_s,[z_i]\rangle
                \right],\\
                \sqrt{|K|}\left(
                p_{i1}\langle x_s,z_i\rangle
                -x_1\langle p_{is},z_i\rangle
                \right)
                &=\|z_i\|\left[
                \cosh(\sqrt{|K|}\gamma_i)\langle x_s,[z_i]\rangle
                -\sinh(\sqrt{|K|}\gamma_i)x_1
                \right].
            \end{aligned}
        \end{equation*}
        Substituting these two identities into (7) gives the eighth equality.
    \end{enumerate}
    When $x=p_i$, we have $K\langle p_i,x\rangle_{\calL}=1$. The quotient $\cosh^{-1}(K\langle p_i,x\rangle_{\calL})/\sqrt{(K\langle p_i,x\rangle_{\calL})^2-1}$ is defined as $1$ by continuity. The remaining factor in (7) vanishes, and hence $v_i(p_i)=0$.
\end{proof}

\linkofproof{thm:spd_fc}
\subsection{Proof of \cref{thm:spd_fc}}
\label{app:subsec:prf_spd_fc}
\begin{proof}
    In the following proof, we first present the expressions of several operators under different metrics, including $v_{ij}(S)$, standard orthonormal bases, and Riemannian exponential maps at the origin. Then, we prove the theorem. In this proof, we follow the notation of the theorem.

    \textbf{$v_{ij}(S)$ under different metrics:}
    The expressions are implied by \citet[Thm. 4.2]{chen2024rmlr}:
    \begin{align}
        \label{app:eq:v_ij_lem}
        \text{LEM}:
        &  \left\langle \log(S)-\log(P_{ij}), Z_{ij} \right\rangle^{\alphabeta}, \\
        \label{app:eq:v_ij_aim}
        \text{AIM}:
        &\left\langle \log(P_{ij}^{-\frac{1}{2}} S P_{ij}^{-\frac{1}{2}}), Z_{ij} \right\rangle^{\alphabeta}, \\
        \text{PEM}:
         &\frac{1}{\theta} \left\langle S^\theta-P_{ij}^\theta, Z_{ij} \right\rangle^{\alphabeta}, \\
        \label{app:eq:v_ij_lcm}
        \text{LCM}:
        & \left\langle \lfloor K\rfloor - \lfloor L_{ij} \rfloor + \dlog(\bbK\bbL_{ij}^{-1}), \lfloor Z_{ij} \rfloor + \frac{1}{2}\bbZ_{ij} \right\rangle, \\
        \text{BWM}:
        & \frac{1}{2} \left\langle (P_{ij}S)^{\frac{1}{2}} + (SP_{ij})^{\frac{1}{2}} -2P_{ij}, \calL_{P_{ij}}(L_{ij} Z_{ij} L_{ij}^\top) \right\rangle.
    \end{align}

    \textbf{Standard orthonormal bases: }
    Next, we show the standard orthonormal bases over $T_I\spd{n}$ under different metrics. As indicated by \cref{app:tab:riem_operators_lem_aim_pem,app:tab:riem_operators_bwm_lc}, the inner products for any $V, W \in T_I\spd{n}$ are
    \begin{align}
        \label{app:eq:inner_i_pem_lem_aim}
        \text{LEM, AIM, and PEM}:& \alphabetainner{V}{W},\\
        \label{app:eq:inner_i_lcm}
        \text{LCM}:&
         \langle \lfloor V \rfloor +\frac{1}{2}\bbV, \lfloor W \rfloor +\frac{1}{2}\bbW \rangle, \\
         \label{app:eq:inner_i_bwm}
        \text{BWM}:&
         \frac{1}{4} \langle V,W\rangle
    \end{align}
    The above comes from the following.
    \begin{enumerate}[label=(\arabic*)]
        \item
        \cref{app:eq:inner_i_pem_lem_aim} comes from $\log_{*,I}(V)=V$ and $\pow_{\theta*,I}(V)=\theta V$;
        \item
        \cref{app:eq:inner_i_lcm} comes from $\chol_{*,I}(V) = \lfloor V \rfloor + \frac{1}{2}\bbV$;
        \item
        \cref{app:eq:inner_i_bwm} comes from $\calL_{I}[V] = \frac{1}{2}V$.
    \end{enumerate}

    As shown by \citet[Thm.2.1]{thanwerdas2023n}, $F_{\sqrt{\alpha+n\beta}, \sqrt{\alpha}}: \{\sym{n}, \alphabetainner{\cdot}{\cdot}\} \rightarrow \{\sym{n},\inner{\cdot}{\cdot}\}$ is the linear isometry pulling the standard inner product back to the $\orth{n}$-invariant one:
    \begin{equation} \label{app:eq:f_pq}
        F_{\sqrt{\alpha+n\beta}, \sqrt{\alpha}}(X)=\sqrt{\alpha} X+\frac{ \sqrt{\alpha+ n \beta}-\sqrt{\alpha}}{n} \tr(X) I_n, \forall X \in \sym{n}.
    \end{equation}
    Given any $Y \in \sym{n}$, its inverse map is
    \begin{equation}
        \label{app:eq:inv_f_p_q}
        \begin{aligned}
            \left(F_{\sqrt{\alpha+n\beta}, \sqrt{\alpha}}\right)^{-1} (Y)
            &= \frac{1}{\sqrt{\alpha}} \left\{ Y - \left( \frac{\sqrt{1+n\frac{\beta}{\alpha}}-1}{n} \frac{1}{\sqrt{1+n\frac{\beta}{\alpha}}} \right) \tr(Y) I \right\} \\
            &= \frac{1}{\sqrt{\alpha}} \left\{ Y - \frac{1}{n} \left( 1 - \frac{1}{\sqrt{1+n\frac{\beta}{\alpha}}} \right) \tr(Y) I \right\} \\
            &= \frac{1}{\sqrt{\alpha}}Y - \frac{1}{n} \left( \frac{1}{\sqrt{\alpha}} - \frac{1}{\sqrt{\alpha+n\beta}} \right)\tr(Y)I.
        \end{aligned}
    \end{equation}

    The standard orthonormal bases over the Euclidean spaces $\{\sym{n},\inner{\cdot}{\cdot}\}$ and $\{\tril{n},\inner{\cdot}{\cdot}\}$ are
    \begin{align}
        \{\sym{n},\inner{\cdot}{\cdot}\}&: U^{\mathrm{sym}}_{ij} =
        \begin{cases}
        E_{ii}, & \text{if } i = j, \\
        \frac{E_{ij}+ E_{ji}}{\sqrt{2}}, & \text{if } i > j.
        \end{cases}\\
        \{\tril{n},\inner{\cdot}{\cdot}\}&: U^{\mathrm{tril}}_{ij} = E_{ij}, \forall i \geq j
    \end{align}
    where $i \geq j, i,j = 1, \cdots, n$, and $\{E_{ij}\}^n_{i,j=1}$ are standard basis matrices, with the $(k,l)$ element defined by
    \begin{equation}
        \left(E_{i j}\right)_{k l}=
        \begin{cases}
        1 & \text { if } k=i \text { and } l=j, \\
        0 & \text { otherwise. }
        \end{cases}
    \end{equation}

    The standard orthonormal bases with respect to \cref{app:eq:inner_i_pem_lem_aim,app:eq:inner_i_lcm,app:eq:inner_i_bwm} are
    \begin{align}
        \text{LEM, AIM, PEM}:&
         U^{\alphabeta}_{ij} \stackrel{(1)}{=}
        \begin{cases}
            \frac{1}{\sqrt{\alpha}}E_{ii}- \frac{1}{n} \left( \frac{1}{\sqrt{\alpha}} - \frac{1}{\sqrt{\alpha+n\beta}} \right)I, & \text{if } i = j, \\
            \frac{E_{ij}+ E_{ji}}{\sqrt{2\alpha}}, & \text{if } i > j.
        \end{cases}\\
        \text{LCM}:&
        U^{\mathrm{LC}}_{ij}  \stackrel{(2)}{=}
        \begin{cases}
            2E_{ii}, & \text{if } i = j, \\
            E_{ij}, & \text{if } i > j.
        \end{cases}\\
        \text{BWM} :&
        U^{\mathrm{BW}}_{ij}  \stackrel{(3)}{=}
        \begin{cases}
            2E_{ii}, & \text{if } i = j, \\
            \sqrt{2}(E_{ij}+ E_{ji}), & \text{if } i > j.
        \end{cases}
    \end{align}
    Here, $i \geq j, i,j = 1, \cdots, n$.
    The above comes from the following.
    \begin{enumerate}[label=(\arabic*)]
        \item
        $U^{\alphabeta}_{ij}=\left(F_{\sqrt{\alpha+n\beta}, \sqrt{\alpha}}\right)^{-1} \left(U^{\mathrm{sym}}_{ij} \right)$, with $F_{\sqrt{\alpha+n\beta}, \sqrt{\alpha}}: \sym{n} \rightarrow \sym{n}$ as the linear isometry pulling back the Frobenius inner product to the $\orth{n}$-invariant inner product;
        \item
        $f^{\mathrm{LC}}(V)=\lfloor V \rfloor +\frac{1}{2}\bbV: \tril{n} \rightarrow \tril{n}$ is the linear isometry pulling the Frobenius inner product to \cref{app:eq:inner_i_lcm};
        \item
        $f^{\mathrm{BW}}(V)=\frac{1}{2}V: \sym{n} \rightarrow \sym{n}$ is the linear isometry pulling the Frobenius inner product back to \cref{app:eq:inner_i_bwm};
    \end{enumerate}

    \textbf{Riemannian exponential maps:}
    Next, we show $\rieexp_I$ under different metrics
    \begin{align}
        \label{app:eq:exp_i_lem_aim}
        \text{LEM and AIM}:
        & \rieexp_I(V)\stackrel{(1)}{=}\exp(V),\\
        \label{app:eq:exp_i_pem}
        \text{PEM}:
        & \rieexp_I(V)\stackrel{(2)}{=}\left(I + \theta V \right)^{\frac{1}{\theta}},\\
        \label{app:eq:exp_i_lcm}
        \text{LCM}:
        & \rieexp_I(V)\stackrel{(3)}{=} \left(\lfloor V \rfloor + \dexp \left(\frac{1}{2}\bbV \right) \right)\left(\lfloor V \rfloor + \dexp \left(\frac{1}{2}\bbV \right) \right)^{\top},\\
        \text{BWM}:
        & \rieexp_I(V)\stackrel{(4)}{=} I+V+\frac{1}{4}V^2 = \left(I+\frac{1}{2}V \right)^2,
    \end{align}
    The above comes from the following.
    \begin{enumerate}[label=(\arabic*)]
        \item
        $\log_{*,I}(V)=V$ and $\log{I}=\bbzero$;
        \item
        $\pow_{\theta*,I}(V)=\theta V$;
        \item
        $\chol_{*,I}(V) = \lfloor V \rfloor + \frac{1}{2}\bbV$;
        \item
        $\calL_{I}[V]=\frac{1}{2}V$.
    \end{enumerate}

    Now, we can prove the results metric by metric.

    \textbf{LEM: }
    \begin{equation}
        \begin{aligned}
             &\rieexp_{I} \left( \sum_{i,j=1,i \geq j}^{m}  v_{ij}^{\LE}(S) U^{\alphabeta}_{ij} \right) \\
             &=\exp \left( \sum_{i,j=1,i \geq j}^{m} \left( \log(S)-\log(P_{ij}), Z_{ij} \rangle^{\alphabeta} U^{\alphabeta}_{ij} \right) \right).
        \end{aligned}
    \end{equation}

    \textbf{AIM: }
    \begin{equation}
        \begin{aligned}
             &\rieexp_{I} \left( \sum_{i,j=1,i \geq j}^{m}  v_{ij}^{\AI}(S) U^{\alphabeta}_{ij} \right) \\
             &=\exp \left( \sum_{i,j=1,i \geq j}^{m} \left( \langle \log(P_{ij}^{-\frac{1}{2}} S P_{ij}^{-\frac{1}{2}}), Z_{ij} \rangle^{\alphabeta} U^{\alphabeta}_{ij} \right) \right).
        \end{aligned}
    \end{equation}

    \textbf{PEM: }
    \begin{equation}
        \begin{aligned}
             &\rieexp_{I} \left( \sum_{i,j=1,i \geq j}^{m}  v_{ij}^{\PE}(S) U^{\alphabeta}_{ij} \right) \\
             &= \left(I + \theta\sum_{i,j=1,i \geq j}^{m} \left( \frac{1}{\theta} \langle S^\theta-P_{ij}^\theta, Z_{ij} \rangle^{\alphabeta} U^{\alphabeta}_{ij} \right) \right)^{\frac{1}{\theta}}\\
             &= \left(I + \sum_{i,j=1,i \geq j}^{m} \left( \langle S^\theta-P_{ij}^\theta, Z_{ij} \rangle^{\alphabeta} U^{\alphabeta}_{ij} \right) \right)^{\frac{1}{\theta}}.
        \end{aligned}
    \end{equation}

    \textbf{LCM: }
    \begin{equation}
        \label{app:eq:lcm_prf_fc}
        \begin{aligned}
             &\rieexp_{I} \left( \sum_{i,j=1,i \geq j}^{m}  v_{ij}^{\LC}(S) U^{\LC}_{ij} \right) \\
             &= \left(\lfloor V^\LC \rfloor + \dexp \left(\frac{1}{2}\bbV^\LC \right) \right)\left(\lfloor V^\LC \rfloor + \dexp \left(\frac{1}{2}\bbV^\LC \right) \right)^\top,
        \end{aligned}
    \end{equation}
    with
    \begin{equation}
    \label{app:eq:lcm_prf_v_lc}
        \begin{aligned}
            V^\LC
            &= \sum_{i,j=1,i \geq j}^{m}  v_{ij}^{\LC}(S) U^{\LC}_{ij}\\
            &= \sum_{i,j=1,i \geq j}^{m}  \left( \left\langle \lfloor K\rfloor - \lfloor L_{ij} \rfloor + \dlog(\bbK\bbL_{ij}^{-1}), \lfloor Z_{ij} \rfloor + \frac{1}{2}\bbZ_{ij} \right\rangle \right) U^{\LC}_{ij}\\
        \end{aligned}
    \end{equation}

    \textbf{BWM: }
   \begin{equation}
        \begin{aligned}
             &\rieexp_{I} \left( \sum_{i,j=1,i \geq j}^{m}  v_{ij}^{\BW}(S) U^{\BW}_{ij} \right) \\
            &= \left(I+\frac{1}{2}V^\BW \right)^2,\\
        \end{aligned}
    \end{equation}
    with $V^\BW$ defined as
    \begin{equation}
        V^\BW= \sum_{i,j=1,i \geq j}^{m} \left\{\frac{1}{2} \left\langle \left(P_{ij}S \right)^{\frac{1}{2}} + \left(S P_{ij} \right)^{\frac{1}{2}} -2 P_{ij}, \calL_{P_{ij}}(L_{ij} Z_{ij} L_{ij}^\top) \right\rangle U^{\mathrm{BW}}_{ij} \right\}.
    \end{equation}
\end{proof}

\linkofproof{prop:spd_fc_lie_hom}
\subsection{Proof of \cref{prop:spd_fc_lie_hom}}
\label{app:subsec:prf:spd_fc_lie_hom}
We begin by recalling two vector structures on the SPD manifold. Next, we identify the expression for the linear homomorphisms. Finally, we present our proof.

We define a map $\phi(\cdot): \spd{n} \rightarrow \tril{n}$ as
\begin{equation}
    \phi(S) = \lfloor L \rfloor + \dlog(\bbL),
\end{equation}
where $P=LL^\top$ is the Cholesky decomposition. For any $P,Q \in \spd{n}$ and $t \in \bbRscalar$, the vector structures over the SPD manifold are defined as
\begin{align}
    P \oplusle Q &= \exp( \log(P) + \log(Q))\\
    t \odotle P &= \exp( t\log(P) ) = P^t\\
    P \opluslc Q &= \phi^{-1}( \phi(P) + \phi(Q))\\
    t \odotlc P &= \phi^{-1}( t\phi(P) ) = P^t
\end{align}
As shown by \citet{arsigny2005fast,chen2024adaptive}, $\{ \spd{n}, \oplusle, \odotle \}$ and $\{ \spd{n}, \opluslc, \odotlc\}$ forms vector spaces. We further present the associated linear homomorphisms.

\begin{lemma}[SPD Homomorphisms]
    Given any homomorphisms
    \begin{align}
        \zeta^{\LE}(\cdot): \{ \spd{n}, \oplusle, \odotle \} \rightarrow \{ \spd{m}, \oplusle, \odotle \}, \\
        \zeta^{\LC}(\cdot): \{ \spd{n}, \opluslc, \odotlc \} \rightarrow \{ \spd{m}, \opluslc, \odotlc \},
    \end{align}
    they can be expressed as
    \begin{align}
        \zeta^{\LE} &= \exp \circ g \circ \log,\\
        \zeta^{\LC} &= \phi^{-1} \circ f \circ \phi,
    \end{align}
    where $f: \tril{n} \rightarrow \tril{m}$ and $g: \sym{n} \rightarrow \sym{m}$ are linear homomorphisms over the Euclidean space $\tril{n}$ and $\sym{n}$, respectively.
\end{lemma}
\begin{proof}
    As shown by \citet{chen2024adaptive}, $\log(\cdot)$ is the linear isomorphism from $\{ \spd{n}, \oplusle, \odotle \}$ to the Euclidean space $\sym{n}$ and $\phi$ is the linear isomorphism from $\{ \spd{n}, \opluslc, \odotlc\}$ to the Euclidean space $\tril{n}$. Therefore, any linear homomorphisms over these two linear spaces have the following forms:
    \begin{align}
        \label{app:eq:linear_hom_lem}
        \zeta^{\LE} &= \log^{-1} f \circ \log,\\
        \label{app:eq:linear_hom_lcm}
        \zeta^{\LC} &= \phi^{-1} g \circ \phi,
    \end{align}
    where $f: \sym{n} \rightarrow \sym{m}$ and $g: \tril{n} \rightarrow \tril{m} $ are linear homomorphisms over the Euclidean space $\sym{n}$ and $\tril{n}$, respectively.
\end{proof}

With all the above theoretical preparation, we begin to present our proof.
\begin{proof}
    Given an SPD matrix $S \in \spd{n}$, \cref{app:eq:linear_hom_lem} can be rewritten as
    \begin{equation}
        \begin{aligned}
            \zeta^{\LE}(S)
            &\stackrel{(1)}{=} \exp\left( \sum_{i,j=1,i \geq j}^{m} \left\langle \log(S), A_{ij} \right\rangle U^{\mathrm{sym}}_{ij} \right)\\
            &\stackrel{(2)}{=} \exp\left( \sum_{i,j=1,i \geq j}^{m} \left\langle \log(S), A_{ij} \right\rangle U^{(1,0)}_{ij} \right)\\
            &\stackrel{(3)}{=} \calF^{\LE}(S;\bfA,\mathbf{I})\\
        \end{aligned}
    \end{equation}
    where $\bfA=\{A_{ij} \in \sym{n} \}_{i,j=1,i \geq j}^{m}$ and $\mathbf{I}=\{I, \cdots, I\}$.
    The above comes from the following.
    \begin{enumerate}[label=(\arabic*)]
        \item
        The linear map $f$ can be represented by $\{A_{ij} \in \sym{n} \}_{i,j=1,i \geq j}^{m}$ under the bases $\{U^{\mathrm{sym}}_{ij}\}_{i,j=1,i \geq j}^{n}$ over $\sym{n}$ and $\{U^{\mathrm{sym}}_{ij}\}_{i,j=1,i \geq j}^{m}$ over $\sym{m}$;
        \item
        $\{U^{\mathrm{sym}}_{ij}\}_{i,j=1,i \geq j}^{m} = \{U^{(1,0)}_{ij}\}_{i,j=1,i \geq j}^{m}$;
        \item
        $\rieexp_I=\exp$ under LEM.
    \end{enumerate}

    Following the above logic, we have the following for $\{\spd{n}, \opluslc, \odotlc\}$:
    \begin{equation}
        \begin{aligned}
            \zeta^{\LC}(S)
            &\stackrel{(1)}{=} \phi^{-1} \left( \sum_{i,j=1,i \geq j}^{m} \left\langle \phi(S), A_{ij} \right\rangle U^{\mathrm{tril}}_{ij} \right)\\
            &\stackrel{(2)}{=} \calF^\LC(S;\mathbf{Z},\mathbf{I}),\\
        \end{aligned}
    \end{equation}
    where $A_{ij}  \in \tril{n}$ for $i,j=1,\cdots,m, i \geq j$, $\mathbf{Z}=\{Z_{ij} = A_{ij} + \bbD(A_{ij})  \in \tril{n} \}_{i,j=1,i \geq j}^{m}$ and $\mathbf{I}=\{I, \cdots, I\}$.
    The above comes from the following.
    \begin{enumerate}[label=(\arabic*)]
        \item
        The linear map $g$ can be represented by $\{A_{ij}\}_{i,j=1,i \geq j}^{m}$;
        \item
        \cref{eq:spd_fc_lcm} and $v^{\LC}_{ij}$.
    \end{enumerate}
\end{proof}

\linkofproof{thm:gras_fc_onb}
\subsection{Proof of \cref{thm:gras_fc_onb}}

Before presenting our proof, we first discuss some basic facts about the ONB Grassmannian FC layer.

As implied by \cref{app:eq:tangent_vec_grasonb}, any tangent vector $V \in T_{\idonb} \grasonb{p,n}$ can be expressed as
\begin{equation}
    \label{app:eq:tangent_vec_onb}
    V = \left(\begin{array}{c}
    \bbzero\\
    I_{n-p}
    \end{array}\right) B_V
    =\left(\begin{array}{c}
    \bbzero\\
    B_V
    \end{array}\right), \text{ with } B_V \in \bbR{(n-p) \times p}.
\end{equation}

According to \cref{thm:riem_fc,app:eq:tangent_vec_onb}, the ONB Grassmannian FC layer $\calF(\cdot): \grasonb{p,n} \rightarrow \grasonb{q,m}$ has the following form:
\begin{equation}
    Y = \rieexp_{\idonbqm} \left( \sum_{\substack{i=1,\cdots, m-q \\ j=1,\cdots,m}} \left( \langle \rielog_{P_{ij}}(X), A_{ij} \rangle_{P_{ij}} U_{ij} \right) \right),
\end{equation}
where $\{U_{ij}\}$ is an orthonormal basis over $T_{\idonbqm}\grasonb{q,m}$. As discussed in \cref{subsec:fc_parameters}, we model the FC parameters by parallel transport and the Riemannian exponential map:
\begin{align}
    A_{ij} &= \pt{\idonb}{P_{ij}}(Z_{ij}),\\
    P_{ij} &= \rieexp_{\idonb} (\gamma_{ij} [Z_{ij}]),
\end{align}
where $Z_{ij}
    =\left(\begin{array}{c}
    \bbzero\\
    B_{Z_{ij}}
    \end{array}\right) \in T_{\idonb}\grasonb{p,n}$. Therefore, we can model each $P_{ij}$ and $A_{ij}$ by $B_{Z_{ij}} \in \bbR{(n-p)\times p}$ and $\gamma_{ij} \in \bbRscalar$. With the above ingredient, we present the proof in the following.

\begin{proof}
    \textbf{The standard orthonormal basis:}
    As the inner product over $T_{\idonbqm}\grasonb{q,m}$ is the Frobenius matrix inner product \citep[Eq. 3.2]{bendokat2024grassmann}, the standard orthonormal basis over $T_{\idonbqm}\grasonb{q,m}$ is
    \begin{equation}
        U_{ij}
        = \left(\begin{array}{c}
        \bbzero\\
        E_{ij}
        \end{array}\right), 1 \leq i \leq m-q \wedge 1 \leq j \leq q,
    \end{equation}
    where $\{E_{ij}\}$ are standard basis matrices over $\bbR{(m-q) \times q}$

    \textbf{The Riemannian exponential map at the origin: }
    The SVD of $V \in T_{\idonb} \grasonb{p,n}$ can be calculated via the SVD of $B_V$:
    \begin{equation}
        V =\left(\begin{array}{c}
                    \bbzero\\
                    B _V
                \end{array}\right)
        = \left(\begin{array}{c}
                    \bbzero\\
                    O
                \end{array}\right)\Sigma R^\top
        = \left(\begin{array}{c}
                    \bbzero\\
                    O \Sigma R^\top
                \end{array}\right),
    \end{equation}
    where $B_V \stackrel{\mathrm{SVD}}{:=} O \Sigma R^\top$. Therefore, the Riemannian exponential map at $\idonb$ can be simplified as
    \begin{equation}
        \begin{aligned}
            \rieexp_{\idonb}(V)
            &= \left(\begin{array}{c}
                    I_{p}\\
                    \bbzero
                \end{array}\right) R \cos (\Sigma) R^T + \left(\begin{array}{c}
                    \bbzero\\
                    O
                \end{array}\right) \sin (\Sigma) R^T\\
            &= \left(\begin{array}{c}
                    R \cos (\Sigma) R^T\\
                    O \sin (\Sigma) R^T
                \end{array}\right)
        \end{aligned}
    \end{equation}

     \textbf{$v_{ij}(U)$ under the ONB perspective:}
     The ONB parallel transport can be further simplified.
     Given $P \in \grasonb{p,n}$, we have the following for the Riemannian logarithm
     \begin{equation}
         \rielog_{\idonb}(P) =
         \left(\begin{array}{c}
            \bbzero \\
            B_P
            \end{array}\right)
         \stackrel{\text{SVD}}{:=} \left(\begin{array}{c}
            \bbzero \\
            O_P \Sigma_P R_P^\top
            \end{array}\right),
     \end{equation}
     with $B_P \stackrel{\text{SVD}}{:=} O_P \Sigma_P R_P^\top$.
     For $P \in \grasonb{p,n}$ and $Z \in T_{\idonb}\grasonb{p,n}$, the parallel transport can be further simplified:
     \begin{equation*}
        \begin{aligned}
            & \pt{\idonb}{P} (Z) \\
            &= \left(\left(\begin{array}{cc}
            \idonb R_P & \left(\begin{array}{c}
            \bbzero \\
            O_P
            \end{array}\right)
            \end{array}\right)
            \left(\begin{array}{c}
            -\sin (\Sigma_P) \\
            \cos (\Sigma_P)
            \end{array}\right)
            \left(\begin{array}{c}
            \bbzero \\
            O_P
        \end{array}\right)^T + \left(I-\left(\begin{array}{c}
            \bbzero \\
            O_P
        \end{array}\right) \left(\begin{array}{c}
            \bbzero \\
            O_P
        \end{array}\right)^T\right) \right) Z \\
            &=\left( \left( - \left(\begin{array}{c}
            I_p \\
            \bbzero
            \end{array}\right) R_{P} \sin (\Sigma_P) + \left(\begin{array}{c}
            \bbzero \\
            O_P
            \end{array}\right) \cos (\Sigma_P) \right) \left(\begin{array}{c}
            \bbzero \\
            O_P
            \end{array}\right)^T + \left(\begin{array}{cc}
            I_p & \bbzero \\
             \bbzero & I_{n-p} - O_P O_P^\top
            \end{array}\right) \right)Z\\
            &= \left( \left(\begin{array}{c}
            - R_{P} \sin (\Sigma_P) \\
            O_P \cos (\Sigma_P)
            \end{array}\right) \left(\begin{array}{cc}
            \bbzero & O_P^\top
            \end{array}\right) + \left(\begin{array}{cc}
            I_p & \bbzero \\
             \bbzero & I_{n-p} - O_P O_P^\top
            \end{array}\right) \right)Z\\
            &= \left( \left(\begin{array}{cc}
            \bbzero & - R_{P} \sin (\Sigma_P)O_P^\top \\
            \bbzero &  O_P \cos (\Sigma_P)O_P^\top
            \end{array}\right) + \left(\begin{array}{cc}
            I_p & \bbzero \\
             \bbzero & I_{n-p} - O_P O_P^\top
            \end{array}\right) \right)Z\\
            &= \left(\begin{array}{cc}
            I_{p} & - R_{P} \sin (\Sigma_P)O_P^\top \\
            \bbzero &  I_{n-p} + O_P \cos (\Sigma_P)O_P^\top - O_P O_P^\top
            \end{array}\right) Z\\
            &= \left(\begin{array}{cc}
            I_{p} & - R_{P} \sin (\Sigma_P)O_P^\top \\
            \bbzero &  I_{n-p} + O_P \cos (\Sigma_P)O_P^\top - O_P O_P^\top
            \end{array}\right) \left(\begin{array}{c}
            \bbzero\\
            B_Z
            \end{array}\right)\\
            &=
            \left(\begin{array}{c}
            - R_{P} \sin (\Sigma_P)O_P^\top B_Z  \\
            \left(O_P \cos (\Sigma_P)O_P^\top + I_{n-p} - O_P O_P^\top \right)B_Z
            \end{array}\right).
        \end{aligned}
     \end{equation*}

     Combining all the above results, one can directly obtain the results.
\end{proof}

\linkofproof{thm:gras_fc_pp}
\subsection{Proof of \cref{thm:gras_fc_pp}}
\begin{proof}
    First, $v_{ij}(X)$ over the Grassmannian $\graspp{p,n}$ takes the following form:
    \begin{equation}
        \label{app:eq:v_ij_gras_pp}
        \begin{aligned}
            v_{ij}(X)
            &= \inner{ \rielog_{P_{ij}} (X)}{\pt{\idpp}{P_{ij}}(Z_{ij})}_{P_{ij}}\\
            &\stackrel{(1)}{=} \frac{1}{2} \inner{ \rielog _{P_{ij}} (X)}{\pt{\idpp}{P_{ij}}(Z_{ij})}\\
        \end{aligned}
    \end{equation}
    where (1) comes from \cref{app:tab:riem_operators_gras}. Here, each $Z_{ij} \in T_{\idpp}\graspp{p,n}$ and $P_{ij} \in \graspp{p,n}$.

    \textbf{Riemannian logarithm.}
    As shown by \citet[Prop. 3.12]{nguyen2024matrix}, the PP Grassmannian logarithm can be calculated using the ONB logarithm:
    \begin{equation}
        \rielog^{\pp}_{P}(X) = \pi_{*,\pi(P)} \left(\rielog^{\onb}_{\pi^{-1}(P)} (\pi^{-1}(X)) \right),
    \end{equation}
    where $\pi(U) = UU^\top: \grasonb{p,n} \rightarrow \graspp{p,n}$ is the Riemannian isometry, and $\pi_{*,U} (V) = UV^\top + VU^\top$ is the differential map for all $U \in \grasonb{p,n}$ and $V \in T_{U}\grasonb{p,n}$.

    \textbf{Tangent vector and Riemannian exponential map at the identity.}
    As implied by \cref{app:eq:tangent_vec_pp}, any tangent vector at the identity has the following form:
    \begin{equation} \label{app:eq:tangent_vec_id}
        V=
        \left(\begin{array}{cc}
            0 & B^T \\
            B & 0
        \end{array}\right)
        \in T_{\idpp}\graspp{p,n} \text{ with } B \in \bbR{(n-p) \times p}.
    \end{equation}
    The Riemannian exponential map at the identity can also be simplified:
    \begin{equation}
        \label{app:eq:exp_id_simplfied}
        \begin{aligned}
            \rieexp_{\idpp}(V)
            &= \exp([V,\idpp])\idpp\exp(-[V,\idpp])\\
            &= \exp \left(\left(\begin{array}{cc}
            0 & -B^T \\
            B & 0
            \end{array}\right) \right) \idpp \exp \left(\left(\begin{array}{cc}
            0 &  -B^T \\
            B & 0
            \end{array}\right) \right)^{\top} \\
            &=
            \left(\exp \left(\left(\begin{array}{cc}
                0 & -B^T \\
                B & 0
            \end{array}\right) \right) \right)_{1:p}
            \left(\left(\exp \left(\left(\begin{array}{cc}
                0 & -B^T \\
                B & 0
            \end{array}\right) \right) \right)_{1:p} \right)^\top
        \end{aligned}
    \end{equation}
    with $(\cdot)_{1:p}$ being the first-$p$ columns of the input square matrix.

    \textbf{Parallel transport starting at the identity.}
    The parallel transport along geodesic from $\idpp$ to $P \in \graspp{p,n}$ can also be simplified. For any $V \in T_{\idpp} \graspp{p,n}$, denoting $\bar{P}=\rielog_{\idpp}(P)$, we have the following:
    \begin{equation}
        \label{app:eq:pt_simplified}
        \begin{aligned}
            \pt{\idpp}{P}(V)
            &\stackrel{(1)}{=} \exp \left(\left[\bar{P}, \idpp \right] \right) V \exp \left( -\left[\bar{P}, \idpp \right] \right)\\
            &\stackrel{(2)}{=}
            \exp \left(\left(\begin{array}{cc}
            0 & -B_{P}^T \\
            B_{P} & 0
            \end{array}\right) \right)
            V
            \exp \left(\left(\begin{array}{cc}
            0 & -B_{P}^T \\
            B_{P} & 0
            \end{array}\right) \right)^\top
        \end{aligned}
    \end{equation}
    The above derivation comes from the following.
    \begin{enumerate}[label=(\arabic*)]
        \item
        \cref{app:tab:riem_operators_gras};
        \item
        $\bar{P}=
        \left(\begin{array}{cc}
            0 & B_{P}^T \\
            B_{P} & 0
            \end{array}\right)
        $
    \end{enumerate}
    \textbf{Trivialization and simplification}
    Combining \cref{app:eq:v_ij_gras_pp,app:eq:tangent_vec_id,app:eq:exp_id_simplfied,app:eq:pt_simplified}, we model each $P_{ij}$ such that
    \begin{equation}
        P_{ij} =
        \exp \left(\left(\begin{array}{cc}
            0 & -B_{P_{ij}}^T \\
            B_{P_{ij}} & 0
        \end{array}\right) \right)
        \idpp
        \exp \left(\left(\begin{array}{cc}
        0 &  -B_{P_{ij}}^T \\
        B_{P_{ij}} & 0
        \end{array}\right) \right)^{\top} \\
    \end{equation}
    where $B_{P_{ij}} = \gamma_{ij} [B_{Z_{ij}}]$ with
    $Z_{ij}=
    \left(\begin{array}{cc}
        0 &  B_{Z_{ij}}^T \\
        B_{Z_{ij}} & 0
    \end{array}\right)$ and $B_{Z_{ij}} \in \bbR{(n-p) \times p}$.

    Denoting $O_{ij} =
    \exp \left(\left(\begin{array}{cc}
            0 & -B_{P_{ij}}^T \\
            B_{P_{ij}} & 0
    \end{array}\right) \right)$, $v_{ij}(X)$ can be simplified as
    \begin{equation} \label{app:eq:v_ij_pp_final}
        v_{ij}(X)=
        \frac{1}{2} \inner{ \pi_{*,\pi(P)} \left(\rielog^{\onb}_{(O_{ij})_{1:p}} (\pi^{-1}(X)) \right)}{O_{ij} Z_{ij} O_{ij}^\top}\\
    \end{equation}

    \textbf{Orthonormal bases.}
    Finally, let us handle the orthonormal bases over $T_{\idppqm}\graspp{q,m}$.
    For any tangent vector $V_1, V_2 \in T_{\idppqm}\graspp{q,m}$, we have the following:
    \begin{equation}
        \begin{aligned}
            \inner{V_1}{V_2}_{\idpp}
            &= \frac{1}{2} \inner{V_1}{V_2} \\
            &= \frac{1}{2} \inner{
            \left(\begin{array}{cc}
                0 & B_{V_1}^T \\
                B_{V_1} & 0
            \end{array}\right)
                }{
            \left(\begin{array}{cc}
                0 & B_{V_2}^T \\
                B_{V_2} & 0
            \end{array}\right)
            }\\
            &= \inner{B_{V_1}}{B_{V_2}}
        \end{aligned}
    \end{equation}
    Therefore, the orthonormal basis is
    \begin{equation} \label{app:eq:orth_basis_pp}
        U_{ij}=
        \left(\begin{array}{cc}
                0 & E_{ij}^\top \\
                E_{ij} & 0
            \end{array}\right), \forall i=1,\cdots, m-q \wedge j=1,\cdots,q
    \end{equation}
    where $E_{ij} \in \bbR{(m-q) \times q}$ is the standard basis matrix.

    Combining \cref{app:eq:exp_id_simplfied,app:eq:v_ij_pp_final,app:eq:orth_basis_pp}, one can readily obtain the results.
\end{proof}

\section*{NeurIPS Paper Checklist}

\begin{enumerate}

\item {\bf Claims}
    \item[] Question: Do the main claims made in the abstract and introduction accurately reflect the paper's contributions and scope?
    \item[] Answer: \answerYes{}
    \item[] Justification: The abstract and \cref{sec:introduction} state the scope, contributions, and empirical claims. The main method, examples, and experiments in \cref{subsec:riem_fc,sec:riem_conv_layers,sec:experiments} support these claims across hyperbolic, SPD, and Grassmannian manifolds.

\item {\bf Limitations}
    \item[] Question: Does the paper discuss the limitations of the work performed by the authors?
    \item[] Answer: \answerYes{}
    \item[] Justification: The limitations are discussed in \cref{app:sec:limitations}. The paper states that the framework applies to computationally tractable Riemannian manifolds and may not directly apply when tractable Riemannian operators are unavailable.

\item {\bf Theory assumptions and proofs}
    \item[] Question: For each theoretical result, does the paper provide the full set of assumptions and a complete (and correct) proof?
    \item[] Answer: \answerYes{}
    \item[] Justification: The theoretical statements are presented in the main text and appendix with assumptions such as well-defined Riemannian operators discussed in \cref{app:rmk:incomplete_spd,app:rmk:cutlocus_gras}. Complete proofs are provided in \cref{app:sec:proofs}.

    \item {\bf Experimental result reproducibility}
    \item[] Question: Does the paper fully disclose all the information needed to reproduce the main experimental results of the paper to the extent that it affects the main claims and/or conclusions of the paper (regardless of whether the code and data are provided or not)?
    \item[] Answer: \answerYes{}
    \item[] Justification: The experimental setup, datasets, modeling choices, architectures, optimizers, and hyperparameters are described in \cref{sec:experiments,app:sec:exp_details}. The method definitions and proofs provide the information needed to reproduce the proposed layers.
    \item[] Guidelines:
    \begin{itemize}
        \item The answer \answerNA{} means that the paper does not include experiments.
        \item If the paper includes experiments, a \answerNo{} answer to this question will not be perceived well by the reviewers: Making the paper reproducible is important, regardless of whether the code and data are provided or not.
        \item If the contribution is a dataset and\slash or model, the authors should describe the steps taken to make their results reproducible or verifiable.
        \item Depending on the contribution, reproducibility can be accomplished in various ways. For example, if the contribution is a novel architecture, describing the architecture fully might suffice, or if the contribution is a specific model and empirical evaluation, it may be necessary to either make it possible for others to replicate the model with the same dataset, or provide access to the model. In general, releasing code and data is often one good way to accomplish this, but reproducibility can also be provided via detailed instructions for how to replicate the results, access to a hosted model (e.g., in the case of a large language model), releasing of a model checkpoint, or other means that are appropriate to the research performed.
        \item While NeurIPS does not require releasing code, the conference does require all submissions to provide some reasonable avenue for reproducibility, which may depend on the nature of the contribution. For example
        \begin{enumerate}
            \item If the contribution is primarily a new algorithm, the paper should make it clear how to reproduce that algorithm.
            \item If the contribution is primarily a new model architecture, the paper should describe the architecture clearly and fully.
            \item If the contribution is a new model (e.g., a large language model), then there should either be a way to access this model for reproducing the results or a way to reproduce the model (e.g., with an open-source dataset or instructions for how to construct the dataset).
            \item We recognize that reproducibility may be tricky in some cases, in which case authors are welcome to describe the particular way they provide for reproducibility. In the case of closed-source models, it may be that access to the model is limited in some way (e.g., to registered users), but it should be possible for other researchers to have some path to reproducing or verifying the results.
        \end{enumerate}
    \end{itemize}

\item {\bf Open access to data and code}
    \item[] Question: Does the paper provide open access to the data and code, with sufficient instructions to faithfully reproduce the main experimental results, as described in supplemental material?
    \item[] Answer: \answerNo{}
    \item[] Justification: Dataset sources and several baseline code sources are cited in \cref{app:sec:exp_details}. The code will be released after the review process.

\item {\bf Experimental setting/details}
    \item[] Question: Does the paper specify all the training and test details (e.g., data splits, hyperparameters, how they were chosen, type of optimizer) necessary to understand the results?
    \item[] Answer: \answerYes{}
    \item[] Justification: Training and testing details are described in \cref{app:subsec:exp_details_hyperbolic,app:subsec:exp_details_spd,app:subsec:exp_details_grass}. The appendix reports dataset splits where applicable, optimizers, learning rates, batch size, epochs, and model-specific settings.

\item {\bf Experiment statistical significance}
    \item[] Question: Does the paper report error bars suitably and correctly defined or other appropriate information about the statistical significance of the experiments?
    \item[] Answer: \answerYes{}
    \item[] Justification: The experimental tables report values with $\pm$ and describe K-fold averages.

\item {\bf Experiments compute resources}
    \item[] Question: For each experiment, does the paper provide sufficient information on the computer resources (type of compute workers, memory, time of execution) needed to reproduce the experiments?
    \item[] Answer: \answerYes{}
    \item[] Justification: \cref{app:subsec:hardware} summarizes the hardware information.
    
\item {\bf Code of ethics}
    \item[] Question: Does the research conducted in the paper conform, in every respect, with the NeurIPS Code of Ethics \url{https://neurips.cc/public/EthicsGuidelines}?
    \item[] Answer: \answerYes{}
    \item[] Justification: The paper uses publicly available benchmark datasets and there is no ethical concern.

\item {\bf Broader impacts}
    \item[] Question: Does the paper discuss both potential positive societal impacts and negative societal impacts of the work performed?
    \item[] Answer: \answerNA{}
    \item[] Justification: The work is primarily methodological and evaluated on benchmark tasks.
    
\item {\bf Safeguards}
    \item[] Question: Does the paper describe safeguards that have been put in place for responsible release of data or models that have a high risk for misuse (e.g., pre-trained language models, image generators, or scraped datasets)?
    \item[] Answer: \answerNA{}
    \item[] Justification: The paper does not release high-risk pretrained models, image generators, language models, or scraped datasets. The experiments use existing benchmark datasets.

\item {\bf Licenses for existing assets}
    \item[] Question: Are the creators or original owners of assets (e.g., code, data, models), used in the paper, properly credited and are the license and terms of use explicitly mentioned and properly respected?
    \item[] Answer: \answerYes{}
    \item[] Justification: Existing datasets and baseline implementations are cited and several URLs are provided in \cref{app:sec:exp_details}.

\item {\bf New assets}
    \item[] Question: Are new assets introduced in the paper well documented and is the documentation provided alongside the assets?
    \item[] Answer: \answerNA{}
    \item[] Justification: The paper does not introduce or release a new dataset, benchmark, or pretrained model asset. The proposed layers and networks are documented as methods in the paper.

\item {\bf Crowdsourcing and research with human subjects}
    \item[] Question: For crowdsourcing experiments and research with human subjects, does the paper include the full text of instructions given to participants and screenshots, if applicable, as well as details about compensation (if any)? 
    \item[] Answer: \answerNA{}
    \item[] Justification: The paper does not conduct crowdsourcing experiments or new human-subject studies. It uses existing benchmark datasets.

\item {\bf Institutional review board (IRB) approvals or equivalent for research with human subjects}
    \item[] Question: Does the paper describe potential risks incurred by study participants, whether such risks were disclosed to the subjects, and whether Institutional Review Board (IRB) approvals (or an equivalent approval/review based on the requirements of your country or institution) were obtained?
    \item[] Answer: \answerNA{}
    \item[] Justification: The paper does not report crowdsourcing or new human-subject research requiring IRB approval. The experiments rely on existing benchmark datasets.

\item {\bf Declaration of LLM usage}
    \item[] Question: Does the paper describe the usage of LLMs if it is an important, original, or non-standard component of the core methods in this research? Note that if the LLM is used only for writing, editing, or formatting purposes and does \emph{not} impact the core methodology, scientific rigor, or originality of the research, declaration is not required.
    \item[] Answer: \answerNA{}
    \item[] Justification: \cref{app:sec:llm-usage} describes their use for language polishing, minor editing, and limited assistance in translating mathematical formulations into PyTorch code.

\end{enumerate}

\end{document}